\documentclass[10pt]{article}
\usepackage{enumitem}
\usepackage{amsmath}
\usepackage{amsfonts}
\usepackage{amssymb}
\usepackage{graphicx}
\usepackage{subfigure}	
\usepackage[paper=a4paper,textwidth=170mm,textheight=260mm]{geometry}
\usepackage{mathrsfs}
\usepackage{amsthm}
\usepackage{algorithm}
\usepackage{algorithmic}
\usepackage{hyperref}
\usepackage{url}
\newtheorem{theorem}{Theorem}

\newtheorem{lemma}{Lemma}
\newtheorem{proposition}{Proposition}
 \newtheorem{definition}{Definition}
 
\newtheorem{example}{Example}
\newcommand{\argmax}{\operatornamewithlimits{argmax}}
\newcommand{\argmin}{\operatornamewithlimits{argmin}}
\newcommand{\scalprod}[2]{\left\langle #1,#2 \right\rangle}
\renewcommand{\algorithmicrequire}{\textbf{Input:}}
\usepackage{color}

\newcommand*\samethanks[1][\value{footnote}]{\footnotemark[#1]}

\title{Image registration with sparse approximations in parametric dictionaries}
\author{Alhussein Fawzi\thanks{Ecole Polytechnique Federale de Lausanne (EPFL), Signal Processing Laboratory (LTS4), Lausanne 1015-Switzerland (\href{mailto:alhussein.fawzi@epfl.ch}{alhussein.fawzi@epfl.ch}, \href{mailto:pascal.frossard@epfl.ch}{pascal.frossard@epfl.ch})}
\and Pascal Frossard\samethanks}

\begin{document}

\maketitle

\begin{abstract}
We examine in this paper the problem of image registration from the new perspective where images are given by sparse approximations in parametric dictionaries of geometric functions. We propose a registration algorithm that looks for an estimate of the global transformation between sparse images by examining the set of relative geometrical transformations between the respective features. We propose a theoretical analysis of our registration algorithm and we derive performance guarantees based on two novel important properties of redundant dictionaries, namely the robust linear independence and the transformation inconsistency. We propose several illustrations and insights about the importance of these dictionary properties and show that common properties such as coherence or restricted isometry property fail to provide sufficient information in registration problems. We finally show with illustrative experiments on simple visual objects and handwritten digits images that our algorithm outperforms baseline competitor methods in terms of transformation-invariant distance computation and classification.
\end{abstract}

\section{Introduction}

With the ever-increasing quantity of information produced by sensors, efficient processing techniques for identifying meaningful information in high-dimensional data sets become crucial.
One of the key challenges is to be able to identify relevant objects captured at different times, from various viewpoints, or by different sensors.
Sparse signal representations, which decompose linearly signals into key features, have recently been shown to be a powerful tool in image analysis tasks \cite{wright2009robust, mairal2008supervised, elad2006image}. 
In general, it is however necessary to align signals a priori in order to derive meaningful comparisons or distances in the analysis. Image alignment or \textit{registration} thus represents a crucial yet non-trivial task in many image processing and computer vision applications, such as object detection, localization and classification to name a few.

In this paper, we propose a registration algorithm for sparse images that are given as a linear combination of geometric features drawn from a parametric \textit{dictionary}. The estimation of the global geometric transformation between images is performed first by building a set of candidate transformation solutions with all the relative transformations between features in each image. The transformation that leads to the smallest transformation-invariant distance is finally selected as the global transformation estimate. While image registration is generally a complex optimization problem, our algorithm offers a low complexity solution when the images have a small number of constitutive components. We analyze its theoretical performance, which mainly depends on the construction of the dictionary that supports the sparse image representations. We introduce two novel properties for redundant dictionaries, namely the robust linear independence and transformation inconsistency, which permit to characterize the performance of the registration algorithm. The benefits of these properties are studied in detail and compared to common properties such as the coherence or the restricted isometry property. We finally provide illustrative registration and classification experiments, where our algorithm outperforms baseline solutions from the literature, particularly when relative transformations between images are large.

The image registration problem has been widely investigated from different perspectives in the literature, but not from the point of view of sparse image approximations as studied in this paper. Image registration algorithms are usually classified into direct (pixel-based) methods, and featured-based methods \cite{szeliski2006image}. We review these two classes of methods, and refer the reader to \cite{zitova2003image, szeliski2006image} for a general survey on image alignment. 

Direct pixel-based methods simply consist in trying all candidate transformations and see how much pixels agree when the images are transformed relatively to each other. A major drawback of these methods is their inefficiency when the number of candidate transformations becomes large. Therefore, hierarchical coarse-to-fine techniques based on image pyramids have been developed \cite{bergen1992hierarchical, vural2013analysis} to offer a compromise between accuracy and computational complexity. 
In a different approach, the authors of \cite{peng2010rasl, zhang2012tilt} formulate the registration problem as a low-rank matrix recovery problem with sparse noise, and leverage the recent advances in convex optimization to find the optimal transformation that best aligns the images. The approaches developed in \cite{segman1992canonical, ferraro1988relationship} map the images to a canonical space where deformations take a simple form and thus allows easier registration. 


The popular \emph{feature-based} approaches \cite{szeliski2010computer} represent a more efficient class of methods for image registration. They are usually built on several steps: (i) \emph{feature detection}, which searches for stable distinctive locations in the images, (ii) \emph{feature description}, which provides a description of each detected location with an invariant descriptor, (iii) \emph{features matching} between the images and (iv) \emph{transformation estimation} that estimates the global transformation by looking at matched features. Note that it is crucial in this class of methods to describe the features in a transformation-invariant way for easier matching. We refer the reader to \cite{mikolajczyk2005performance} for a comparison of the main different methods.  A popular  example of the feature-based approach relies on the scale invariant feature transform (SIFT) \cite{lowe2004distinctive} that combines the Difference-of-Gaussian (DoG) detector with a descriptor based on image gradient orientations around the keypoint. The SIFT method is invariant to rotation, scaling and translation and some of its extensions achieve invariance to affine transformations \cite{morel2009asift}. Moreover, the SIFT descriptors are often used in combination with affine invariant detectors such as those proposed in \cite{matas2002robust, mikolajczyk2005comparison, kimmel2011mser, bronstein2010spatially} for affine image registration.
Even though SIFT has been very successful in many computer vision applications, it is mostly built on empirical results and several parameters need to be set manually. Feature-based methods in general are not well suited for estimating large transformations between target images, as the matching accuracy and keypoint localization degrade for large transformations. 

   

Finally, we mention some recent advances in transformation-invariant distance estimation, which is closely related to image registration. 
The transformation-invariant distance is defined as the minimum distance between the possible transformations of two patterns. In general, the signals generated by the possible transformations of a pattern can be represented by a non linear manifold. Computing the transformation-invariant distance between two patterns or equivalently the manifold distance is thus a difficult problem in general. The authors in \cite{simard1998transformation} locally approximate the transformation invariant distance with the distance between the linear spaces that are tangent to both manifolds. Vasconcelos et. al. \cite{vasconcelos2005multiresolution} go beyond the limitations of local invariance in tangent distance methods by embedding the tangent distance computation in a multiresolution framework. Kokiopoulou et. al. in \cite{kokiopoulou2009minimum} achieve global invariance by approximating the original pattern with a linear combination of atoms from a parametric dictionary. Thanks to this approximation, the manifold is given in a closed form and the objective function becomes equal to a difference of convex functions that can be globally minimized using cutting plane methods. Unfortunately, this class of optimization methods have a slow convergence rate with complexity limitations in practical settings.

In this paper, we propose to examine the image registration problem from a novel perspective by building on our earlier work \cite{fawzi2012_icassp} where we consider that images are given in the form of sparse approximations. 
Unlike the existing methods, this approach guarantees invariance to transformations of arbitrary magnitude and is generic with respect to the transformation group considered in the registration problem. The detailed analysis of our new framework further provides useful insights on the connections between image registration problems and sparse signal processing.

The rest of this paper is organized as follows. In Section \ref{sec:registration}, we formulate the problem of registration of sparse images and present our registration algorithm. Section \ref{sec:theoretical_analysis} proposes a theoretical performance analysis of our algorithm, and introduces two new dictionary properties. We finally present illustrative experiments in Section \ref{sec:experiments}.

\section{Registration of sparse images}
\label{sec:registration}

\subsection{Preliminaries}

We first define the notations and conventions used in this paper. We denote respectively by $\mathbb{R}$, $\mathbb{R}^+$, $\mathbb{R}^+_{^*}$ the set of real numbers, the set of non negative real numbers and the set of positive real numbers. We consider images to be continuous functions in $L^2 = \{ f: \mathbb{R}^2 \rightarrow \mathbb{R}: \int_{-\infty}^{+\infty} |f(x)|^2 dx < \infty \}$. We denote the scalar product associated with $L^2$ as: $\scalprod{f}{g} = \int_{-\infty}^{+\infty} f(x) g(x) dx$, and the norm by $\| f \|_2 = \sqrt{\int_{-\infty}^{+\infty} |f(x)|^2 dx}$. Then, we define $\mathcal{T}$ to be a transformation group and denote by $\circ$ its associated composition rule. We consider that the group $\mathcal{T}$ includes the transformations between pairs of images in our registration problem. We represent any transformation $\eta \in \mathcal{T}$ by a vector in $\mathbb{R}^P$ (where $P$ denotes the dimension of $\mathcal{T}$) containing the parameters of the transformation. 

Alternatively, we represent a transformation $\eta \in \mathcal{T}$ with its unitary representation $U(\eta)$ in $L^2$. Therefore, for any $\eta \in \mathcal{T}$, $U(\eta)$ is the function that maps an image $f$ to its transformed image $U(\eta) f \in L^2$ by $\eta$. Moreover, as $U(\eta)$ is a unitary operator, we have $\| U(\eta) f \|_2 = \| f \|_2$. In order to avoid heavy notations, we also use $f_{\eta}$ to denote $U(\eta) f$. We give in Table \ref{tab:transf_groups} some examples of transformation groups and their unitary representation in $L^2$. 

\begin{table}[ht]
\footnotesize
\centering
\begin{tabular}{|c|c|c|c|}
  \hline
  Group & Parameters & Composition   & Unitary representation  \\
  \hline 
   &  $\eta$ &  $\eta \circ \eta'$  &  $U(\eta) f = f_{\eta}$ \\
  \hline \hline
  $\mathbb{R}^2$ & $b$ & $b + b'$ & $f(x_1 - b_1, x_2 - b_2)$ \\
  \hline
  Special Euclidean group $SE(2)$ & $(b, \theta)$ & $(b + R_{\theta} b', \theta + \theta')$ & $f\left( R_{-\theta} ( x - b ) \right)$ \\
  \hline
  Similarity group $SIM(2)$ & $(b, a, \theta)$ & $(b + a R_{\theta} b', a a', \theta + \theta')$ & $a^{-1} f\left( \frac{R_{-\theta}}{a} ( x - b ) \right)$ \\
  \hline
\end{tabular}
\caption{\label{tab:transf_groups} Examples of transformation groups and their unitary representation in $L^2$. Parameters with a prime are associated with a secondary transformation $\eta'$, and $R_{\theta}$ denotes the rotation matrix with angle $\theta$.}
\end{table}

The group $\mathbb{R}^2$ is the group of translations in the plane. The Special Euclidean group $SE(2)$ is the group of translations and rotations in the plane. Its dimension is equal to $3$ ($2$ degrees of freedom are associated with the translation and one is associated with rotation). The similarity group $SIM(2)$ of the plane is the set of transformations consisting of translations, isotropic dilations and rotations. This group plays a particular importance in transformation invariant image processing since it contains the basic transformations we usually want to be invariant to.

Finally, if $c \in \mathbb{R}^n$ and $1 \leq p < \infty$, we denote by $\| c \|_p$ the $\ell_p$ norm of $c$ defined by $\| c \|_p = \left( \sum_{i=1}^n |c_i|^p \right)^{1/p}$. Note that the notation $\| \cdot \|_2$ is overloaded since it denotes either the continuous $L^2$ norm or the discrete $\ell_2$ norm. However, the distinction between both cases will be clear from the context.

\subsection{Problem formulation}
\label{sec:probform}

We formulate now the registration problem that we consider in the paper. Let $I_1$ and $I_2$ be two images in $L^2$. We are interested in computing the optimal transformation between images $I_1$ and $I_2$. Hence, we formulate the original alignment problem as follows:
\begin{align*}
\text{(P'): Find } \eta'_0 = \argmin_{\eta \in \mathcal{T}} \left\| U(\eta) I_1 - I_2 \right\|_2.
\end{align*}
We denote by $d(I_1, I_2) = \left\| U(\eta'_0) I_1 - I_2 \right\|_2$ the \emph{transformation invariant distance} between $I_1$ and $I_2$. It corresponds to the regular Euclidean distance when the images are aligned optimally in the $L^2$ sense. Unfortunately, computing the transformation $\eta'_0$ and the transformation invariant distance $d(I_1, I_2)$ is a hard problem since the objective function is typically non convex and exhibits many local minima. 

In order to circumvent this problem, we consider that the images are well approximated by their sparse expansion in a series of geometric functions. Specifically, let $\mathcal{D}$ be a \textit{parametric dictionary of geometric features} constructed by transforming a generating function $\phi \in L^2$ as follows: 
\begin{align}
\label{eq:dictionary}
\mathcal{D} =\{ \phi_{\gamma}: \gamma \in \mathcal{T}_d \} \subset L^2,
\end{align}
where $\mathcal{T}_d \subset \mathcal{T}$ is a finite discretization of the transformation group $\mathcal{T}$ and $\phi_{\gamma} = U(\gamma) \phi$ denotes the transformation of the generating function $\phi$ by $\gamma$. We denote by $p$ and $q$ the respective $K$-sparse approximations of $I_1$ and $I_2$  in the dictionary $\mathcal{D}$: 
\begin{eqnarray}
p & = \sum_{i=1}^K c_i \phi_{\gamma_i}, \nonumber \\
q & = \sum_{i=1}^K d_i \phi_{\delta_i}. \label{eq:sparseapprox}
\end{eqnarray}
Since the dictionary $\mathcal{D}$ contains features that represent potential parts of the image, we assume that coefficients $c_i$ and $d_i$ are all non negative so that the different features do not cancel each other. 
 
We refer to any element $\phi_{\gamma}$ in $\mathcal{D}$ as a \emph{feature} or \emph{atom}. We suppose in this paper that the generating function $\phi$ is non negative. 
Besides, we suppose for simplicity that $\gamma \mapsto \phi_{\gamma}$ defines a one-to-one mapping. This assumption means that the generating function does not have any symmetries in $\mathcal{T}$\footnote{We extend this assumption to the more general setting where the stabilizer of $\phi$ defined by $\mathcal{S}_{\phi} = \{ \gamma \in \mathcal{T}: U(\gamma) \phi = \phi \}$ is a finite set in Appendix \ref{app:not_bijective}.}. Finally, we suppose without loss of generality that the mother function $\phi$ is normalized so that $\| \phi \|_2 = 1$. 

We can now reformulate the registration problem as the problem of finding the optimal relative transformation between sparse patterns. In particular, we reformulate our registration problem as follows: 
\begin{align*}
\text{(P): Find } \eta_0 = \argmin_{\eta \in \mathcal{T}} \left\| U(\eta) p - q \right\|_2.
\end{align*}
The smallest distance $d(p,q) = \left\| U(\eta_0) p - q \right\|_2$ is the transformation invariant distance computed between the sparse image approximations $p$ and $q$. 
Compared to the original problem, the images $I_1$ and $I_2$ are replaced by their respective sparse approximations $p$ and $q$.
This presents some potential advantages in applications where users do not have access to the original images; more importantly, the prior information on the support of $p$ and $q$ effectively guides the registration process, as we will see in the next paragraph. We should note that if the images are not well approximated by their sparse expansions, the solution of $(P)$ may substantially differ from the true transformation obtained by solving $(P')$.
%

\subsection{Registration algorithm}

We propose now a novel and simple algorithm to solve the registration problem for images given by their sparse approximations. The core idea of our registration algorithm lies in the \textit{covariance} property of the dictionary $\mathcal{D}$: a global transformation applied on the image induces an equivalent transformation on the corresponding features\footnote{The meaning of covariance that is used in this paper is not to be confused with that of covariance used in statistics.}. 
Thanks to this covariance property, it is possible to infer the global transformation between the images by a simple computation of the relative transformations between the features in both images. 

Specifically, let $\mathcal{T}_a^{p,q}$ be the set of relative transformations between pairs of features taken respectively in $p$ and $q$: $\mathcal{T}_a^{p,q} = \{ \delta_i \circ \gamma_{j}^{-1}: 1 \leq i,j \leq K \}$. We can thus estimate the relative transformation between the images by solving the following relaxed problem of $(P)$:
\begin{align*}
\text{($\hat{P}$): Find } \hat{\eta} = \argmin_{\eta \in \mathcal{T}_a^{p,q}} \left\| U(\eta) p - q \right\|_2.
\end{align*}
The minimum of the objective function $d_a(p,q) = \| U(\hat{\eta}) p - q \|_2$ is defined as the \emph{approximate transformation-invariant distance} between $I_1$ and $I_2$.

Even though problems $(P)$ and $(\hat{P})$ share some similarities, they differ in an important aspect, that is the search space. It is reduced from $\mathcal{T}$ to the finite set $\mathcal{T}_a^{p,q}$. This constrains the estimated transformation to be equal to a transformation that exactly maps two features taken respectively from $p$ and $q$. The assumption that $\mathcal{T}$ can be replaced by $\mathcal{T}_a^{p,q}$ originates from the observation that features are covariant to the global transformation applied on the original image. Even though this assumption is not necessarily true for all features when innovation exists between the images (other than a global transformation), we expect to have at least one feature whose transformation is consistent with the optimal transformation $\eta_0$. We analyze in detail the error due to this assumption in Section \ref{sec:theoretical_analysis}. The advantage of replacing $\mathcal{T}$ by $\mathcal{T}_a^{p,q}$ is however immediate: we have reduced an intractable problem to a problem whose search space is of cardinality at most $K^2$. Since $K$ is generally chosen to be small enough, the problem $(\hat{P})$ can be efficiently solved by a full search over all the elements of $\mathcal{T}_a^{p,q}$. The registration algorithm is summarized in Algorithm \ref{alg:registration_algo}. 

\begin{algorithm}[ht]
\algorithmicrequire{\hspace{1mm} sparse approximations $p = \sum_{i=1}^K c_i \phi_{\gamma_i}$ and $q = \sum_{i=1}^K d_i \phi_{\delta_i}$.} \\
\begin{algorithmic}
\STATE $\mathbf{1.}$ Construct the set $\mathcal{T}_a^{p,q}$:
\begin{align*}
\mathcal{T}_a^{p,q} & = \{ \delta_i \circ \gamma_j^{-1}: 1 \leq i,j \leq K \}.
\end{align*}
\STATE $\mathbf{2.}$ Estimate the transformation $\hat{\eta}$ and $d_a(p,q)$.
\begin{align*}
\hat{\eta} & \leftarrow \argmin_{\eta \in \mathcal{T}_a^{p,q}} \left\| U(\eta) p - q \right\|_2, \\
d_a(p,q) & \leftarrow \left\| U(\hat{\eta}) p - q \right\|_2.
\end{align*}
\STATE $\mathbf{3.}$ Return $(\hat{\eta}, d_a(p,q))$.
\end{algorithmic}
\caption{\label{alg:registration_algo} Image registration algorithm}
\end{algorithm}

The value of $K$ controls the computational complexity of Algorithm \ref{alg:registration_algo}: a large value of $K$ results in a large cardinality of the search space $\mathcal{T}_a^{p,q}$. Furthermore, the value of $K$ also generally controls the error in the approximation of the original images by their sparse expansions $p$ and $q$. We discuss more in detail the influence of $K$ on our registration algorithm in Section \ref{sec:experiments}. Note finally that we have supposed for simplicity that both images $I_1$ and $I_2$ are approximated by the same number of features. However, it is easy to see that one can generalize it to the case where the number of features are different in the two images. In this case, we have $| \mathcal{T}_a^{p,q} | = K_1 K_2$ instead of $K^2$, where $K_1$ and $K_2$ are the number of features in $I_1$ and $I_2$ respectively.

In the next section, we analyze the performance of the proposed registration algorithm in different settings, and focus in particular on the influence of the dictionary $\mathcal{D}$ on the registration performance.


\section{Theoretical analysis}

\label{sec:theoretical_analysis}

In this section, we examine the penalty of relaxing the original problem $(P')$ into $(\hat{P})$ in terms of registration performance. We first discuss the framework and the assumptions used in our analysis. Then, we study a simple case where the image patterns are exactly related by a (possibly very large) geometrical transformation. We show that under a mild assumption on the dictionary, our algorithm achieves perfect registration. We then extend the analysis to the general case and introduce two key properties of the dictionary (namely \emph{robust linear independence} and \emph{transformation inconsistency}). We show that under some conditions on these properties, our algorithm succeeds in recovering the correct relative transformation with a bounded error in the general case, as long as the innovation between the images (other than the global geometrical transformation) is controlled. We give at each step of the analysis the main intuitions and several examples to illustrate the novel notions introduced in our analysis.

\subsection{Analysis framework}

We first define a performance metric to measure the image registration accuracy. As we want to capture the performance of our registration algorithm with respect to the optimal image alignment obtained by solving $(P')$, a natural metric consists in computing the difference between the transformation invariant distance and its approximate version, i.e., $E'(p, q, I_1, I_2) = | d_a(p,q) - d(I_1, I_2) |$. We however assume in this paper that the images are given by their sparse expansions. Therefore, we use an alternative registration performance given by $E(p,q) = d_a(p,q) - d(p,q)$, where we use the transformation invariant distance computed between the sparse image approximations $p$ and $q$ instead of the original images. Note that $E(p,q) \geq 0$ since $\mathcal{T}_a^{p,q} \subset \mathcal{T}$.

We relate in the following proposition the two registration metrics $E(p,q)$ and $E'(p, q, I_1, I_2)$ to the sparse approximation errors $\| I_1 - p \|_2$ and $\| I_2 - q \|_2$.
\vspace{4mm}
\begin{proposition}
\quad $E'(p,q,I_1,I_2) \leq E(p,q) + \| I_1 - p \|_2 + \| I_2 - q\|_2$.
\label{prop:assumption_1}
\end{proposition}
\vspace{4mm}
\begin{proof}
We have:
\begin{align*}
E'(p,q,I_1,I_2) & = |d_a(p,q) - d(I_1, I_2)| \\
					  & = |d_a(p,q) - d(p,q) + d(p,q) - d(I_1, I_2)| \\
					  & \leq E(p,q) + | d(p,q) - d(I_1, I_2) |,
\end{align*}
using the triangle inequality. We now show that $|d(p,q) - d(I_1,I_2)| \leq \| I_1 - p \|_2 + \| I_2 - q \|_2$. Let $\eta \in \mathcal{T}$. We  have:
\begin{align*}
\| U(\eta) I_1 - I_2 \|_2 & = \| U(\eta) (p + I_1 - p) - (q + I_2 - q) \|_2 \\ & = \| U(\eta) p - q + U(\eta) (I_1 - p) - (I_2 - q) \|_2.
\end{align*}
Using the triangle inequality, we derive a lower and an upper bound as follows:
\begin{align*}
\| U(\eta) p - q \|_2 - \| U(\eta) (I_1 - p) \|_2 - \| I_2 - q \|_2 \leq \| U(\eta) I_1 - I_2 \|_2 \leq \| U(\eta) p - q \|_2 + \| U(\eta) (I_1 - p) \|_2 + \| I_2 - q \|_2.
\end{align*}
As $U$ is a unitary operator, we have $\| U(\eta) (I_1 - p) \|_2 = \| I_1 - p \|_2$. Hence, rewriting the previous equation, we get:
\begin{align}
\label{eq:pf1_eq1}
\| U(\eta) p - q \|_2 - \| I_1 - p \|_2 - \| I_2 - q \|_2 \leq \| U(\eta) I_1 - I_2 \|_2 \leq \| U(\eta) p - q \|_2  + \| I_1 - p \|_2 + \| I_2 - q \|_2.
\end{align}
Recall that $d(p,q) = \min_{\eta \in \mathcal{T}} \| U ( \eta ) p - q \|_2$ and $d(I_1, I_2) = \min_{\eta \in \mathcal{T}} \| U ( \eta ) I_1 - I_2 \|_2$. Hence, by taking the minimum over all $\eta \in \mathcal{T}$, we obtain $|d(I_1, I_2) - d(p,q)| \leq \| I_1 - p \|_2 + \| I_2 - q \|_2$, which concludes the proof of the proposition.
\end{proof}

When most of the energy of $I_1$ and $I_2$ is captured by $p$ and $q$ (namely when $\| I_1 - p \|_2 + \| I_2 - q \|_2$ is small), the registration errors $E(p,q)$ and $E'(p, q, I_1, I_2)$ are equivalent. We suppose in the rest of this section that this condition is satisfied and we measure the registration error with $E(p,q) = d_a(p, q) - d(p,q)$.
Hence we focus exclusively in this analysis on the penalty induced by restricting the search space $\mathcal{T}$ to $\mathcal{T}_a^{p,q}$, that is the penalty induced by relaxing the problem $(P)$ into the problem $(\hat{P})$ in the above section. 

Before studying the registration performance, we describe additional assumptions on the discretization of the transformation group $\mathcal{T}$ . Recall that the transformation $\eta_0$ optimally aligns $p$ and $q$ in the $L^2$ sense in problem $(P)$. We assume that it satisfies the following assumptions:
\begin{align}
\eta_0 \circ \gamma_i & \in \mathcal{T}_d \text{ for all } i \in \{1, \dots, K\} \label{eq:assumption_2}, \\
\eta_0^{-1} \circ \delta_i & \in \mathcal{T}_d \text{ for all } i \in \{1, \dots, K\},  \label{eq:assumption_3}
\end{align}
where $\mathcal{T}_d$ is the discretization of $\mathcal{T}$ used to construct dictionary $\mathcal{D}$ as given in Eq. (\ref{eq:dictionary}). These hypotheses state that the atoms of $U(\eta_0) p$ and $U(\eta_0^{-1}) q$ belong to the dictionary, where $U(\eta_0) p$ is the optimal alignment of $p$ with $q$ and $U(\eta_0^{-1}) q$ is the optimal alignment of $q$ with $p$. As $\eta_0$ is obviously not known beforehand, it is difficult to verify this assumption in practice. However, we can assume that Eq.   (\ref{eq:assumption_2}) and Eq. (\ref{eq:assumption_3}) hold when the parameter space used to design $\mathcal{D}$ is discretized finely.

Finally, the assumptions in our performance analysis can be summarized as follows:
\begin{align*}
\mathbf{(A_1):} \quad & \| I_1 - p \|_2 + \| I_2 - q \|_2 \approx 0, \\
\mathbf{(A_2):} \quad & \eta_0 \circ \gamma_i \in \mathcal{T}_d, \\
									    & \eta_0^{-1} \circ \delta_i \in \mathcal{T}_d.
\end{align*}

\subsection{Registration performance with exact pattern transformation}
\label{sec:performance_exactly_transformed_patterns}

In our performance analysis, we first consider the special case where $d(p, q) = 0$. This means that there exists a transformation $\eta_0 \in \mathcal{T}$ for which $q = U(\eta_0) p$, i.e., the sparse image approximations can be aligned exactly. We show that in this case, our registration algorithm is able to recover the exact global transformation between $p$ and $q$, as long as any subset of size $2K$ in $\mathcal{D}$ is linearly independent. 
We have the following proposition: 
\vspace{4mm}
\begin{proposition}
\quad Suppose that any subset of size $2K$ in $\mathcal{D}$ is linearly independent. In this case, if $d(p,q) = 0$, then $E(p,q) = 0$.
\label{prop:prop1}
\end{proposition}
\vspace{4mm}
\begin{proof}
If $d(p,q) = 0$, then we have $\sum_{i=1}^K c_i \phi_{\eta_0 \circ \gamma_i} - \sum_{i=1}^K d_i \phi_{\delta_i} = 0$. Thanks to the linear independence of any subset of size $2K$ in $\mathcal{D}$, for any $\gamma_i$ there exists $\delta_j$ such that $\phi_{\eta_0 \circ \gamma_i} = \phi_{\delta_j}$. Indeed, if this is not the case, we could write $\phi_{\eta_0 \circ \gamma_i}$ as a linear combination of $2K-1$ atoms in $\mathcal{D}$ that are all different from $\phi_{\eta_0 \circ \gamma_i}$  and that all belong to $\mathcal{D}$ thanks to assumption $(A_2)$. This contradicts the assumption that any subset of $2K$ atoms in $\mathcal{D}$ is linearly independent. Then, since the mapping $\gamma \mapsto U(\gamma) \phi$ is one-to-one function thanks to our dictionary design assumption, we have $\eta_0 \circ \gamma_i = \delta_j$. Thus, $\eta_0 = \delta_j \circ \gamma_i^{-1} \in \mathcal{T}_a^{p,q}$ and $d_a(p,q) = \min_{\eta \in \mathcal{T}_a^{p, q}} \left\| U(\eta) p - q \right\|_2 = d(p,q) = 0$.
\end{proof}



We can make the following remark about the design of the dictionary. The linear independence assumption guarantees that, when two $K$-sparse signals are equal, they have at least one atom in common\footnote{The linear independence of any subset of size $2K$ in the dictionary actually guarantees a stronger result: it guarantees that any $K$-sparse signal has a unique decomposition in $\mathcal{D}$ \cite{donoho2003optimally}. In other words, it guarantees that when two $K$-sparse signals are equal, \textit{all} the atoms are equal.}. If this condition is violated, the patterns $U(\eta_0) p$ and $q$ can have several decompositions in the dictionary with disjoint supports. In this case, all the features of the transformed pattern $U(\eta_0) p$ and $q$ are distinct, which generally lead to $d_a(p,q) \neq d(p,q)$. Note that this assumption appears in many problems related to overcomplete dictionaries since it guarantees the uniqueness of $K$-sparse decompositions \cite{davenport2011introduction, Candes05, tropp2004greed}. 

Finally, since Proposition \ref{prop:prop1} ensures that $E(p,q) = 0$ for an exactly transformed pattern, and we have $E'(p, q, I_1, I_2) \approx E(p,q)$ when the sparse approximation errors are not too large (Assumption $(A_1)$), we can guarantee that the registration error $E'(p, q, I_1, I_2)$ is small in this case.

\subsection{Registration performance in the general case}

\subsubsection{Bound on the registration error}

We now study the performance of our registration algorithm in the general case. The previous result only applies to an ideal scenario since the condition $d(p,q) = 0$ is rarely satisfied in practice. There is usually some slight innovation between the images (other than a transformation in $\mathcal{T}$), which result in a distance $d(p, q)$ that is non-zero. In addition, even when the original images are exactly related by a global transformation (i.e., $d(I_1, I_2) = 0$), there is no guarantee that the sparse approximations are can be perfectly aligned (i.e., $d(p,q) = 0$) due to the discretization of the dictionary. 

We study the general case where where the sparse image approximations $p$ and $q$ have differences that cannot be explained by a global geometric transformation in $\mathcal{T}$. In more detail, when $c$ and $d$ denote respectively the coefficient vectors for patterns $p$ and $q$ following Eq. (\ref{eq:sparseapprox}), we suppose that there exists a real number $\epsilon > 0$ such that $d(p,q) < \epsilon \sqrt{ \| c \|_2^2 + \| d \|_2^2 }$. The quantity $\epsilon$ therefore measures the normalized innovation between $p$ and $q$.


We now turn to the main result of our paper, which is formulated in Theorem \ref{th:main_theorem}. This result relates the error of the registration algorithm in Algorithm \ref{alg:registration_algo} to the properties of the dictionary, namely the \emph{Robust Linear Independence (RLI)} and the \emph{transformation inconsistency}. It reads as follows.

\vspace{4mm}
\begin{theorem}
\label{th:main_theorem}
\quad
If $d(p,q) < \epsilon \sqrt{\| c \|_2^2 + \| d \|_2^2}$ with $\epsilon > 0$, then:
\begin{align*}
E(p,q) \leq \alpha \rho \min \left( \| c \|_1, \| d \|_1 \right), 
\end{align*}
when $\mathcal{D}$ is $\left(2K, \epsilon, \alpha \right)$-RLI for some $\alpha \in [0, \sqrt{2})$, and $\rho$ is the transformation inconsistency of $\mathcal{D}$.
\end{theorem}
\vspace{4mm}

Theorem \ref{th:main_theorem} shows that robust linear independence with a small $\alpha$ and a small transformation inconsistency are key properties of the dictionary in order to guarantee the success of our algorithm. The RLI property can be thought as an extension of the linear independence assumption to the case where $d(p,q) \neq 0$. Specifically, it guarantees the existence of two approximately similar features in $U(\eta_0) p$ and $q$ when $d(p,q)$ is small. The transformation inconsistency captures the fact that geometrical transformations have a  different effect on distinct atoms in the dictionary. We defer the proof of Theorem \ref{th:main_theorem} to Appendix \ref{app:proof_main_theorem}, and we study in details in the rest of this section the novel RLI and transformation inconsistency properties.

\subsubsection{Robust linear independence}

We study now in more detail the novel dictionary properties. We first show that the linear independence assumption introduced in Section \ref{sec:performance_exactly_transformed_patterns} is no longer sufficient to bound the registration performance in the case where $d(p,q) \neq 0$ (but close to zero). To see this, we construct a linearly independent dictionary $\mathcal{D}$ and two sparse patterns $p$ and $q$ for which $d(p,q)$ can be made arbitrarily close to zero (i.e., $\epsilon \rightarrow 0$) yet the registration error is large. As illustrated in Fig. \ref{fig:non_rli}, we consider a dictionary $\mathcal{D}$ containing four square atoms and an additional big square atom parametrized by its position $\kappa$ with respect to $\phi_{\gamma_1}$. Clearly, when $\kappa \neq 0$, the dictionary $\mathcal{D}$ is linearly independent since one cannot write an atom as a linear combination of the four other atoms. We consider the patterns $p = \frac{1}{2} \sum_{i=1}^4 \phi_{\gamma_i}$ and $q = \phi_{\gamma_5}$. 
When $\kappa$ is small, the transformation that best aligns $p$ and $q$ is the identity transformation\footnote{If we look among all possible transformations, the optimal transformation $\eta_0$ is a translation that exactly aligns $p$ and $q$. However, this transformation does not satisfy  the assumptions in Eq. (\ref{eq:assumption_2}) and (\ref{eq:assumption_3}). To illustrate the main issue here, we consider only transformations that satisfy these assumptions. For small $\kappa$, the optimal transformation is therefore the identity.}. 
All relative transformations between features in $p$ and $q$ are however dilations composed with translations, which result in an estimated transformation $\hat{\eta}$ in our algorithm that is significantly different from the identity. Hence we obtain a large registration error $d_a(p,q) - d(p,q)$ in this example. This example shows that the linear independence assumption defined in Section \ref{sec:performance_exactly_transformed_patterns} is fragile: it does not allow us to bound the registration error even when $d(p,q)$ is very small. One needs a more \emph{robust} condition 
in order to guarantee a small registration error even in cases where the innovation between images is small (but nonzero). 

\begin{figure}[ht]
\centering
\includegraphics[width=0.25\textwidth]{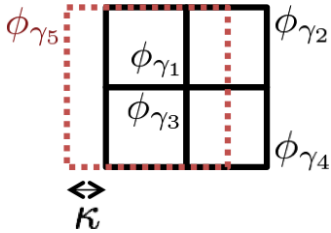}
\caption{\label{fig:non_rli}Example of a linearly independent dictionary $\mathcal{D}$ that induces a large registration error $d_a(p,q) - d(p,q)$, when $p = 1/2 (\phi_{\gamma_1} + \phi_{\gamma_2} + \phi_{\gamma_3} + \phi_{\gamma_4})$ and $q = \phi_{\gamma_5}$. We note that for $\epsilon > \sqrt{\kappa}$, this dictionary is not RLI unless $\alpha \geq 1$. Note that $ \epsilon$ can be made very small since $\kappa$ can be chosen to be any positive real number. }
\end{figure}
Therefore, we propose to extend the notion of linear independence to a novel property called \emph{robust linear independence (RLI)} to characterize sets of vectors. It is formally defined as follows.

\vspace{4mm}

\begin{definition}
\quad Let $\left( H, \| \cdot \| \right)$ be a normed space and $K \geq 1$. A family of vectors $(v_1, \dots, v_K) \in H^K$ is $(\epsilon, \alpha)$-\emph{robustly linearly independent} (RLI) if the following implication holds for any vector $a \in \mathbb{R}^K$: 
\begin{align}
\label{eq:rli}
\left\| \sum_{i=1}^K a_i v_i \right\|  < \epsilon \| a \|_2 \implies \exists i, j \text{ with } a_i, a_j \neq 0, \left\| \frac{a_i v_i}{\left\| a_i v_i \right\|} + \frac{a_j v_j}{\left\| a_j v_j \right\|} \right\| \leq \alpha.
\end{align}
\label{def:rli}
\end{definition}
 \vspace{4mm}

In other words, when $\epsilon$ and the parameter $\alpha$ are small, any linear combination of vectors that nearly vanishes in a RLI vector set contains at least two vectors that approximately cancel each other. 

We now discuss the relation between RLI and linear independence. While linear independence prevents having collinear vectors, it is natural in our registration framework to allow collinear vectors in the dictionary since they represent essentially the same feature. Specifically, as the underlying transformation parameter of collinear atoms is the same, selecting one atom or the other is not important for the purpose of registration\footnote{Note that this is in contrast to recovery problems (e.g., compressed sensing) where collinear vectors (in the measurement matrix) are not allowed, since it will not be possible then to recover the active component of the signals.}. The notion of linear independence where collinear vectors are allowed can be written as follows. For any $a \in \mathbb{R}^K$ such that $a \neq 0$,   
\begin{align*}
\sum_{i} a_i v_i = 0 \implies \exists i, j \text{ with } a_i, a_j \neq 0, \frac{a_i v_i}{\left\| a_i v_i \right\|} + \frac{a_j v_j}{\left\| a_j v_j \right\|} = 0.
\end{align*}
Note that this essentially corresponds to the notion of robust linear independence in the case where $\alpha, \epsilon = 0$. Since we want to study the behavior of the algorithm for \emph{nonzero} innovation between the images, we naturally extend the notion of linear independence (where collinear vectors are allowed) to Definition \ref{def:rli}; if a linear combination of vectors has a small magnitude (where $\epsilon$ quantifies the magnitude), there exist two vectors that approximately cancel each other (where $\alpha$ quantifies this approximation). Note that, for a fixed $\alpha$, the RLI gets harder to satisfy for a larger $\epsilon$. In addition, for a fixed $\epsilon$, the condition is harder to satisfy for a smaller $\alpha$. 

The following toy example illustrates the notion of robust linear independence in $\mathbb{R}^3$. 

\begin{example}
\quad Consider the setting of Figure \ref{fig:rli_R3} with $\theta = \varphi = \pi/20$. Then, for $\epsilon = 0.2$, we have:
\begin{enumerate}
\item $(e_1, e_2, v)$ is RLI with $\alpha = 0.2$.
\item $(e_1, e_2, v')$ is \emph{not} RLI  unless $\alpha \geq 0.78$.
\end{enumerate}
\begin{figure}[ht]
\centering
\includegraphics[width=0.25\textwidth]{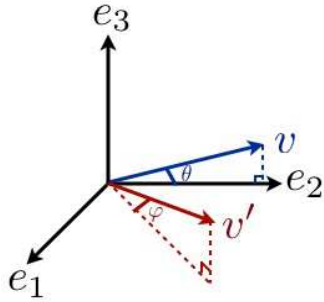}
\caption{\label{fig:rli_R3} Illustration of robust linear independence property in $\mathbb{R}^3$.} 
\end{figure}
\label{ex:example_rli_R3}
\end{example}
The proof of Example \ref{ex:example_rli_R3} is straightforward from simple trigonometry.
The set of vectors $(e_1, e_2, v)$ has a better behavior in terms of robust linear independence than $(e_1, e_2, v')$. The underlying reason is that $v$ is very close to the vector $e_2$ (i.e., $\| e_2 - v \|_2$ is close to zero), while $v'$ is close to a \emph{linear combination} of $e_1$ and $e_2$ (but not to $e_1$ or $e_2$).  While it is acceptable to have vectors that are close to each other, the RLI property prevents having a vector that is close to a linear combination of the other vectors. This can also be readily seen in the example of Fig. \ref{fig:non_rli}



The definition of robust linear independence can be extended to dictionaries as follows.

\vspace{1mm}
\begin{definition}
\quad
A dictionary $\mathcal{D}$ is $(K, \epsilon, \alpha)$-RLI if any subset of size $K$ in $\mathcal{D}$ is $(\epsilon, \alpha)$-RLI.
\end{definition}

The dictionary in the example of Fig. \ref{fig:non_rli} is \emph{not} $K$-RLI for $K=5$, with small $\epsilon$ (unless $\alpha$ is large). Indeed, by choosing a vector of coefficients $a = [0.5, 0.5, 0.5, 0.5, -1]^T$, we obtain $\left\| \sum_i a_i v_i \right\|_2 \approx 0$, yet $\alpha = 1$. Note that the RLI property on the dictionary has to be satisfied in order to obtain a good registration performance, as it ensures the existence of two approximately similar features (in the $L^2$ sense) in $U(\eta_0) p$ and $q$, when $d(p,q)$ is small. 


We study now in more detail the RLI property on dictionaries. In particular, we examine the main difference between RLI and the well known Restricted Isometry Property (RIP) \cite{Candes05}. The restricted isometry condition assumes that a collection of vectors behaves almost like an orthonormal system but only for sparse linear combinations. Specifically, the RIP with constant $\delta_K$ implies that any linear combination of $K$ elements in the dictionary satisfies:
\begin{align*}
\left\| \sum_{i=1}^K a_i v_i \right\|^2 \geq (1 - \delta_K) \left\| a \right\|_2^2  
\end{align*}
By imposing a RIP property on the dictionary $\mathcal{D}$ with $\delta_K \ll 1$ , the norm of any sparse linear combination of atoms is guaranteed to be large (i.e., larger than $\sqrt{(1 - \delta_K) \| a \|_2}$). In our case, contrarily to the RIP, we are interested in linear combinations of atoms that nearly vanish. The RLI property imposes in this case the existence of two atoms that approximately cancel each other in the signal support. Consequently, RLI can be seen as a weak form of RIP, where we allow the norm of linear combinations to be close to zero provided that two atoms approximately cancel each other in the sense of Eq. (\ref{eq:rli}). 
In particular, any dictionary $\mathcal{D}$ that satisfies the RIP property with a parameter $\delta_K$ will be ($K$, $\sqrt{1-\delta_K}$, $0$)-RLI. Indeed, since $\left\| \sum_{i=1}^K a_i v_i \right\| \geq \sqrt{1 - \delta_K} \| a \|_2$ holds for any subset of $K$ dictionary elements, the left hand side of Eq. (\ref{eq:rli}) cannot be satisfied when $\epsilon = \sqrt{1 - \delta_K}$. 

Let us consider a simple example to compare the new RLI property with the common ways of characterizing dictionaries, namely, the coherence \cite{tropp2004greed} and the restricted isometry property \cite{Candes05}\footnote{Even though the definitions of RIP and coherence are originally for vectors in $\mathbb{R}^N$, we consider here a straightforward extension of the definitions of RIP and coherence to the case where vectors are in $L^2$.}.

\vspace{4mm}
\begin{example}[Dictionary of translated box functions]
\label{prop:rli_example}
Let $H = L^2(\mathbb{R})$ and define the box function $$v(t) = \begin{cases}  1, & \text{if } t \in [0,1] \\
  																													0, & \text{otherwise.}
  																					                    \end{cases}$$ 
We consider the infinite-size dictionary $\mathcal{D}_{\text{box}} = \{ T_{\tau} v = v_{\tau}: \tau \in \mathbb{R} \}$, where $T_\tau$ is the translation operator by $\tau$. The dictionary has the following properties:

\begin{itemize}
\item $\mathcal{D}_{\text{box}}$ is RIP with a constant $\delta_K(\mathcal{D}_{\text{box}})$ equal to 1, for any $K \geq 2$.
\item The coherence of $\mathcal{D}_{\text{box}}$ is equal to 1.
\item $\mathcal{D}_{\text{box}}$ is $\left(K, \epsilon, \epsilon \sqrt{\frac{2}{3} (4^K-1)}\right)$-RLI for $K \geq 1$ and $\epsilon \in \left(0, \sqrt{\frac{3}{4^K-1}} \right)$.
\end{itemize}

\end{example}
\vspace{4mm}

As the proof of the robust linear independence of $\mathcal{D}_{\text{box}}$ is rather technical and not essential to the main understanding of the paper, it is given in Appendix \ref{app:proof_example_rli}.

Even if the dictionary $\mathcal{D}_{\text{box}}$ hardly satisfies the RIP and is highly coherent, it is still an interesting one in our framework. Indeed, it satisfies the key property that two sparse signals that are close in the $L^2$ sense have at least two approximately similar features. When applied to our registration problem, this guarantees the existence of two features that are related approximately by a transformation $\eta_0$ in the $L^2$ sense\footnote{More precisely, this means that there exists a $\gamma_i$ and a $\delta_j$ such that $\| U(\eta_0) \phi_{\gamma_i} - \phi_{\delta_j} \|_2$ is small.} when $d(p,q)$ remains small. This property is at the core of our registration algorithm since we infer the global transformation by looking at the relative transformations between the features.

We finally stress the differences between the proposed RLI property and other dictionary properties as the RIP, coherence or more recently the properties introduced in \cite{candes2011compressed, giryes2012greedy, peleg2012performance}. While the latter properties are specifically designed for the task of  \textit{signal recovery}, the proposed RLI property is introduced in the context of image registration. This explains in particular why a dictionary can be well-behaved in terms of RLI property despite having coherent atoms. In contrast, coherent columns are forbidden in the context of recovery problems (e.g., compressed sensing) as it is then difficult to distinguish between similar components in the signal reconstruction. 
 




%


\subsubsection{Transformation inconsistency}

The second dictionary property that is important to study the performance of our algorithm is \emph{the transformation inconsistency}, which measures the difference in the effect of the same transformation on distinct atoms in the dictionary. It is formally defined as follows for parametric dictionaries given by Eq. (\ref{eq:dictionary}).

\vspace{4mm}

\begin{definition}
\quad The transformation inconsistency $\rho$ of a parametric dictionary $\mathcal{D}$ is equal to:
\begin{align*}
\rho = \sup_{\gamma, \gamma' \in \mathcal{T}_d} \sup_{\eta \in \mathcal{T} \backslash \{\mathbb{I} \}} \left\{ \frac{\| U(\eta) \phi_{\gamma'} - \phi_{\gamma'} \|_2}{\| U(\eta) \phi_{\gamma} - \phi_{\gamma} \|_2} \right\},
\end{align*}
\label{def:inconsistency}
\end{definition}
where $\mathbb{I}$ is the identity transformation. The transformation inconsistency $\rho$ is always larger than or equal to $1$. Furthermore, when $\mathcal{T}$ is commutative, the transformation inconsistency takes it minimal value and is equal to $1$. Indeed, for any $\gamma, \gamma'$ in $\mathcal{T}_d$ and $\eta \in \mathcal{T}$, we have:
\begin{align*}
\frac{\| U(\eta) \phi_{\gamma'} - \phi_{\gamma'} \|_2}{\| U(\eta) \phi_{\gamma} - \phi_{\gamma} \|_2} & = \frac{\| U(\gamma') ( \phi_{\eta} - \phi ) \|_2}{\| U(\gamma) ( \phi_{\eta} - \phi ) \|_2}  = \frac{\| \phi_{\eta} - \phi \|_2}{\| \phi_{\eta} - \phi \|_2}  = 1.
\end{align*}
Hence, taking the supremum over all $\eta \in \mathcal{T}$ and atoms $\gamma, \gamma'$ in $\mathcal{T}_d$ results in having $\rho = 1$. This is expected since when $\mathcal{T}$ is commutative, a fixed transformation acts on all atoms similarly. 

On the other hand, a large value of the transformation inconsistency $\rho$ (i.e., $\rho \gg 1$) means that there exist two atoms in the dictionary that are affected in a very different way when they are subject to the same transformation. The transformation inconsistency plays a key role in our registration algorithm. Indeed, as the global transformation between two sparse patterns is estimated from one of the relative transformations between features, it is preferable that transformations act in a similar way on all the features of the sparse patterns for more consistent registration. That means that dictionaries with small transformation inconsistency provide better registration performance. 

In order to outline the importance of this novel property in our registration framework, we give a few illustrative examples of dictionaries with different transformation inconsistency parameters.

\vspace{4mm}

\begin{example}[Dictionary with quasi isotropic mother function, $\mathcal{T} = SE(2)$]
We consider $\mathcal{T}$ to be the Special Euclidean group ($\mathcal{T} = SE(2)$). That is, $\mathcal{T}$ accounts for translations, rotations and combinations of those. We consider an ellipse-shaped mother function $\phi$ as shown in Figure \ref{fig:example1_rho} (a) with anisotropy $r = \frac{l}{L}$. 
Then, we suppose for the sake of simplicity that $\mathcal{T}_d = \mathcal{T}$ (i.e., the dictionary is built by applying all transformations $\gamma \in \mathcal{T}$ to the generating function $\phi$). 

We illustrate in Fig \ref{fig:example1_rho} (b) the effect of transformation $\eta$, which is a simple rotation, on two different atoms with parameters $\gamma$ and $\gamma'$ positioned at different points in the 2D plane. While the rotation of the atom parametrized by $\gamma$ induces a very slight change on it (when $r \approx 1$), the same rotation applied on the atom $\phi_{\gamma'}$ changes completely its position. This is due to the fact that translations and rotations do not commute. Hence, the transformation $\eta$ has a very different impact on atoms $\phi_{\gamma}$ and $\phi_{\gamma'}$, and we get $\rho \rightarrow \infty$ from Definition \ref{def:inconsistency}. Therefore, when the generating function $\phi$ approaches isotropy, the transformation inconsistency grows to infinity.

In this example, our registration algorithm is not guaranteed to have a small error. To illustrate it, let us consider the patterns $p$ and $q$ illustrated in Fig. \ref{fig:example1_rho} (c), which are each composed of two atoms whose coefficients are all equal. The distance $d(p,q)$ between the patterns can be made arbitrarily small with a generating function that is close to isotropic (i.e., $r \rightarrow 1$) while the minimal distance $d_a(p,q)$ in our algorithm remains large. Indeed, since our algorithm considers only relative transformations between pairs of atoms, the estimated global transformation between the patterns can only be equal to a combination of a translation and rotation of $\frac{\pi}{2}$. However, when $r \approx 1$, the optimal transformation is clearly the identity, which cannot be selected with our algorithm: this results in a large registration error $d_a(p,q) - d(p,q)$. Note that the error here is entirely related to the fact that the transformation inconsistency $\rho$ is large, and not to the RLI property since the dictionary under consideration here is robustly linearly independent for small values of the sparsity $K$.
\label{ex:example_rho1}

\end{example}

\vspace{4mm}

\begin{figure}[ht]
\centering
\includegraphics[width=0.7\textwidth]{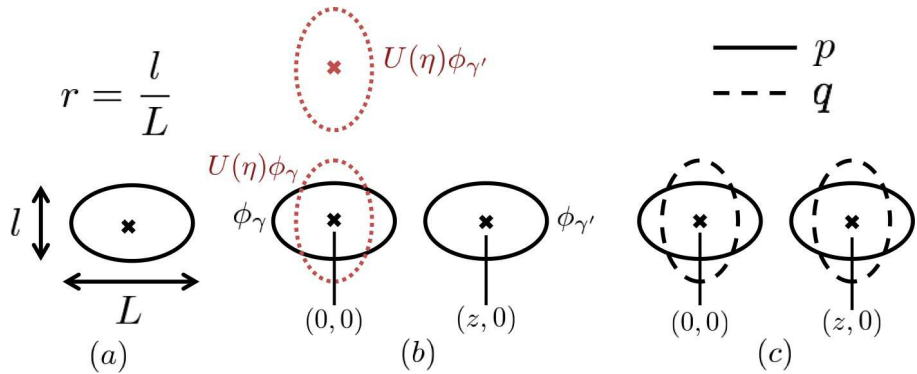}
\caption{\label{fig:example1_rho} Example of a dictionary where the transformation inconsistency $\rho$ is large. (a): Mother function of the dictionary with anisotropy $r = \frac{l}{L}$. (b): Atoms $\phi_{\gamma}$, $\phi_{\gamma'}$, and a transformation $\eta$ that leads to a large transformation inconsistency $\rho$ . (c): Examples of patterns $p$ (atoms represented with solid line) and $q$ (atoms represented with dashed line) where our algorithm has a large registration error $d_a(p,q) - d(p,q)$.}
\end{figure}

\vspace{4mm}
\begin{example}[Dictionary built on an elongated mother function, $\mathcal{T} = SE(2)$]
Similarly to the previous example, we consider the transformation group $\mathcal{T} = SE(2)$ and that $\mathcal{T}_d = \mathcal{T}$. However, the dictionary is now built on an elongated mother function as shown in Fig \ref{fig:example2_rho} (a). As in the previous example, we can make the transformation inconsistency $\rho$ very large by taking elongated atoms (large $L$) 
and a transformation $\eta$ that is a small translation, as shown in Fig \ref{fig:example2_rho} (b). It is again possible to construct an example where the registration algorithm performs poorly (see Fig \ref{fig:example2_rho} (c)) : the set $\mathcal{T}_a^{p,q}$ of transformations between features in each sparse pattern contains only translations and rotations of $\frac{\pi}{2}$. Therefore, any candidate transformation $\eta \in \mathcal{T}_a^{p,q}$ results in a large value of the global registration error term $\| U (\eta) p - q \|_2$; the optimal global transformation is the identity in this case, which leads to a small value of the minimal distance $d(p,q)$ between the patterns when $L$ is large.
\label{ex:example_rho2}
\end{example}

\begin{figure}[ht]
\centering
\includegraphics[width=0.7\textwidth]{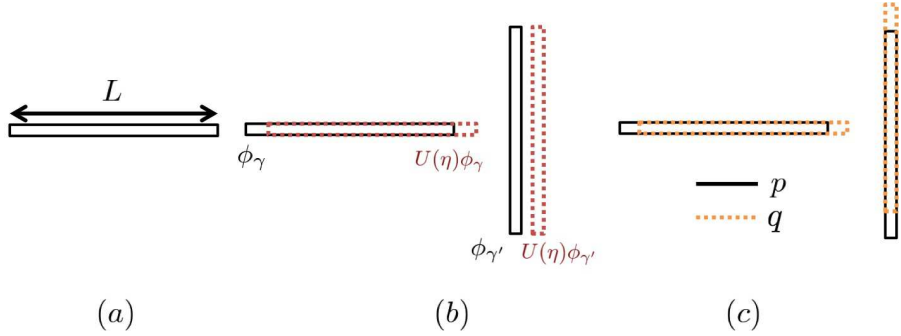}
\caption{\label{fig:example2_rho} Example of a dictionary where the transformation inconsistency $\rho$ is large. (a): Mother function of the dictionary, where $L$ is the length of the atom. (b): Atoms $\phi_{\gamma}$, $\phi_{\gamma'}$, along with the results of a transformation $\eta$ that causes the transformation inconsistency $\rho$ to be large. (c): Examples of patterns $p$ (atoms represented with solid line) and $q$ (atoms represented with dashed line) where our algorithm has a large registration error $d_a(p,q) - d(p,q)$.}
\end{figure}
\vspace{4mm}

To be complete, we should note that the one-to-one mapping assumption defined in Section \ref{sec:probform} for the function $\gamma \mapsto U(\gamma) \phi$ is not satisfied in Example \ref{ex:example_rho1} and Example \ref{ex:example_rho2}, since $\phi$ has a rotational symmetry of $\pi$. In this case, a slightly more complicated definition of the transformation inconsistency $\rho$ has to be made to avoid having $\rho = \infty$ (with the definition of $\rho$ given in Definition \ref{def:inconsistency}, we obtain $\rho = \infty$ by setting $\eta$ to be a rotation of $\pi$, $\gamma$ to be the identity and choosing any $\gamma'$ different from $\gamma$). The main intuitions of the transformation inconsistency $\rho$, as defined in Definition \ref{def:inconsistency} however hold when $\phi$ has a finite number of symmetries. We study in detail the generalization of the transformation inconsistency $\rho$ to the case where $\phi$ has symmetries in $\mathcal{T}$ in Appendix \ref{app:not_bijective}. 

\vspace{4mm}
\begin{example}[Dictionary built with translation and isotropic dilations, $\mathcal{T} = \mathcal{T}_d = \mathbb{R}^2 \times \mathbb{R}^+ _{*}$]
In this example, we let $\mathcal{T}$ to be the group of translations and isotropic dilations. The generating function of the dictionary could have any form, as long as its support is much smaller than the dimension of the image. For example, we can choose a circle-shaped mother function, as depicted in Fig \ref{fig:example3_rho} (a). Then, we consider the scenario where the two atoms $\phi_{\gamma}$ and $\phi_{\gamma'}$ are separated by $z$ (where $z$ is considered to be very large) as illustrated in Fig.\ref{fig:example3_rho} (b). A transformation $\eta$ that consists of a small isotropic dilation has a very different effect on both atoms since translations and dilations do not commute. In particular, the transformation $\eta$ applied to $\phi_{\gamma'}$ results in an atom that has no intersection with $\phi_{\gamma'}$, while the same transformation has almost no effect on $\phi_{\gamma}$, i.e., $U(\eta) \phi_{\gamma} \approx \phi_{\gamma}$. Thus, the transformation inconsistency is very high and $\rho \approx \infty$ according to Definition \ref{def:inconsistency}. In Fig.  \ref{fig:example3_rho} (c), we illustrate why this may cause a problem in our registration algorithm: we consider the two sparse patterns $p$ and $q$ composed of two features each, where the coefficients of all the atoms are equal. It is not hard to see that the optimal global transformation between both patterns is the identity. At the same time, our algorithm can only estimate a global transformation that is a dilation (combined possibly with a translation) since all transformations between pairs of atoms in $p$ and $q$ consist in combinations of dilation and translation. 
\label{ex:example_rho3}
\end{example}

\begin{figure}[ht]
\centering
\includegraphics[width=0.7\textwidth]{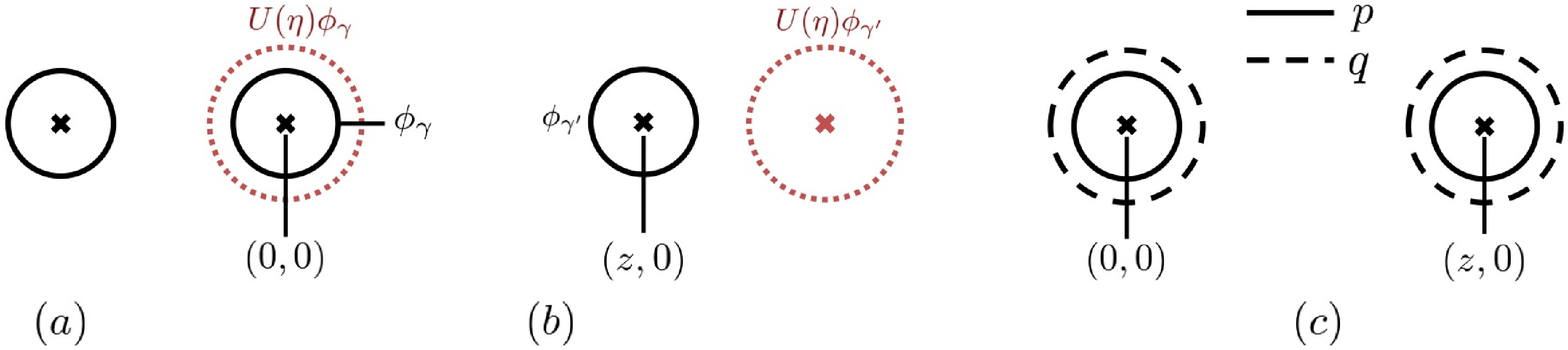}
\caption{\label{fig:example3_rho} Example of a dictionary where the transformation inconsistency $\rho$ is large. (a): Mother function of the dictionary (b): Atoms $\phi_{\gamma}$, $\phi_{\gamma'}$, and transformation $\eta$ that causes $\rho$ to be large. (c): Examples of patterns $p$ (atoms represented with solid line) and $q$ (atoms represented with dashed line) where our algorithm has a large registration error $d_a(p,q) - d(p,q)$.}
\end{figure}

Overall, the above examples suggest that, whenever the transformation inconsistency of the dictionary is large, one may construct an example where our registration algorithm approximates poorly the transformation invariant distance. It is worth mentioning that even though the previous examples consider localized atoms with finite support, our approach is not constrained to such atoms. In the general setting where $\mathcal{T}$ is any transformation group (and $\mathcal{T}_d = \mathcal{T}$ for the sake of simplicity), such example of failure could be constructed as follows. The basic idea is to build two patterns $p$ and $q$ of the form $p = \phi_{\gamma} + \phi_{\gamma'}$ and $q = U(\eta_1) \phi_{\gamma} + U(\eta_2) \phi_{\gamma'}$ for which: (i) $p \approx q$, (ii) $\| U(\eta_1) p - q \|_2$ and $\| U(\eta_2) p - q \|_2$ are large (with respect to $\| p - q \|_2$). The optimal transformation between $p$ and $q$ is then simply the identity, whereas the transformations considered in our algorithm (namely $\eta_1$ and $\eta_2$, along with $\eta_1 \circ \gamma \circ \gamma'^{-1}$ and $\eta_2 \circ \gamma' \circ \gamma^{-1}$) result in a poor registration performance as they all differ from the identity transformation.  

In more details, when $\rho \gg 1$ we know that there exist two atoms $\phi_{\gamma}$ and $\phi_{\gamma'}$ with $\gamma \in \mathcal{T}$ and $\gamma' \in \mathcal{T}$, along with a transformation $\eta_1$ for which $\| U ( \eta_1 ) \phi_{\gamma} - \phi_{\gamma} \|_2 \approx 0$ while $\| U ( \eta_1 ) \phi_{\gamma'} - \phi_{\gamma'} \|_2$ is large. By posing $\eta_2 = (\gamma' \circ \gamma^{-1}) \circ \eta_1 \circ (\gamma \circ (\gamma')^{-1})$, we get that $\| U(\eta_2) \phi_{\gamma'} - \phi_{\gamma'} \|_2 = \| U(\eta_1) \phi_{\gamma} - \phi_{\gamma} \|_2 \approx 0$. Hence, the norm $\| p - q \|_2$ is necessarily small since $\| p - q \|_2 = \| \phi_{\gamma} + \phi_{\gamma'} - U(\eta_1) \phi_{\gamma} - U(\eta_2) \phi_{\gamma'} \|_2 \leq \| \phi_{\gamma} - U(\eta_1) \phi_{\gamma} \|_2 + \| \phi_{\gamma'} - U(\eta_2) \phi_{\gamma'} \|_2$. Besides, we know by construction that $\| U ( \eta_1 ) \phi_{\gamma'} - \phi_{\gamma'} \|_2$ is large and $\| U ( \eta_2 ) \phi_{\gamma} - \phi_{\gamma} \|_2$ is also generally large since the group $\mathcal{T}$ is non commutative. This gives us, in general, large values of $\| U (\eta_1) p - q \|_2$ and $\| U(\eta_2) p - q \|_2$. This construction shows that, when the dictionary has a large inconsistency parameter, one can find patterns for which the registration algorithm fails to recover the right global transformation. 

In general, the above examples show that it is better to choose a dictionary with a small transformation inconsistency (i.e., $\rho$ small) to have good registration performance irrespectively of the patterns to be aligned.
 
\vspace{4mm}
The performance of the registration algorithm depends on the transformation inconsistency as well as on the robust linear independence of the dictionary, as shown in Theorem \ref{th:main_theorem}. The success of our registration algorithm for \textit{all} sparse signals in the dictionary is guaranteed when the RLI and transformation inconsistency conditions are satisfied. Note that the conditions on the dictionary properties are essentially tight, as one can construct an example where our algorithm fails whenever one of the parameters is large enough. The performance bound should be interpreted more in a qualitative way than a quantitative way. It provides two rather intuitive conditions for our algorithm to provide low registration error. In order to use this bound quantitatively, one has however to be able to compute explicitly the newly defined properties on generic dictionaries. We outline here the fact that such a bound could not have been established with traditional measures for characterizing dictionaries, namely coherence or restricted isometry property constant.  Finally, we remark that the result in Theorem \ref{th:main_theorem} can be used to bound the registration error $E'(p, q, I_1, I_2)$ thanks to Proposition \ref{prop:assumption_1}. The price to pay in this case is the approximation error $\| I_1 - p \|_2 + \| I_2 - q \|_2$. 

\section{Image registration experiments}
\label{sec:experimental_results} \label{sec:experiments}

In this section, we evaluate the performance of our algorithm in image registration experiments. We first describe the implementation choices in our registration algorithm. Then, we study its performance for different dictionaries and put the results in perspective with the theoretical guarantees in Section \ref{sec:theoretical_analysis}. Then, we present illustrative image registration and classification experiments with simple test images and handwritten digits. Finally, we provide some simple comparisons with baseline registration algorithms with simple features from the computer vision literature. 

\subsection{Algorithm implementation}

In all the experiments of Section \ref{sec:illustrative_examples}, we focus on achieving invariance to translation, rotation and scaling. Invariance to these transformations is indeed considered to be a minimal requirement in invariant pattern recognition. These three operations generate the \emph{group of similarities} that we denote by $\mathcal{T} = SIM(2)$. Any element in $\mathcal{T}$ is therefore indexed by 4 parameters: a translation vector $b = (b_x, b_y)$, dilation $a$ and rotation parameter $\theta$.  We describe now the sparse approximation algorithm and the dictionary design used in our experiments.

\subsubsection{Sparse approximation algorithm}

There are many methods to construct sparse approximations of images. In our experiments, we use a modified implementation of the Matching Pursuit (MP) \cite{Mallat93} algorithm, as MP is a pretty simple algorithm that works relatively well in practice. It is an iterative algorithm that successively identifies the atoms in $\mathcal{D}$ that best match the image to be approximated. More precisely, MP iteratively computes the correlation between the atoms in $\mathcal{D}$ and the signal residual, which is obtained by subtracting the contributions of the previously chosen atoms from the original image. At each iteration, the atom with the highest correlation is selected and the residual signal is updated. While the standard MP algorithm solves the sparse approximation problem without positivity constraint on the coefficients, we propose a slightly modified algorithm (that we call Non negative Matching Pursuit (NMP)) in order to select atoms that have the highest positive correlation with the residual signal. This choice is driven by the objective of having a part-based signal expansion, where each feature participate to constructing the signal representation. The NMP algorithm is formally defined in Algorithm \ref{alg:mp}. 

\begin{algorithm}[ht]
\algorithmicrequire{\hspace{1mm} image $I$, sparsity $K$, dictionary $\mathcal{D}$.}  \\ 
\algorithmicensure{\hspace{1mm} coefficients $c$, support $\Gamma$.}
\begin{algorithmic}
\STATE $\mathbf{1.}$ Initialization of the residual: $r_0 \leftarrow I$ and support: $\Gamma \leftarrow \emptyset$.
\STATE $\mathbf{2.}$ While $1 \leq i \leq K$, do:
\STATE \hspace{2mm}$\mathbf{2.1}$ Selection step:
	\begin{align*}
		\gamma_i & \leftarrow \argmax_{\gamma \in \mathcal{T}_d} \scalprod{r_{i-1}}{\phi_{\gamma}} \\
		\Gamma & \leftarrow \Gamma \cup \{ \gamma_i \}.
	\end{align*}
\STATE \hspace{2mm}$\mathbf{2.2}$ If $\scalprod{r_{i-1}}{\phi_{\gamma_i}} \leq 0$, go to $\mathbf{3}$.
\STATE \hspace{2mm}$\mathbf{2.3}$ Update step:
	\begin{align*}
			c_{i} & \leftarrow \scalprod{r_{i-1}}{\phi_{\gamma_i}} \\
			r_{i} & \leftarrow r_{i-1} - \scalprod{r_{i-1}}{\phi_{\gamma_i}} \phi_{\gamma_i}
	\end{align*}
\STATE $\mathbf{3.}$ Return $c$, $\Gamma$. 
\end{algorithmic}
\caption{Non negative Matching Pursuit (NMP) for feature extraction}
\label{alg:mp}
\end{algorithm}

One way to choose the sparsity $K$ consists in controlling the approximation error of $I_1$ and $I_2$. Specifically, we can impose a stopping criterion in the NMP algorithm of the form $\| r_K \|_2 \leq e$ where $r_K$ is the residual at iteration $K$ and $e$ is a fixed threshold controlling the approximation error. When $e$ is chosen to be small enough, this guarantees a relatively small sparse approximation error. 

Note that the complexity of NMP is governed by the selection step, hence $O(K|\mathcal{D}|)$ operations need to be performed. Besides, the complexity of solving $(\hat{P})$ using Algorithm \ref{alg:registration_algo} is $O(K^2 N)$ with  $N = \max(N_1, N_2)$ with $N_1$ and $N_2$ respectively the dimensions of the discretized images corresponding to $p$ and $q$. 
Therefore, if the sparse approximation step is necessary for registration, the complexity of the overall registration algorithm is $O(K |\mathcal{D}| + K^2 N)$. Depending on the factor $\frac{|\mathcal{D}|}{K N}$, the complexity might be governed by either step of the algorithm. Overall, the choice of $K$ results from a trade-off between approximation error (hence registration performance) and computational complexity. 
Finally, note that in applications involving the registration of a \emph{test} image with possibly many \emph{training} images, the sparse approximations of the training images are computed offline. Hence, only the sparse approximation of the test image needs to be computed during the test phase.

\subsubsection{Choice of the dictionary}
 
We discuss now the choice of the dictionary $\mathcal{D}$ that is used in our experiments. As pointed out in Eq. (\ref{eq:dictionary}), the dictionary $\mathcal{D}$ is simply constructed by applying geometric transformations $\gamma \in \mathcal{T}_d$ to a mother function $\phi$. We thus need to choose appropriately the mother function $\phi$ as well as the discretization for constructing the subset $\mathcal{T}_d$ of $\mathcal{T}$. In the light of the derived analytical results, ideally we would like to design a dictionary that satisfies the following constraints: 
\begin{itemize}
\item Images should have a good sparse approximation in the dictionary (assumption $(A_1)$ of the analysis).
\item The dictionary should be robustly linearly independent. (Theorem \ref{th:main_theorem}).
\item The transformation inconsistency parameter of the dictionary should not be too large (Theorem \ref{th:main_theorem}).
\end{itemize}

We propose to use an anisotropic Gaussian generating function as it has been shown to provide good approximation results in natural images \cite{test_article}. It is defined as follows:
\begin{align*}
\phi(x,y) = \frac{1}{\xi} \exp \left( - \left( \frac{x}{\nu} \right)^2 - y^2 \right),
\end{align*}
where $\nu > 1$ controls the anisotropy and the normalization factor $\xi$ is chosen to have $\| \phi \|_2 = 1$\footnote{Formally, the Gaussian mother function does not satisfy the one-to-one mapping assumption of $\gamma \mapsto U(\gamma) \phi$. We circumvent this by slightly modifying the definition of $\mathcal{T}_a^{p,q}$. We define the stabilizer of $\phi$ to be the set that keeps the mother function unchanged: $\mathcal{S}_\phi = \{ \gamma: U(\gamma) \phi = \phi \}$. Then, we define $\mathcal{T}_a^{p,q} = \{ \delta_i \circ \pi \circ (\gamma_j)^{-1} : 1 \leq i,j  \leq K, \pi \in \mathcal{S}_\phi \}$. For more details, refer to Appendix \ref{app:not_bijective}.}. The choice of $\nu \approx 1$ results in an isotropic mother function that causes the transformation inconsistency $\rho$ to be very large (see Example \ref{ex:example_rho1}). The transformation inconsistency is also large when the value of $\nu$ is chosen to be large  (see Example \ref{ex:example_rho2}). In our experiments, we have generally chosen an intermediate value $\nu = 4$ as a compromise between the two extreme values.

The dictionary $\mathcal{D}$ is built by transforming the generating function $\phi$ with all transformations in $\mathcal{T}_d$. 
In our experiments, we consider the following discretization: 

\begin{itemize}
\item The translation parameters can take any positive integer value smaller than the image dimension.
\item The rotation angles are uniformly discretized in $[0, \pi)$ with a step size of $\frac{\pi}{8}$. We have seen experimentally that this step size results in a good directional accuracy. A denser discretization comes at the expense of higher computational cost. 
\item The scaling parameters are sampled uniformly on a logarithmic scale with a step size of half an octave. This step size results in a compromise between the sparse approximation error and an oversampling of the scale space that might lead to wrong registration (and a too high computational complexity). We set the minimum scale to one, and the maximum scale is designed to have 99\% of the energy of a centred atom inside the image domain.
\end{itemize}


Fig. \ref{fig:parts_based_representation} illustrates several examples of parts-based representations obtained with NMP and a dictionary of Gaussian atoms, as described above. We observe that the part-based decomposition manages to approximate well the main geometric characteristics of the image. Furthermore, the same features are used in the different approximations, up to some geometrical transformation that corresponds to the relative transformation between the different versions of the original image. This is exactly the property that is at the core of our registration algorithm.

\begin{figure}[ht]
\centering
\subfigure[]{
\includegraphics[width=0.1\textwidth]{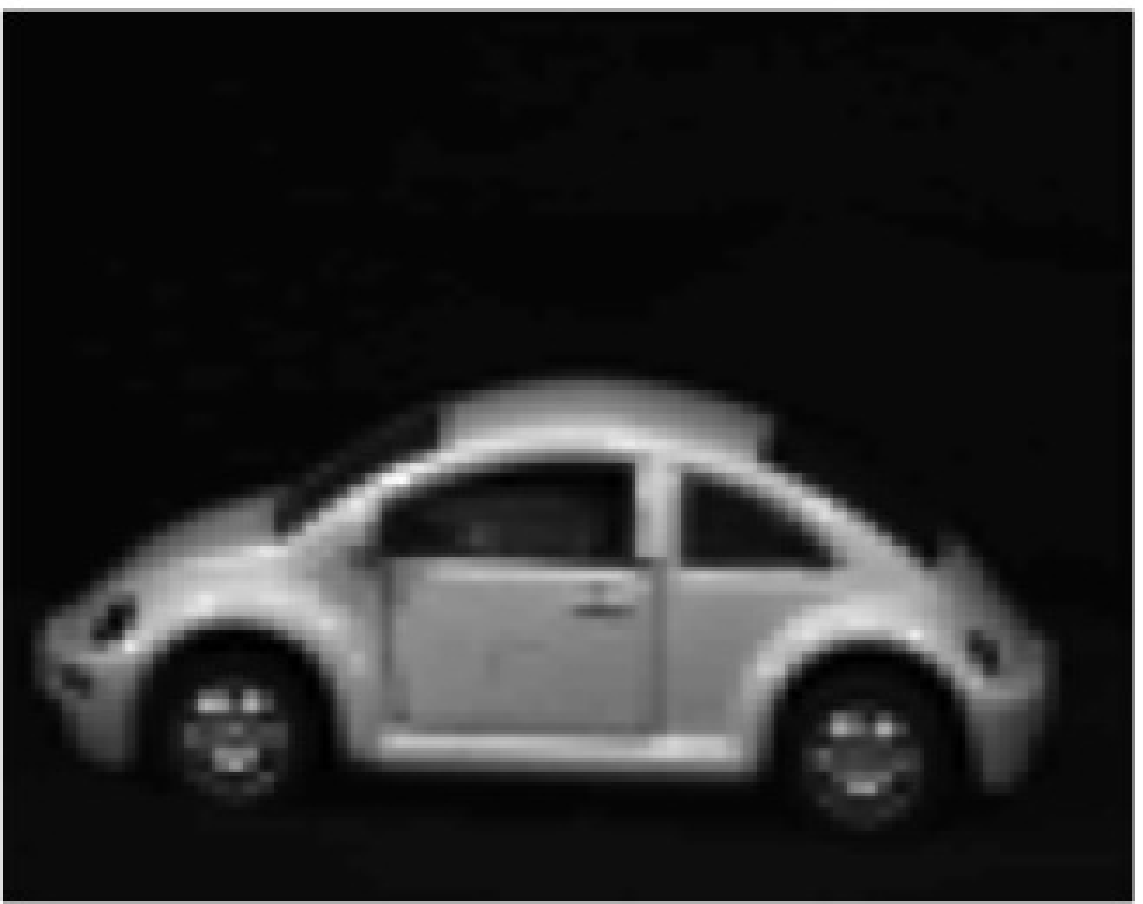}
}
\subfigure[]{
\includegraphics[width=0.1\textwidth]{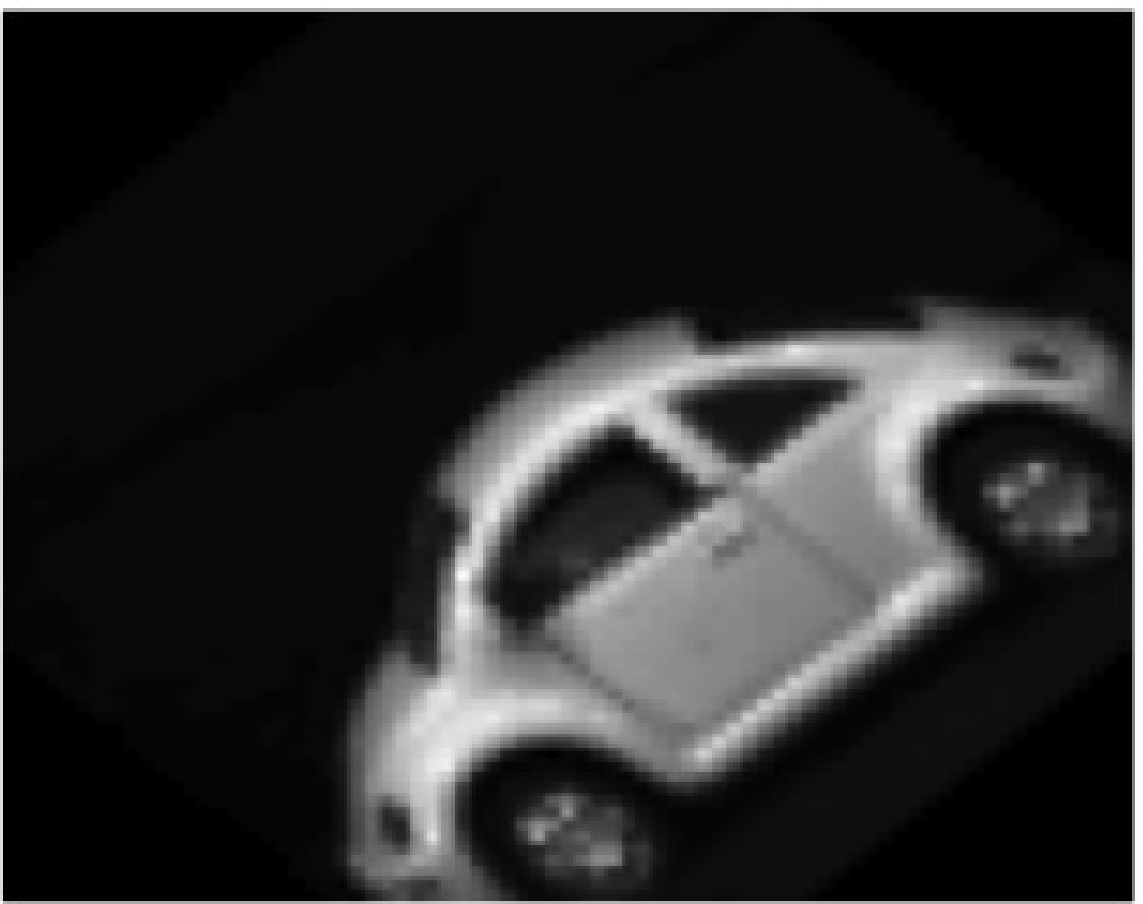}
}
\subfigure[]{
\includegraphics[width=0.1\textwidth]{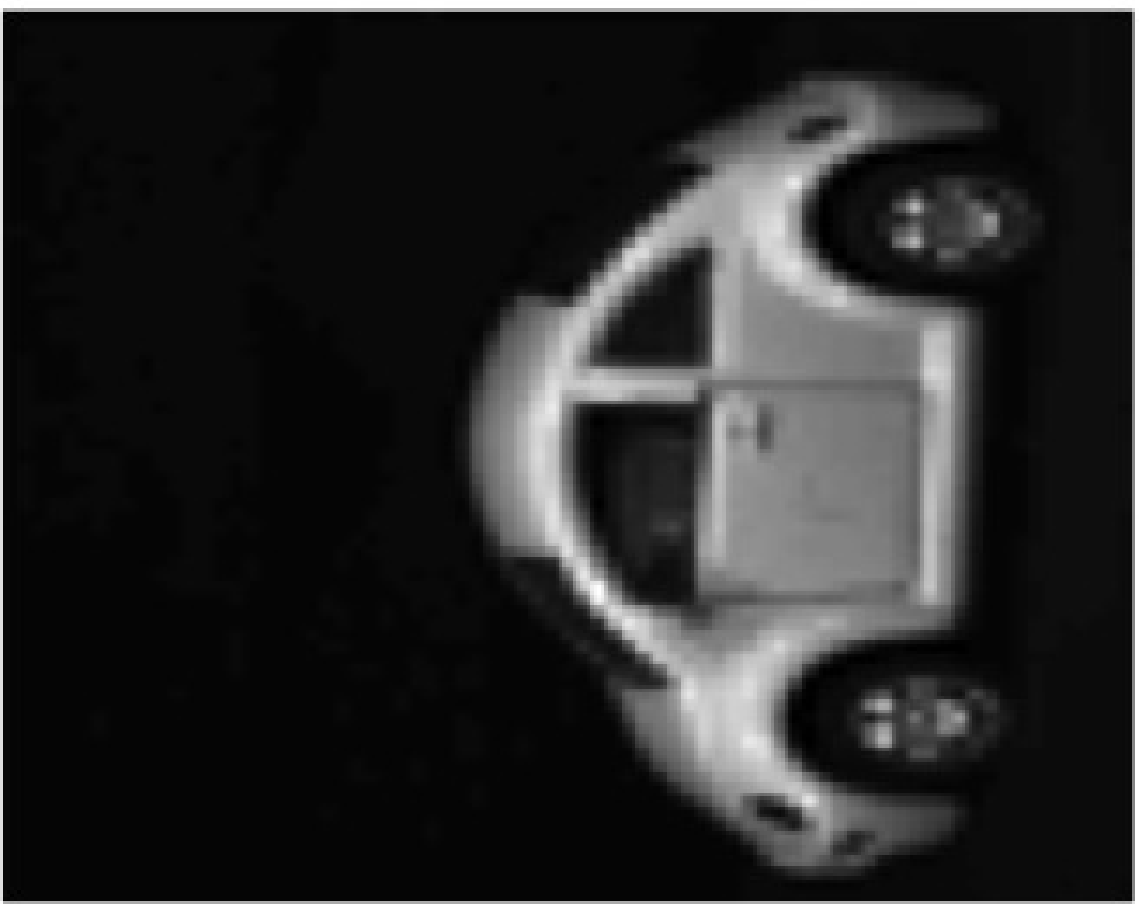}
}
\\
\subfigure[]{
\includegraphics[width=0.1\textwidth]{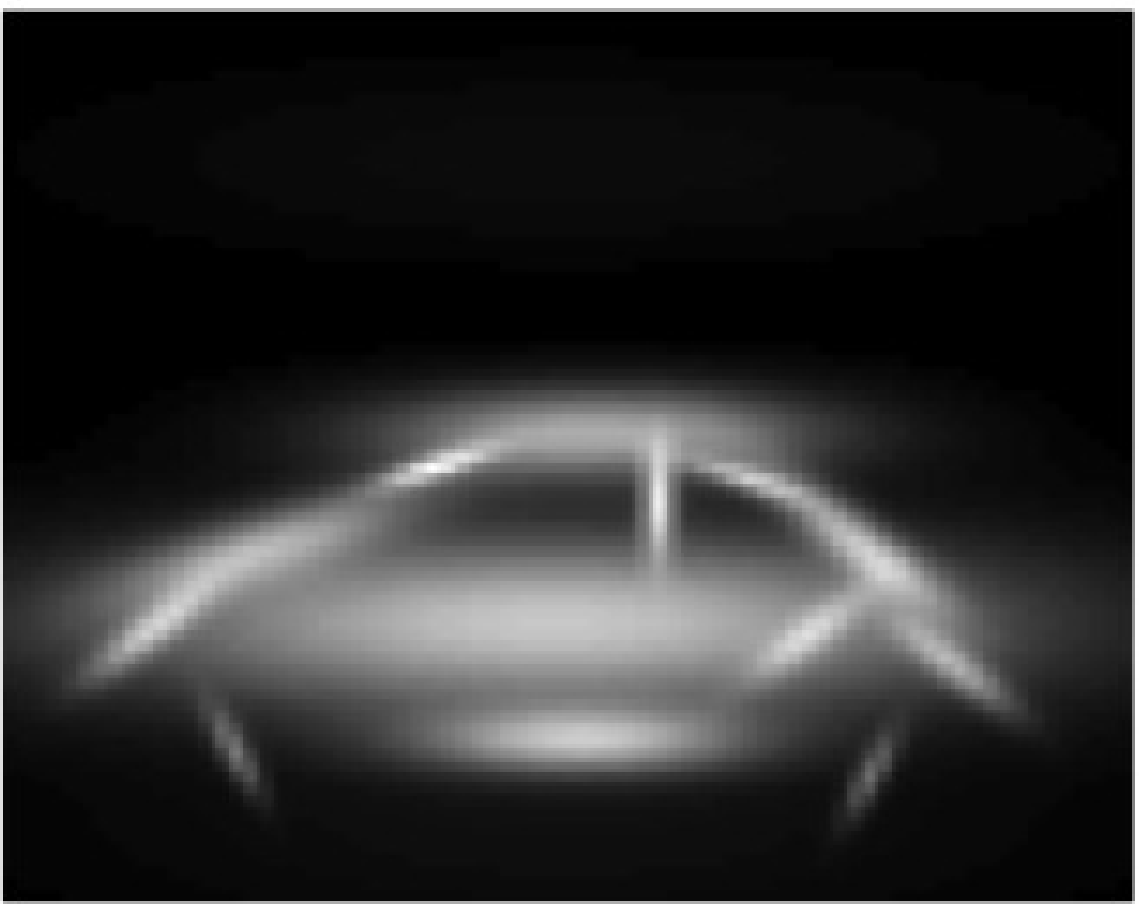}
}
\subfigure[]{
\includegraphics[width=0.1\textwidth]{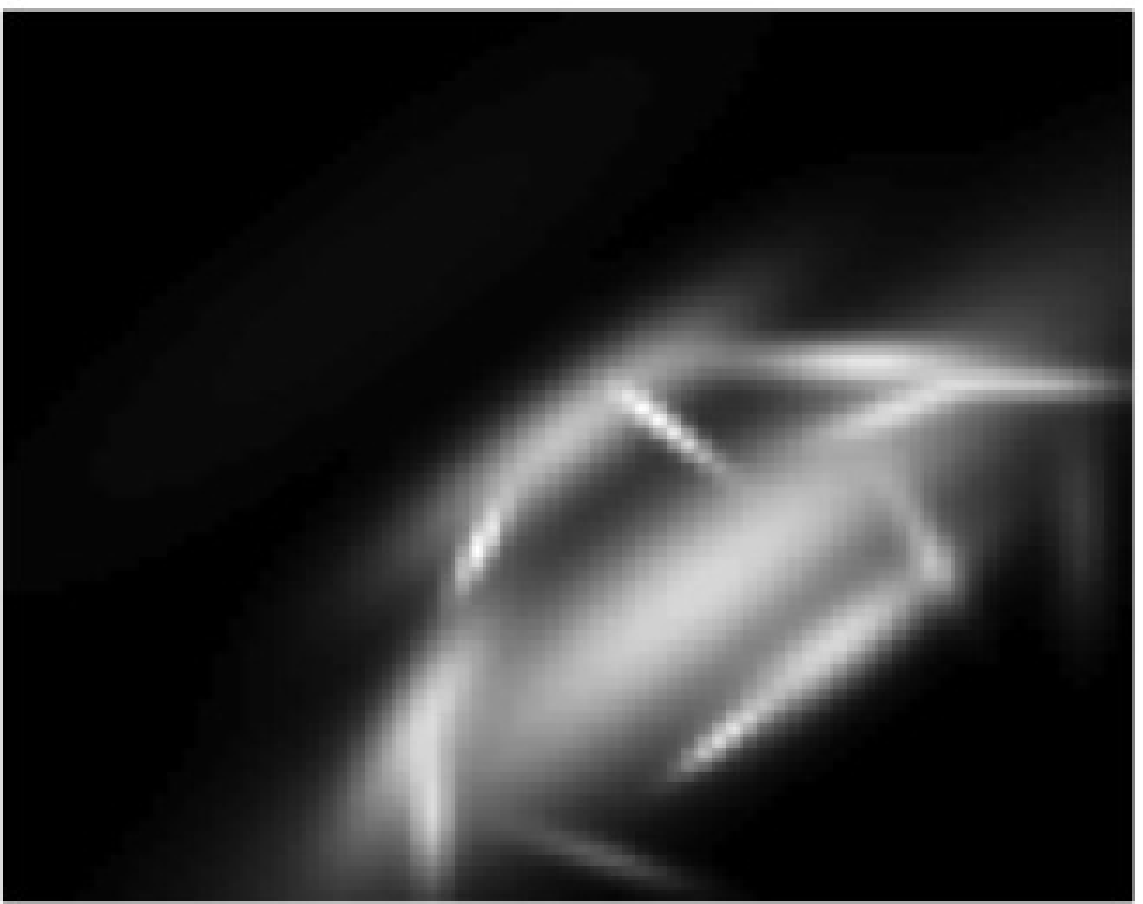}
}
\subfigure[]{
\includegraphics[width=0.1\textwidth]{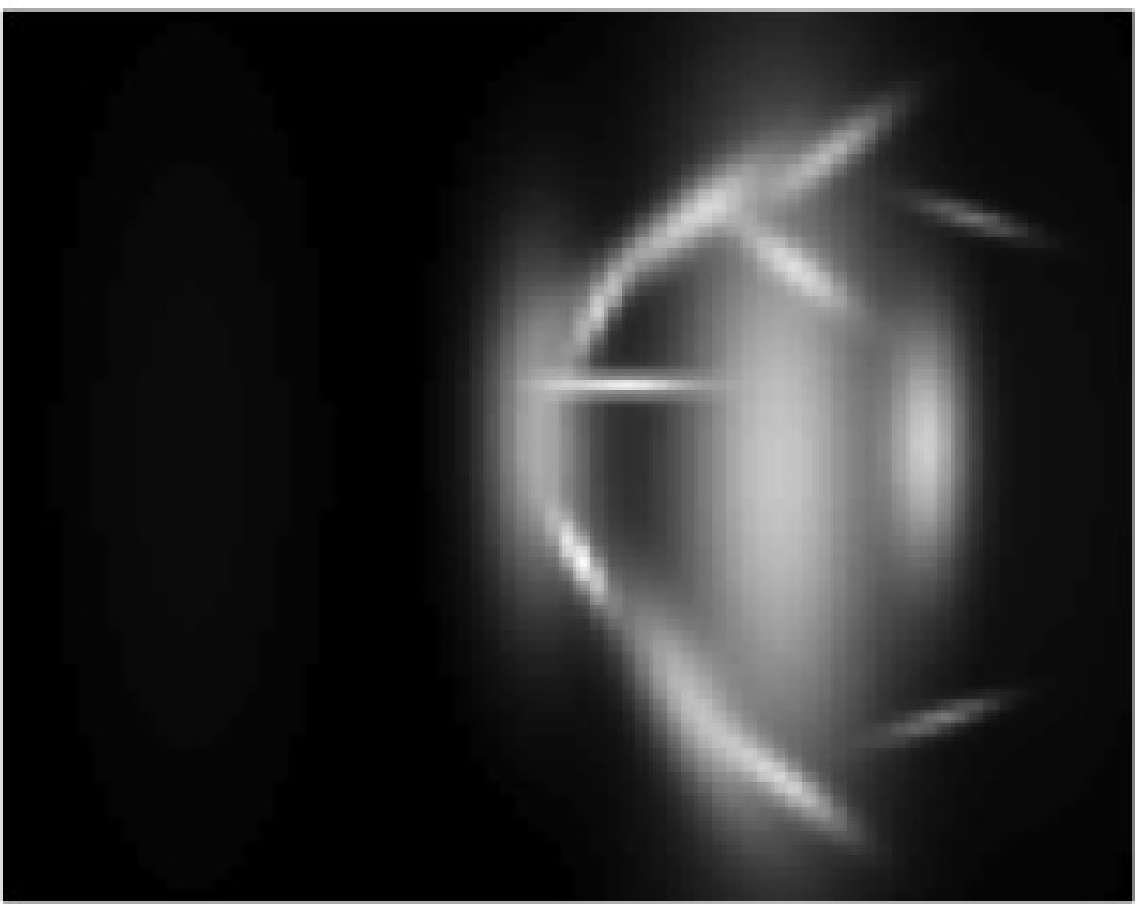}
}
\caption{\label{fig:parts_based_representation} Sparse approximations of transformed versions of the 'Car' image, computed with NMP and a dictionary constructed from a Gaussian generating function, with $\nu = 4$. The first row shows the original images, size $75 \times 75$ pixels. The second row shows the corresponding sparse approximations with a sparsity of $K = 15$ atoms.} 
\end{figure}

\subsubsection{Registration refinement}

Our registration algorithm estimates a transformation in the set $\mathcal{T}_a^{p,q} \subset \{ \gamma \circ \delta^{-1}: \gamma, \delta \in \mathcal{T}_d \}$, where $\mathcal{T}_d$ is the chosen discretization of the parameter space $\mathcal{T}$. In order to reduce the registration error that is due to the discretization of the dictionary, we have chosen in the experiments to extend our registration algorithm with a gradient descent technique that refines the estimated transformation. Hence, even if the optimal transformation $\eta_0$ is not located on the lattice formed by the discretization of the transformation parameter space, the additional local optimization step allows to converge to the optimal transformation if it lies close to the estimation computed by our registration algorithm. 

Specifically, the problem consists in minimizing the objective function $J(\eta) = \| U(\eta) p - q \|_2^2$, where the unknown transformation $\eta$ is constrained to be in $\mathcal{T}$. 
Following the same approach as the authors in \cite{jacques2008geometrical}, we consider the gradient descent induction given by:
\begin{align*}
\tau_{i+1} = \tau_{i} - w \nabla J(\tau_{i}) \text{ for } i \geq 0,
\end{align*}
where the gradient is defined by
\begin{align*}
\nabla J(\tau_{i}) = G_{\tau_i}^{-1} \begin{bmatrix} \partial_1 J(\tau_i) \\ \vdots \\ \partial_P J(\tau_i) \end{bmatrix},
\end{align*} 
with $G_{\gamma} = (\scalprod{\partial_i \phi_{\gamma}}{\partial_j \phi_{\gamma}})_{1 \leq i,j \leq P}$ (for any $\gamma \in \mathcal{T}$) and $w$ defines the step size. For more details on the derivation of this gradient descent scheme, we refer the reader to Appendix \ref{app:gradient_descent}.


\subsection{Influence of the dictionary on the registration performance}
\label{sec:influence_of_dictionary_reg_performance}


\begin{figure}[ht]
\centering
\includegraphics[width=0.7\textwidth]{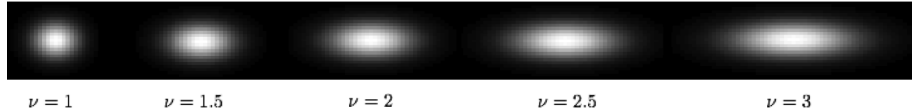}
\caption{\label{fig:gaussian_mother_functions} Gaussian mother functions with different values of the anisotropy $\nu$.}
\end{figure}

In a first set of experiments, we examine the influence of the dictionary choice on the registration performance\footnote{In this set of experiments, we apply our registration algorithm without the gradient descent refinement. We do so in order to focus exclusively on the performance of Algorithm \ref{alg:registration_algo} in terms of the considered dictionary.}. We fix here the transformation group $\mathcal{T}$ to be the special Euclidean group $SE(2)$ (containing translations and rotations). We consider that the dictionary mother function is a 2D anisotropic Gaussian function. We vary the anisotropy parameter $\nu$ of the mother function to generate a class of different dictionaries. Several generating functions obtained by varying the anisotropy parameter $\nu$ are illustrated in Fig. \ref{fig:gaussian_mother_functions}. Note that the discretization of the parameter space $\mathcal{T}_d$ is kept fixed for all dictionaries. 
We study now the registration performance of each of these dictionaries. 

\begin{figure}[ht]
\centering
\subfigure[$I_1$]{
\includegraphics[width=0.2\textwidth]{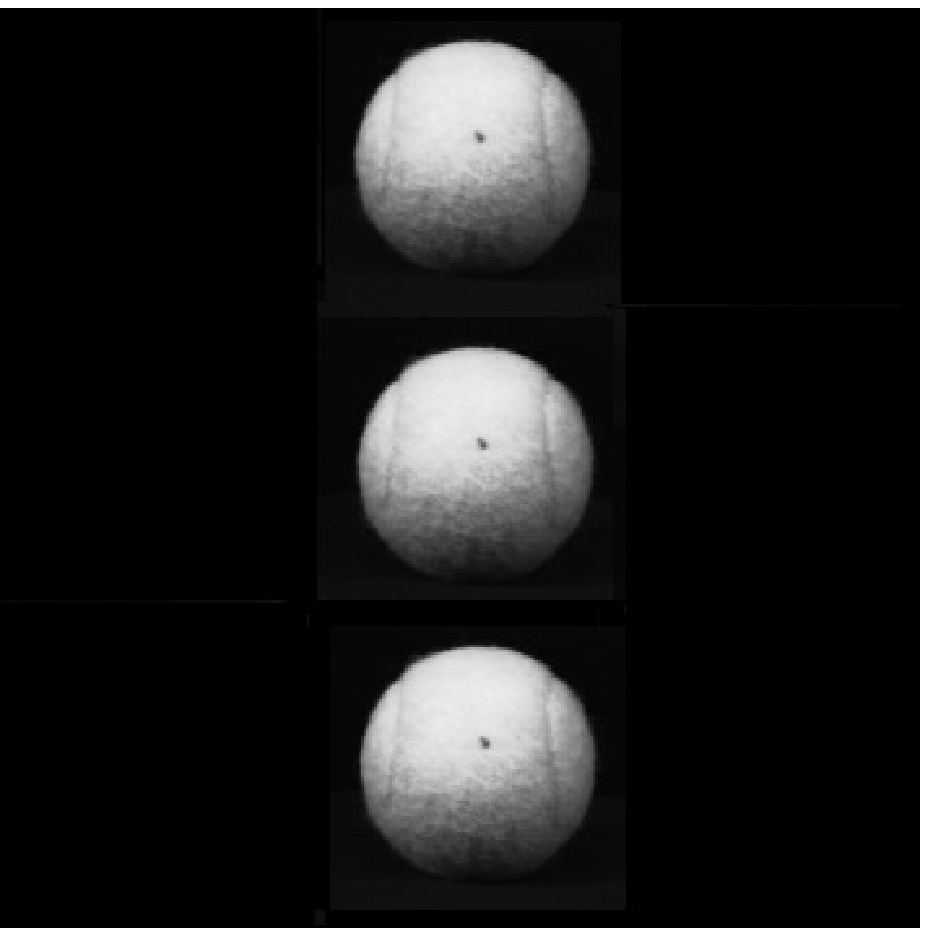}
}
\subfigure[$I_2$]{
\includegraphics[width=0.2\textwidth]{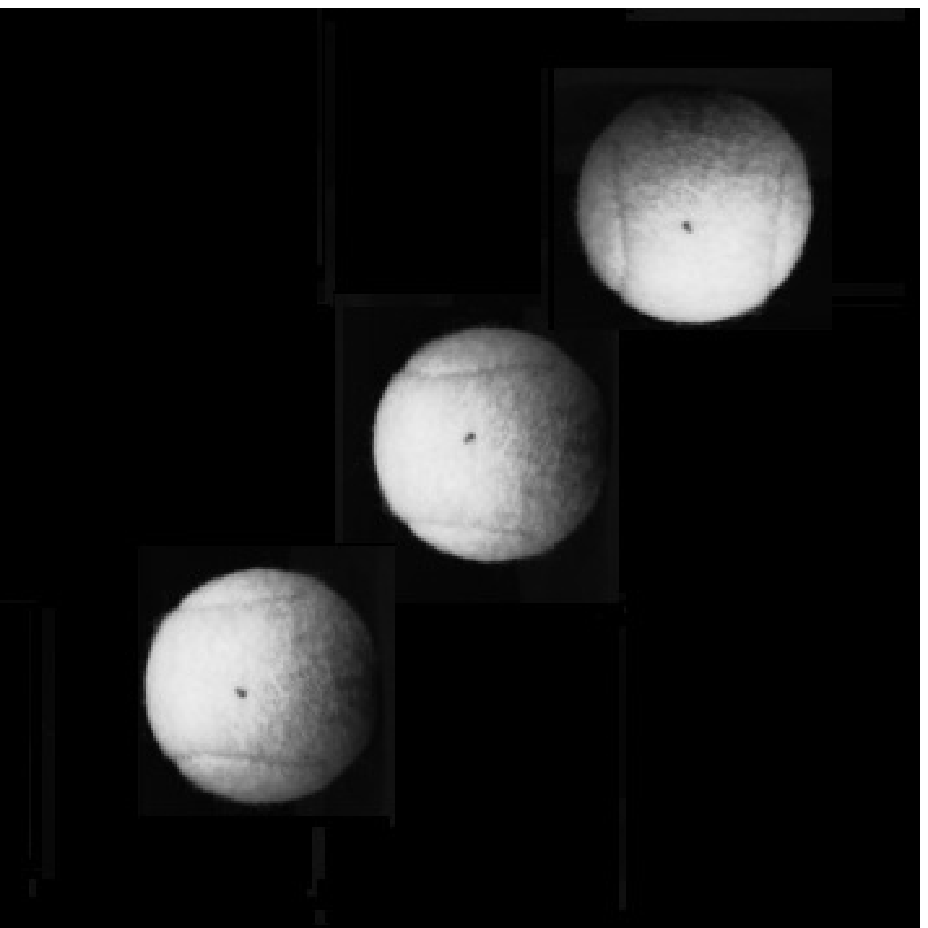}
}
\caption{\label{fig:tennis1}Original images used in the first experiment.}
\end{figure}
 
In a first experiment, we test our registration algorithm with the images $I_1$ and $I_2$ illustrated in Fig. \ref{fig:tennis1} for our class of dictionaries. We first represent in Fig. \ref{fig:tennis_1_results} the mean sparse approximation error $\frac{1}{2} \left( \| I_1 - p \|_2 + \| I_2 - q \|_2 \right)$ for decompositions with $K=3$ atoms, when the anisotropy parameter $\nu$ in the dictionary mother function varies. For the same class of dictionaries, we also measure the registration performance \mbox{$\left| \| U( \eta_0 ) I_1 - I_2 \|_2 - \| U( \hat{\eta} ) I_1 - I_2 \|_2 \right|$} where $\eta_0$ and $\hat{\eta}$ are respectively the optimal transformation (namely a rotation of $\pi/4$), and the estimated transformation. Note that we used this notion of error instead of $E'(p, q, I_1, I_2)$ in order to focus exclusively on the error due to a wrong estimate of the transformation. 
The registration performance is illustrated in Fig. \ref{fig:tennis_1_results}. One can see clearly that the sparse approximation error is increasing with the anisotropy of the mother function. Indeed, when the mother function approaches isotropy, the dictionary approximates well the tennis balls in images $I_1$ and $I_2$. The registration performance has however an opposite behavior: the error decreases with increasing values of the anisotropy. This suggests that the sparse approximation error is not the only quantity controlling the performance of the registration algorithm, as predicted by our theoretical performance analysis. Indeed, using the same arguments as in Example \ref{ex:example_rho1}, we know that the transformation inconsistency parameter goes to infinity when the mother function is isotropic: this explains the poor registration performance for generating functions that are close to isotropic. 



\begin{figure}[ht]
	\centering
		\includegraphics[width=0.4\textwidth]{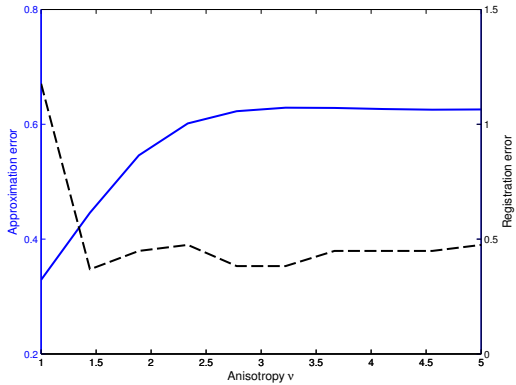}
	\caption{\label{fig:tennis_1_results}Approximation error (solid) and registration error (dashed) for images in Fig. \ref{fig:tennis1} as a function of the anisotropy of the dictionary generating function. The sparsity $K$ is fixed to $3$. The sparse approximation error is given by $\frac{1}{2} ( \| I_1 - p\|_2 + \| I_2 - q \| )$, and the registration error by  $\left| \| U( \eta_0 ) I_1 - I_2 \|_2 - \| U( \hat{\eta} ) I_1 - I_2 \|_2 \right|$.}
\end{figure}

As the transformation inconsistency parameter looks crucial in the registration performance, we estimate its value for the same class of dictionaries. This estimation is performed by applying the definition of the transformation inconsistency\footnote{Note that we applied the definition in Eq. (\ref{eq:inconsistency_symmetry}) since these atoms have a rotational symmetry of $\pi$.}, where the infinite set $\mathcal{T}$ is finely discretized. In a final step of the estimation, the transformation $\eta \in \mathcal{T}$ that maximizes the transformation inconsistency is refined with a local gradient descent search. Fig. \ref{fig:rho_experiment} shows the estimated value of transformation inconsistency parameter with respect to the anisotropy of the generating function. One can see that the evolution of the transformation inconsistency parameter is consistent with the theoretical analysis in Section \ref{sec:theoretical_analysis}. For near-isotropic atoms, the parameter $\rho$ is large (Example \ref{ex:example_rho1}). Similarly, when $\nu$ is large, the transformation inconsistency increases 
as shown in Example \ref{ex:example_rho2}. 
Even though our estimation of the transformation inconsistency may not be perfectly accurate (due to the discretization of $\mathcal{T}$), it confirms the tendencies described earlier in the theoretical analysis. It further contributes to explaining the trade-off between approximation and registration error that has been illustrated in Fig. \ref{fig:tennis_1_results}.

\begin{figure}[ht]
\centering
\includegraphics[width=0.4\textwidth]{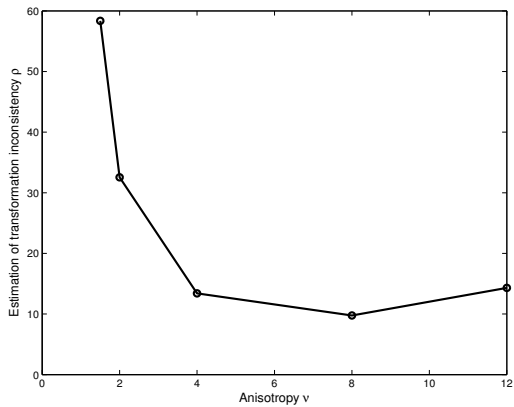}
\caption{\label{fig:rho_experiment} Estimation of the transformation inconsistency parameter for dictionaries built on Gaussian mother functions with different anisotropy $\nu$.}
\end{figure}

We study now a second experiment where we consider that the transformation group is $\mathcal{T} = \mathbb{R}^2 \times \mathbb{R}^+_{*}$. That is, $\mathcal{T} $ contains transformations that can be written as combinations of translation and isotropic dilation. We construct another class of dictionaries by fixing the generating function to be an isotropic Gaussian (as shown in Fig.  \ref{fig:gaussian_mother_functions}, $\nu = 1$) but we vary the step size that is used for the discretization of the dilation parameter. More precisely, the set of transformations $\mathcal{T}_d$ that is used to build the dictionary, is constructed from $\mathcal{T}$ by imposing a fixed uniform discretization of the translation parameter and a uniform discretization of the dilation parameter whose step size $\Delta_s$ can take different values. Note that the minimum and maximum scales are kept fixed in all dictionaries and only the space $\Delta_s$ between two consecutive scale parameters is varied. We finally measure the sparse approximation performance with $K = 3$, as well as the registration accuracy that can be obtained with this second class of dictionaries for the images $I_1$ and $I_2$ shown in Fig. \ref{fig:tennis2}. Both sparse approximation and registration errors are computed similarly to the previous experiment. They are illustrated in Fig. \ref{fig:tennis2_results} as a function of the different values of the scale step size $\Delta_s$.

\begin{figure}[ht]
\centering
\subfigure[Original image $I_1$]{
\includegraphics[width=0.4\textwidth]{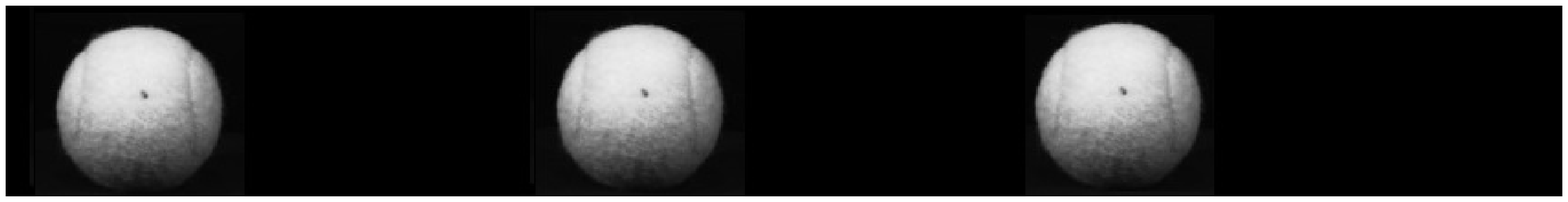}
}
\subfigure[Transformed image $I_2$]{
\includegraphics[width=0.4\textwidth]{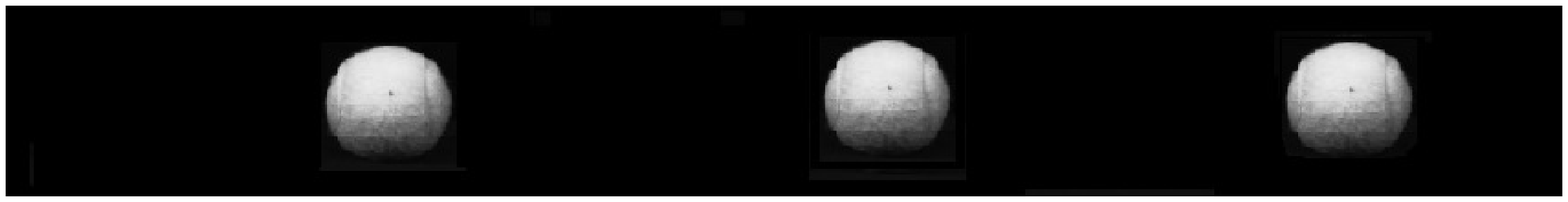}
}
\caption{\label{fig:tennis2} Original images used in the second experiment.}
\end{figure}

\begin{figure}[H]
	\centering
		\includegraphics[width=0.4\textwidth]{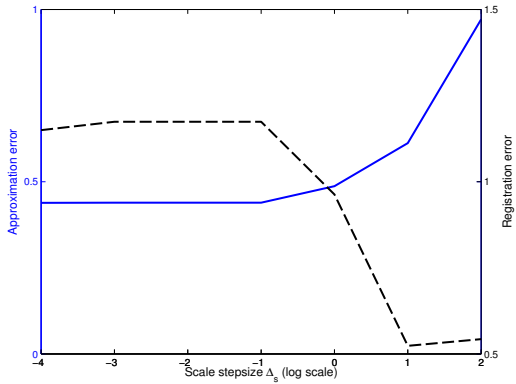}
	\caption{\label{fig:tennis2_results}Approximation error (solid) and registration error (dashed) for images in Fig \ref{fig:tennis2} as a function of the scale stepsize used for constructing the dictionary. The sparsity $K$ is fixed to $3$.}
\end{figure}

We observe in Fig. \ref{fig:tennis2_results} that the sparse approximation and registration errors have opposite behaviors with respect to the scale space discretization. This is in-line with our observations on the first experiment above. Indeed, a fine discretization leads to a small approximation error. At the same time, the registration is less accurate when the discretization is fine. Conversely, coarser discretization of the scale parameter results in less compact dictionary, hence in larger approximation errors, but better registration performance. These tendencies can be explained using the arguments developed in Example \ref{ex:example_rho3}.

In summary, these two experiments show that constructing a dictionary that guarantees a small approximation error of the images is not enough to have a low registration error. As we have seen earlier in Section \ref{sec:theoretical_analysis}, crucial parameters such as robust linear independence and transformation inconsistency have to be taken into account in the design of the dictionary in order to reach good registration performance. 

\subsection{Illustrative examples}
\label{sec:illustrative_examples}
We propose in this section some illustrative experiments that study the performance of our registration algorithm for determining the transformation between pairs of images, or for image classification. We further compare the properties of our registration algorithm to other baseline solutions for computing transformation invariant distances. 

\begin{figure}[ht]
\centering
\subfigure[Duck]{
\includegraphics[width=0.1\textwidth]{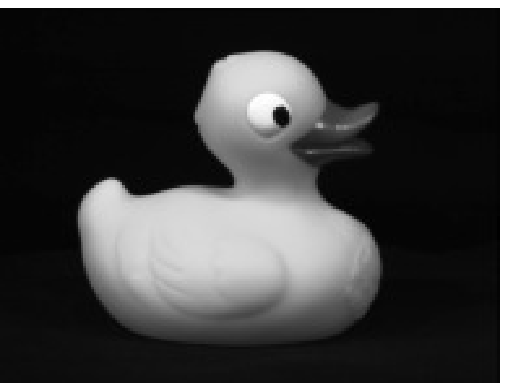}
}
\subfigure[Car]{
\includegraphics[width=0.1\textwidth]{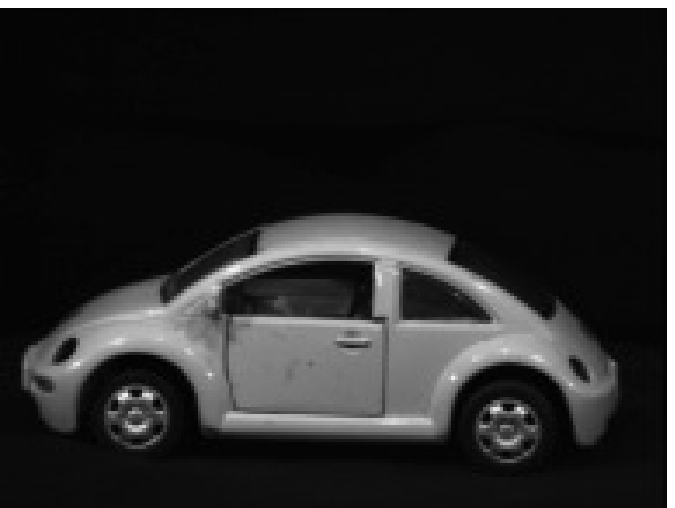}
}
\subfigure[Bear]{
\includegraphics[width=0.1\textwidth]{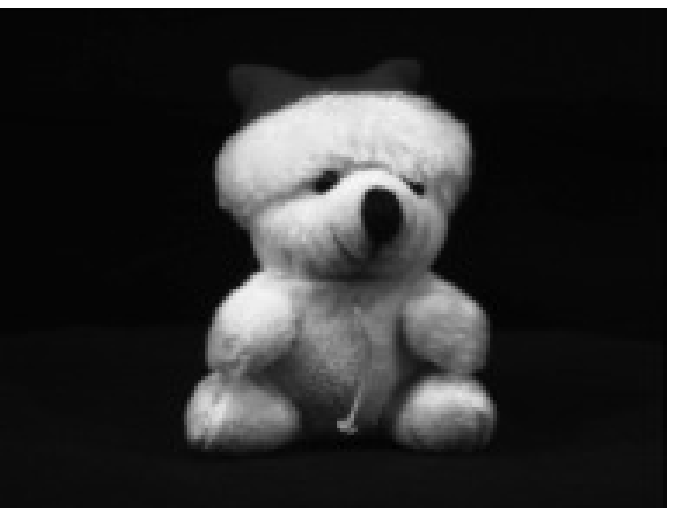}
}
\caption{\label{fig:test_images}Test images \cite{geusebroek2005amsterdam}. All images are resized to be of dimension $75 \times 75$ pixels.}
\end{figure}

In our first experiments, we consider the test images shown in Fig.  \ref{fig:test_images}, which have been collected from the ALOI dataset \cite{geusebroek2005amsterdam}. We generate $100$ random transformations and apply them to the test images. Each of the transformation belongs to $\mathcal{T}$ and consists in a combination of translation, rotation and isotropic scaling. Both components of the translation vector are smaller than half the image size and the isotropic scaling parameter is constrained to be in $[0.5, 1.5]$. These restrictions guarantee that most of the image energy lies in the image space, possibly with some occlusions. We put no specific restrictions on the rotation angle. Fig. \ref{fig:duck_transformations} illustrates some examples of transformed images.  

\begin{figure}[H]
	\centering
		\includegraphics[width=0.25\textwidth]{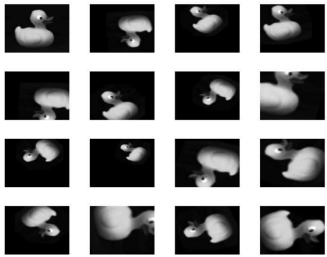}
	\caption{Sample set of test images built by applying random geometric transformations to the Duck image.}
	\label{fig:duck_transformations}
\end{figure}

We first examine the accuracy of our algorithm in estimating the correct global transformation between pairs of images. We register each of the transformed test images with the original image and compute the average registration accuracy over $100$ such operations. Fig. \ref{fig:errors_transformation} shows the average error in the translation, scaling and rotation parameters when registering pairs of 'Duck' images for different number of features $K$ in the sparse image approximations. We see that for $K \geq 10$ our algorithm determines a very good approximation $\hat{\eta} = (\hat{b}, \hat{a}, \hat{\theta})$ of the optimal transformation $\eta_0' = (b_0', a_0', \theta_0')$. That is, we have in average a translation error of approximately 1 pixel, a scaling error of $0.02$ and an angle error of $10$ degrees. 

\begin{figure}[ht]
\centering
\subfigure[Translation error: $\| \hat{b} - b_0' \|_2$]{
\includegraphics[width=0.25\textwidth]{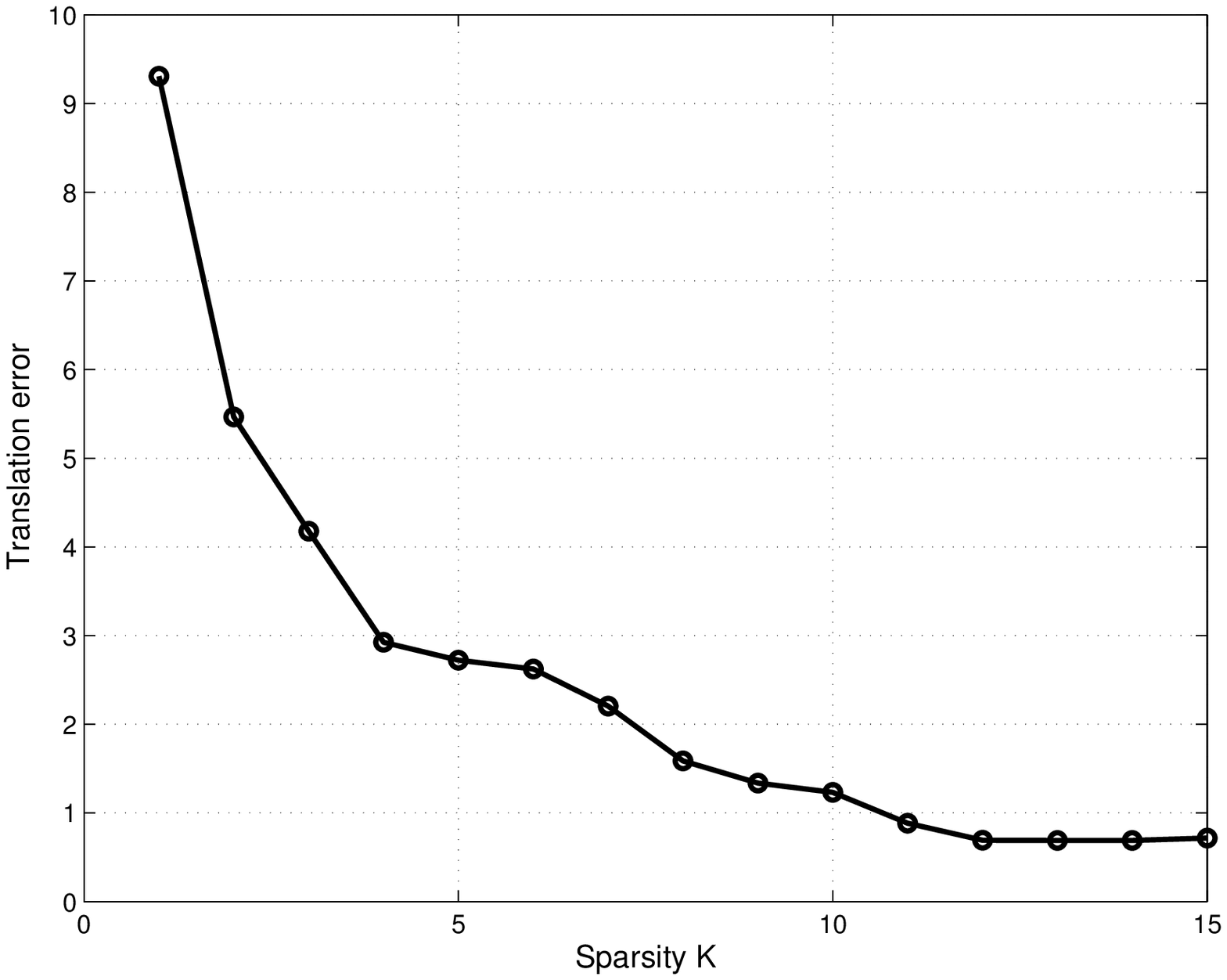}
}
\subfigure[Scale error: $|\hat{a} - a_0'|$ ]{
\includegraphics[width=0.25\textwidth]{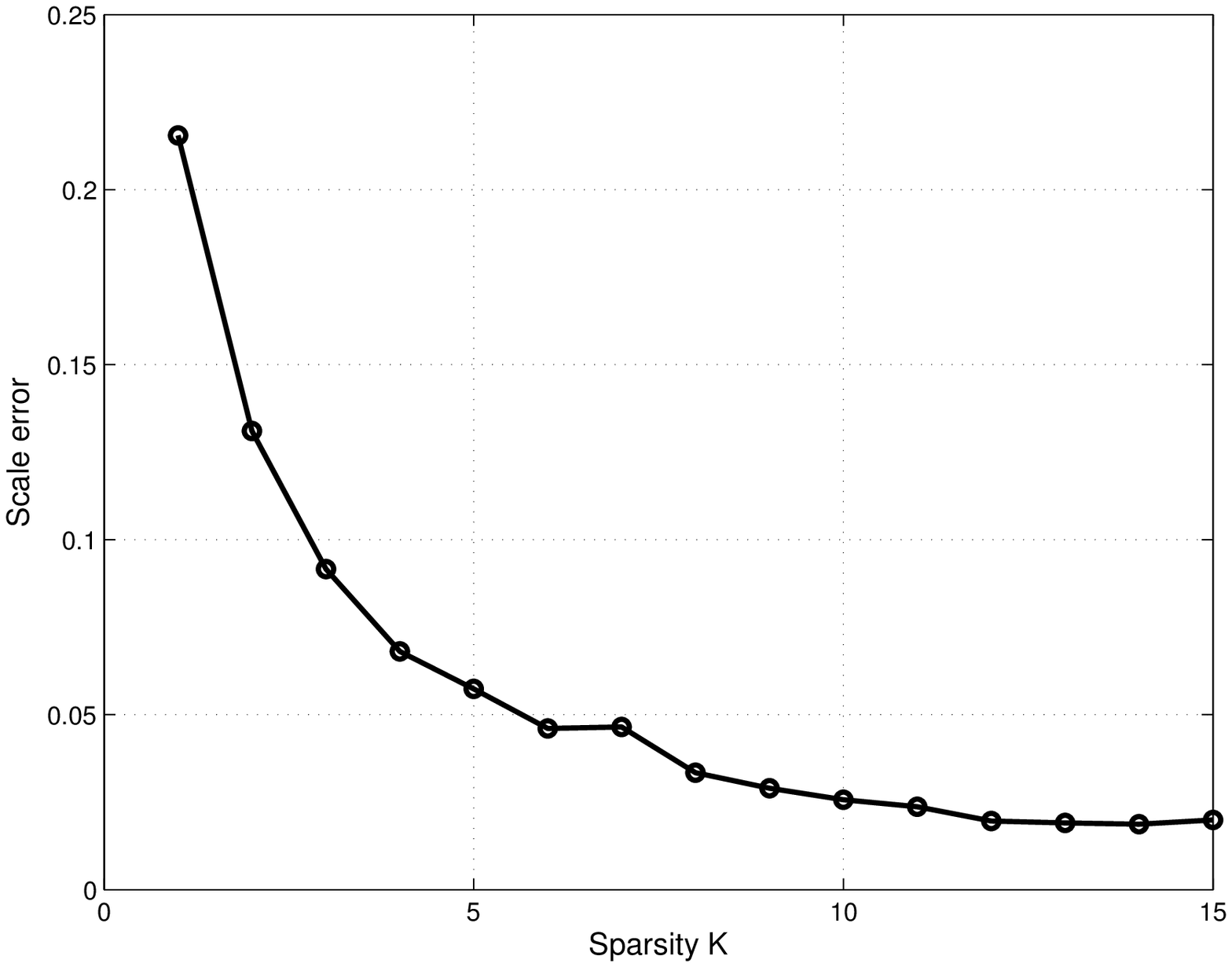}
}
\subfigure[Rotation error: $\min(|\hat{\theta} - \theta_0'|, 180 - |\hat{\theta} - \theta_0'|)$]{
\includegraphics[width=0.25\textwidth]{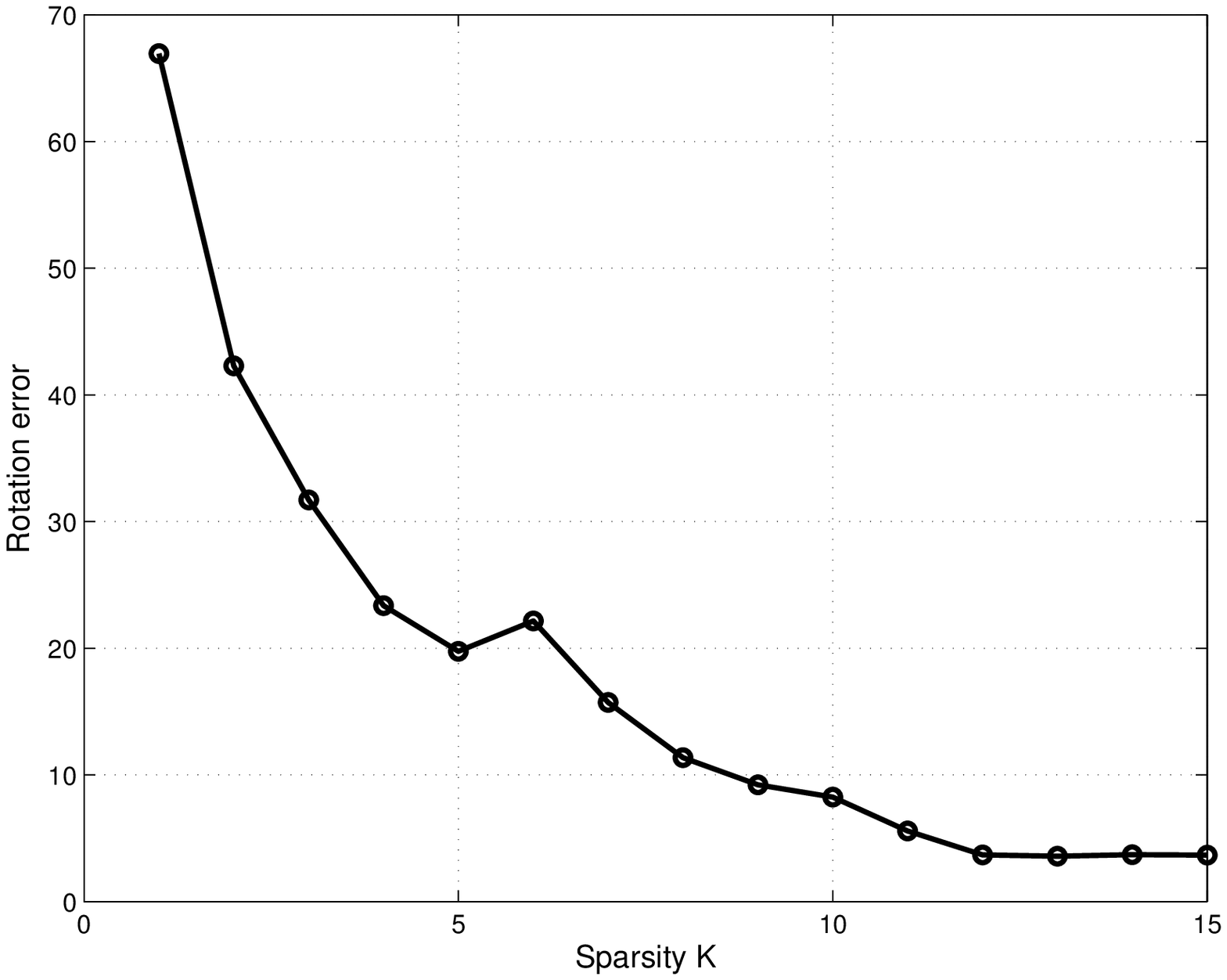}
}
\caption{\label{fig:errors_transformation}Errors in translation, scaling and rotation (in degrees) versus the sparsity $K$ in the approximation of 'Duck' images. The parameter of the optimal transformation obtained by solving $(P')$ is denoted with $\eta_0' = (b_0', a_0', \theta_0')$ and the estimated transformation with $\hat{\eta} =  (\hat{b}, \hat{a}, \hat{\theta})$. The results are averaged over 100 tests.}
\end{figure}


We compare now our method with several baseline algorithms for computing distances that are invariant to transformations. The first of these methods is based on the tangent distance \cite{simard1998transformation} that approximates the transformation invariant distance between the images with the distance between two linear subspaces that can be easily computed. Specifically, the authors in \cite{simard1998transformation} approximate the distance $d(I_1, I_2)$ with:
\[
d_{TD} (I_1, I_2) = \min_{I'_1 \in T(I_1), I'_2 \in T(I_2)} \| I'_1 - I'_2 \|_2,
\]
where $T(I_1)$ and $T(I_2)$ are the tangent planes to the manifold of transformed images of $I_1$ and $I_2$ respectively, evaluated at $I_1$ and $I_2$. The equations of $T(I_1)$ and $T(I_2)$ can be explicitly computed and the original problem of computing the transformation invariant distance reduces to solving a least squares problem \cite{simard1998transformation}.  We also compare our method with an approach that solves the original problem $(P')$ using a simple gradient descent technique starting from the identity transformation. Finally, the last comparative scheme is simply based on the computation of the regular Euclidean distance between the images $I_1$ and $I_2$. Note that in all three competitor solutions, the distances are computed directly on the original images, whereas, in our approach we use only \textit{the sparse image approximations} to compute the distance. We choose to do so since our aim here is to show that our method can be used without explicitly using the complete images in the transformation estimation: a good sparse approximation is indeed sufficient to obtain accurate registration results.

\begin{figure}[ht]
\centering
\subfigure[Euclidean distance]{
\includegraphics[width=0.35\textwidth]{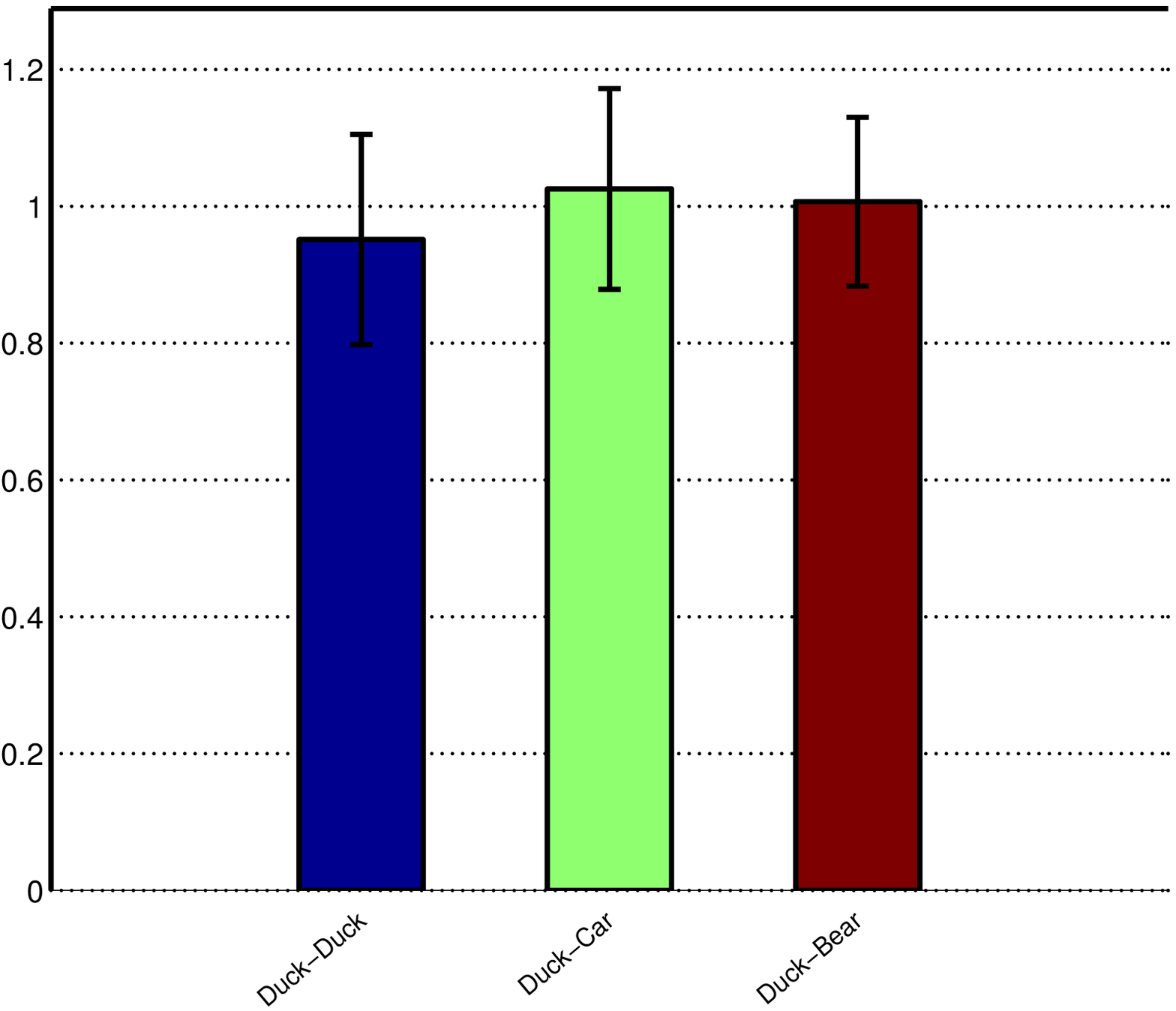}
}
\subfigure[Tangent distance]{
\includegraphics[width=0.35\textwidth]{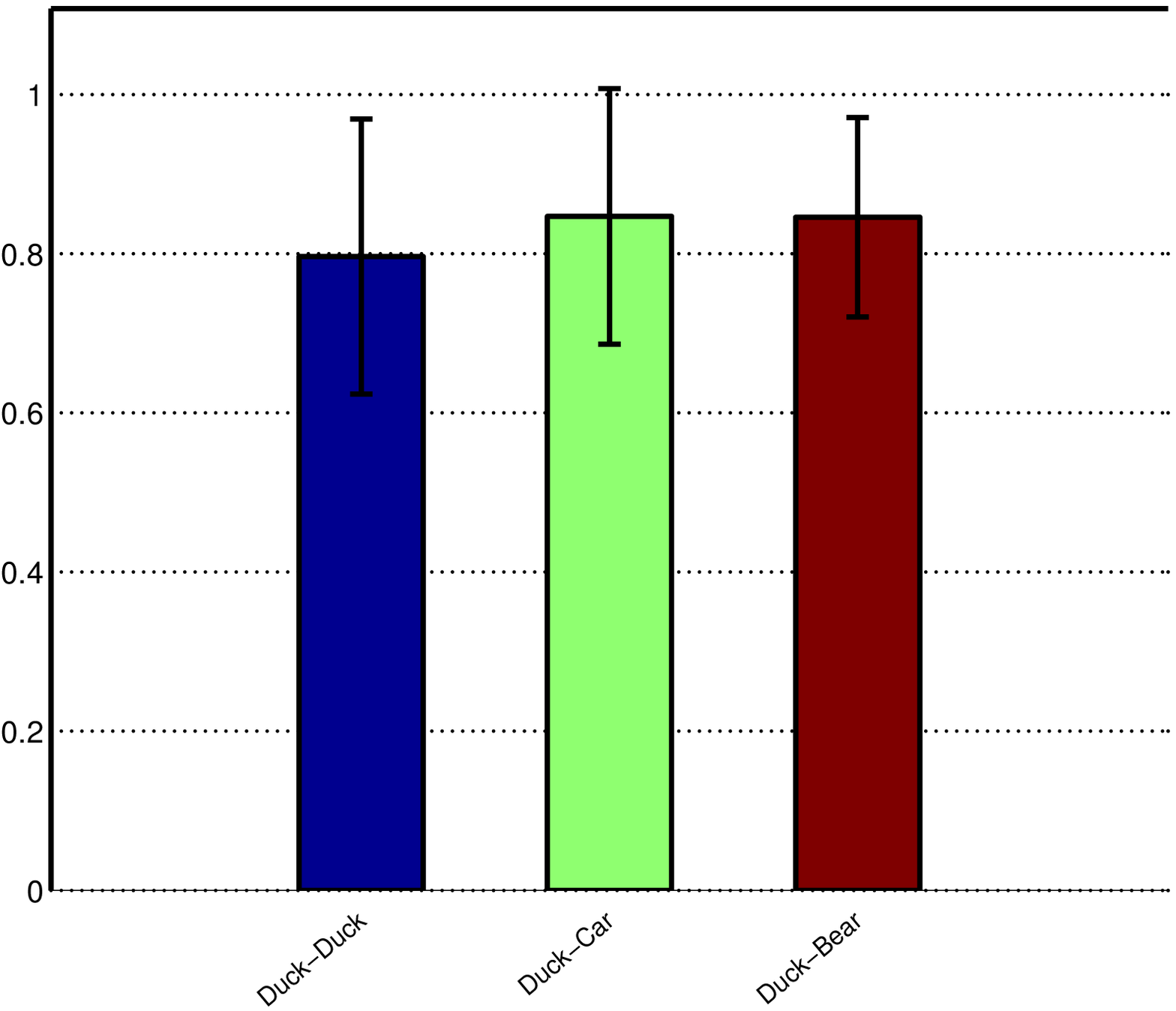}
}
\subfigure[Gradient descent]{
\includegraphics[width=0.35\textwidth]{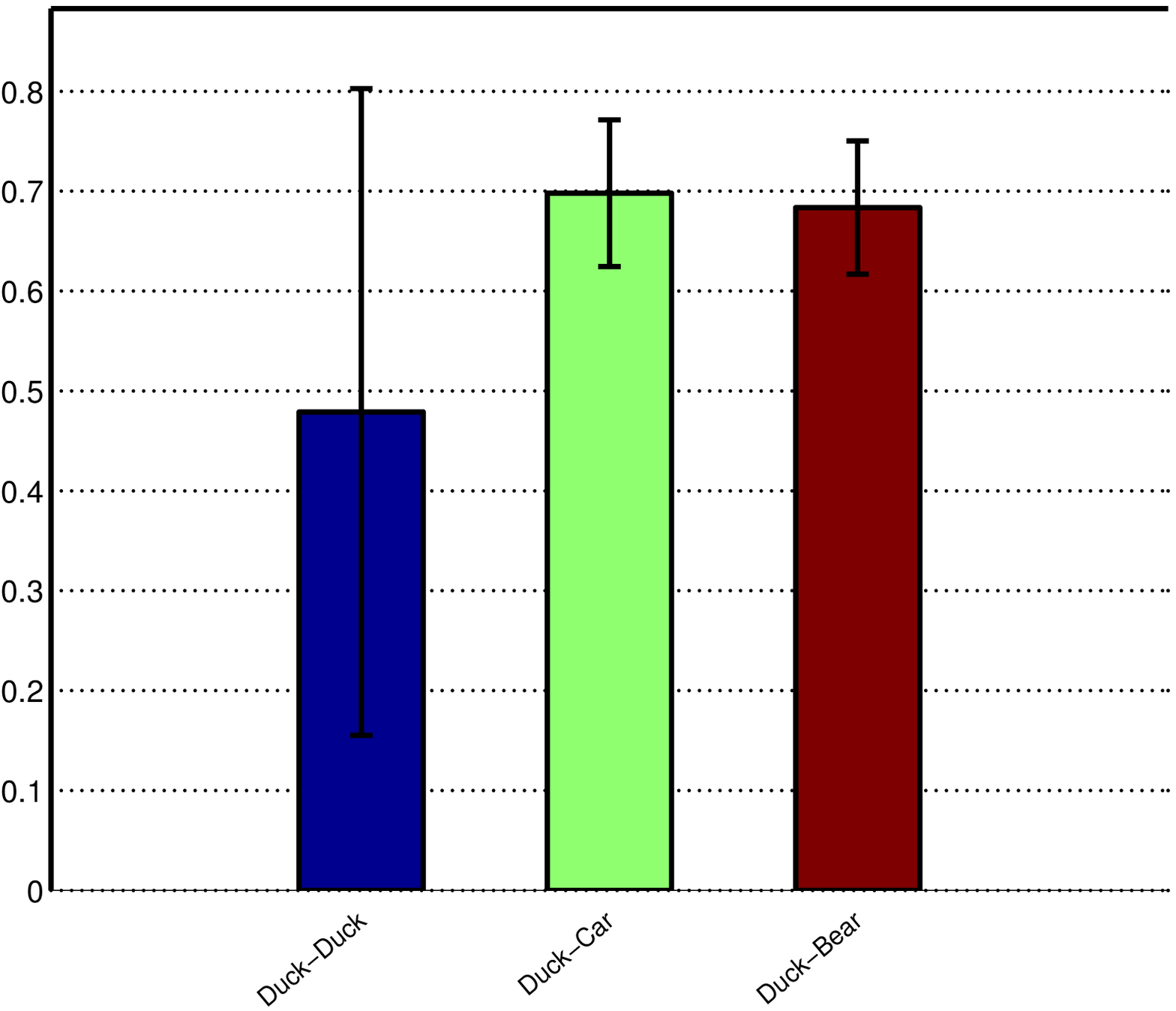}
}
\subfigure[Our method ($K=10$)]{
\includegraphics[width=0.35\textwidth]{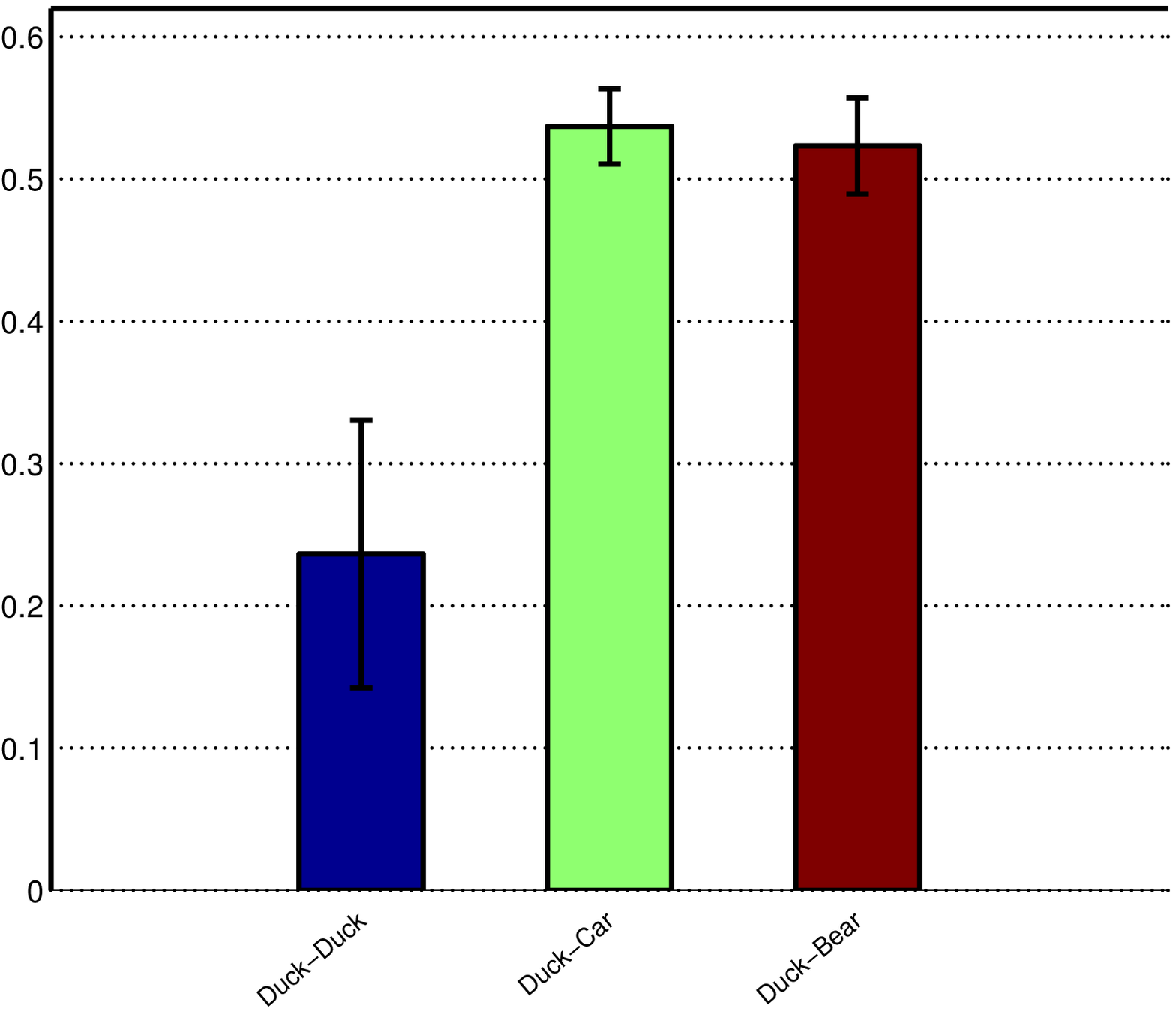}
}
\caption{\label{fig:algo_results1}Average and standard deviation of intra- and inter-class distances for different methods. The blue color denotes intra-class distance while the green and red colors refer to the distance between Duck-Car and Duck-Bear images respectively. The distance has been computed between one reference image ('Duck') and 100 randomly generated transformed images in each class. The intra-class distance should be ideally at zero. For our approach, the transformation invariant distance $d_a(p,q)$ is computed based on the sparse approximations, while the original images are used in the other methods.}
\end{figure}

We extend the previous experiments towards classification of images. In particular, we compare the transformation-invariant distance for images of the same class to the same distance computed between images of different classes. Ideally, the first one (the intra-class distance) should be smaller than the latter one (the inter-class distance) in order to obtain good classification performance. We start with a simple scenario where the reference image is chosen to be the 'Duck' image in Fig. \ref{fig:test_images}. We then compute the transformation invariant distance between the reference image and the transformed versions of images in the same class ('Duck' ), and in the other classes  ('Car' and 'Bear'). Fig.  \ref{fig:algo_results1} shows the average of the transformation invariant distances computed with the different methods. One can see that the euclidean distance between images of the same class is not significantly different from the distance between images of different classes. The tangent distance does not improve the performance since this method provides only local invariance to transformations. Similarly, the gradient descent approach converges to the correct transformation only when it is close enough to the initial transformation. As this happens rarely, this approach does not provide results that are significantly different for intra- and inter-class comparisons. In our method however, one can see that the intra-class distance is significantly smaller than the inter-class distance. Fig. \ref{fig:classVSsparsity} further shows the evolution of the transformation-invariant distance with respect to the sparsity of the images. We see that the intra-class distance is always smaller than any of the inter-class distances in our algorithm, even for very small values of the sparsity $K$. This provides a confirmation that salient geometric features in sparse images are crucial for proper registration. Hence, without having a very accurate sparse representation of the patterns, our registration algorithm succeeds in having an approximation of the distances that allows at least to classify the simple patterns under test. Note that this observation does not contradict the worst case theoretical analysis in which we assume that the sparse approximation error is small. We observe in practice that, even when this assumption does not hold, one can still obtain a good registration accuracy that is sufficient for the classification of simple signals. Finally, we note that the results are essentially the same if we repeat the same experiments with a different reference image in our dataset. Overall, our illustrative experiments so far show that, with a coarse approximation of the original images in the dictionary, our approach succeeds in obtaining an accurate estimation of the transformation, and the computation of the distances show that the intra- and inter-class images are well distinguished. This is an interesting property towards the development of registration algorithms in applications where access to the original (high quality) images is not possible.

\begin{figure}[ht]
	\centering
		\includegraphics[width=0.4\textwidth]{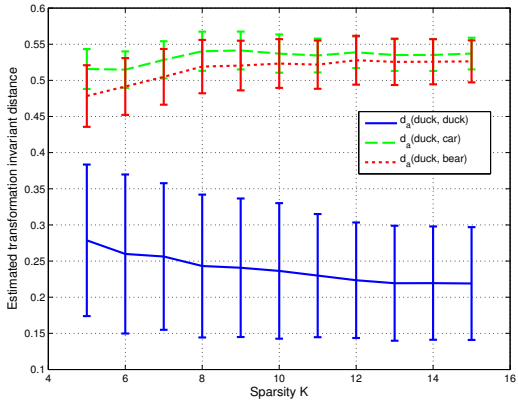}
	\caption{Evolution of the intra-class and inter-class transformation invariant distance $d_a(p,q)$ in the proposed algorithm as a function of the number of features in the sparse images.}
	\label{fig:classVSsparsity}
\end{figure}

\label{sec:handwritten_digits}
We extend the simple classification experiments proposed above and study now the performance of our registration method in a more challenging task of transformation-invariant handwritten digit classification. We use the digits '0' to '5' from the standard MNIST database of handwritten digits \cite{mnist_digits}. We construct the training data by randomly choosing 100 images for each digit, which results in 600 training images. The test data is constructed similarly: 100 images are taken in each class in order to generate 600 test images. Note that the test data does not contain any of the training images. Finally, we apply to each test image a random transformation built on translation, rotation and isotropic scaling. Our classifier then works as follows: each test data is assigned the label of the digit in the training set that best aligns with it, or equivalently that minimizes the transformation-invariant distance to the test image. 
In other words, the label of a test image is chosen to be the label of its nearest neighbour in the training dataset, up to a geometrical transformation. We compare the classification results when the transformation-invariant distance is computed with the different methods proposed above. Moreover, for completeness, we also compare our method to an approach that first extracts MSER regions \cite{matas2002robust}, followed by a similarity normalization \cite{muse2005theory} that transforms the image into a common system of coordinates. The transformation invariant distance between two digits is then defined as the distance between the normalized images. 

 The classification results are shown in Table \ref{tab:results_classification}. 

\begin{table}[H]
\footnotesize
\centering
\begin{tabular}{|c|c|c|}
  \hline
   & Classification accuracy  \\
  \hline 
  Euclidean distance & 14 \% \\
  \hline
  Tangent distance & 33 \% \\
  \hline
  Gradient descent & 62 \% \\
  \hline
  MSER + similarity normalization & 75 \% \\
  \hline
  Proposed registration algorithm ($K=10$) & 86 \% \\
  \hline
\end{tabular}
\caption{Handwritten digits classification accuracy for different approaches in computing transformation-invariant distances.}
\label{tab:results_classification}
\end{table}

One can see that using the Euclidean distance on the transformed test images results in a very poor classifier, whose performance is actually close to the one of a random classifier. Using the tangent distance results in some improvement, but it is still far away from the desired performance. This is due to the fact that the tangent distance is appropriate only for local transformations, while the transformations that we consider are generally of large magnitude. Similarly, the gradient descent approach does not perform well, since it is only guaranteed to reach a local minima. The MSER-based approach outperforms these local methods and achieves a classification performance of $75 \%$. Using our registration method however, we achieve a relatively high classification rate, which is by far the best performance among the compared methods. It is worth noting that the performance of our algorithm ($86 \%$ of classification accuracy) is only slightly worse than the performance of a Euclidean nearest neighbour classifier with aligned images (i.e., no transformations are applied on the test data), which reaches a classification accuracy of $94 \%$. The latter classifier provides an upper-bound on the performance we could achieve in our settings where test images are transformed.

Finally, note that existing methods in the literature achieve close to zero error rate on the MNIST database \cite{mnist_digits}. However, unlike the proposed approach, these methods generally do not support invariance to large transformations. Furthermore, our method is general in the sense that it is not specific to handwritten digit classification and can be used in any application involving image alignment.

\subsection{Relation to feature-based methods}

The proposed registration method shares several similarities with feature-based approaches in the computer vision literature. In such methods, we represent an image using a set of local features (\textit{keypoints}) along with high dimensional \textit{descriptors} that describe the local behaviour of the image around the keypoints. In order to register accurately two images using a feature-based approach, the following two conditions must be met:

\begin{itemize}
\item \textit{Keypoints covariance to transformations:} The keypoints undergo the same transformation as the original image. 

\item \textit{Descriptor invariance to transformations:} The descriptors are oblivious to the transformation of the original image. 
\end{itemize}

Since these conditions are ideal and hard to satisfy in practice, inaccuracies generally happen in keypoint locations and matching. To account for these issues, the registration process first excludes outlier keypoints (i.e., the keypoints that are not consistent with most of the other keypoints). This is usually performed with the RANSAC procedure \cite{fischler1981random}. The relative transformation between pairs of images is finally estimated as the most likely global transformation based on the remaining (inlier) keypoints. Specifically, if $\{x_i\}_{i=1}^r$ and $\{x'_i\}_{i=1}^r$ denote the positions of the matched inlier keypoints respectively in the first and second image, the registration is performed by solving the following minimization problem:
\[
\min_{\eta \in \mathcal{T}} \sum_{i=1}^r \| f(x_i, \eta) - x'_i \|_2,
\]
where $f(x_i, \eta)$ gives the position of the keypoint $x_i$ after the transformation with $\eta$. When $\mathcal{T} = SIM(2)$, the minimum can be found by solving a system of normal equations \cite{szeliski2010computer}.

We compare now our registration approach to a baseline feature-based approach, where the features are built on the popular Scale Invariant Feature Transform (SIFT) \cite{lowe2004distinctive} \footnote{We used the opensource implementation of SIFT available at \url{http://www.vlfeat.org/~vedaldi/code/sift.html} for the experiments.} and a RANSAC \cite{fischler1981random} method for rejecting outliers.
We compare our approach to the SIFT-based solution for the estimation of large rotations. We consider the Duck image in Fig. \ref{fig:test_images} along with multiple transformed versions of this image obtained by rotation around the center of the image. Fig. \ref{fig:reg_error_sift_our} illustrates the registration error versus the angle of rotation, for the SIFT-based approach and for our registration method. The registration error is measured on the original images with $\| U(\eta) I_1 - I_2 \|_2$, where $\eta$ is the estimated transformation. It can be seen that the estimated transformation with the SIFT-based scheme becomes less accurate as the rotation angle increases. On the contrary, the performance of our method is independent of the magnitude of the transformation. This confirms that SIFT keypoints are not covariant to large rotations of the image.

\begin{figure}[ht]
	\centering
		\includegraphics[width=0.4\textwidth]{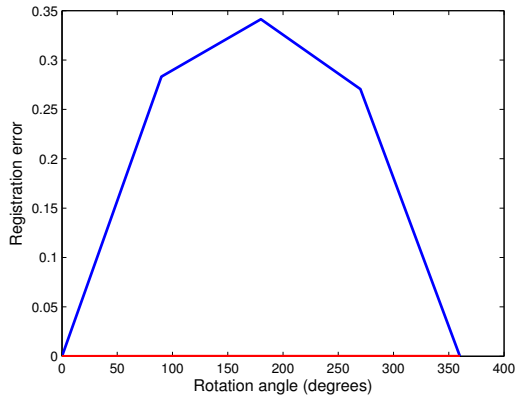}
	\caption{\label{fig:reg_error_sift_our}Registration error vs. transformation angle using our SIFT and our approach.}
\end{figure}

We finally look at the problem of handwritten digits registration with the baseline feature-based approach. We illustrate in Fig. \ref{fig:digits_sift} several examples of handwritten digits, together with the matched keypoints. One can see clearly that the matched keypoints are either inaccurate or insufficient to estimate a similarity transformation, as we need at least two matches for such an estimation. Note that we consider in Fig. \ref{fig:digits_sift} the exact transformation of handwritten digits and that there is no innovation between a pair of images apart from the global geometric transformation. Therefore, in the more difficult case where we consider different handwritten styles, the SIFT-based approach clearly fails in estimating the correct transformation. For instance, the classification of handwritten digits using the baseline SIFT-based registration approach along with a nearest-neighbour classifier leads to a classification accuracy of only $46\%$ in the same setting as above. 

\begin{figure}[ht]
\centering
\subfigure[1 match]{
\includegraphics[width=0.3\textwidth]{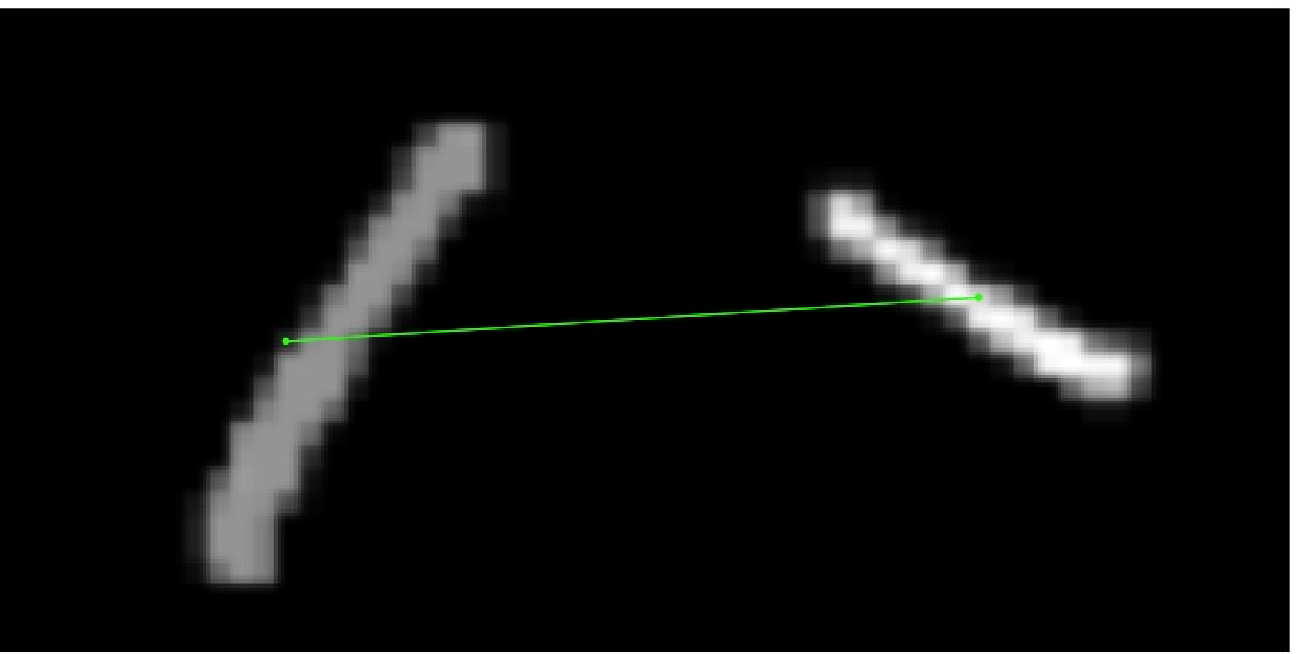}
}
\subfigure[2 matches]{
\includegraphics[width=0.3\textwidth]{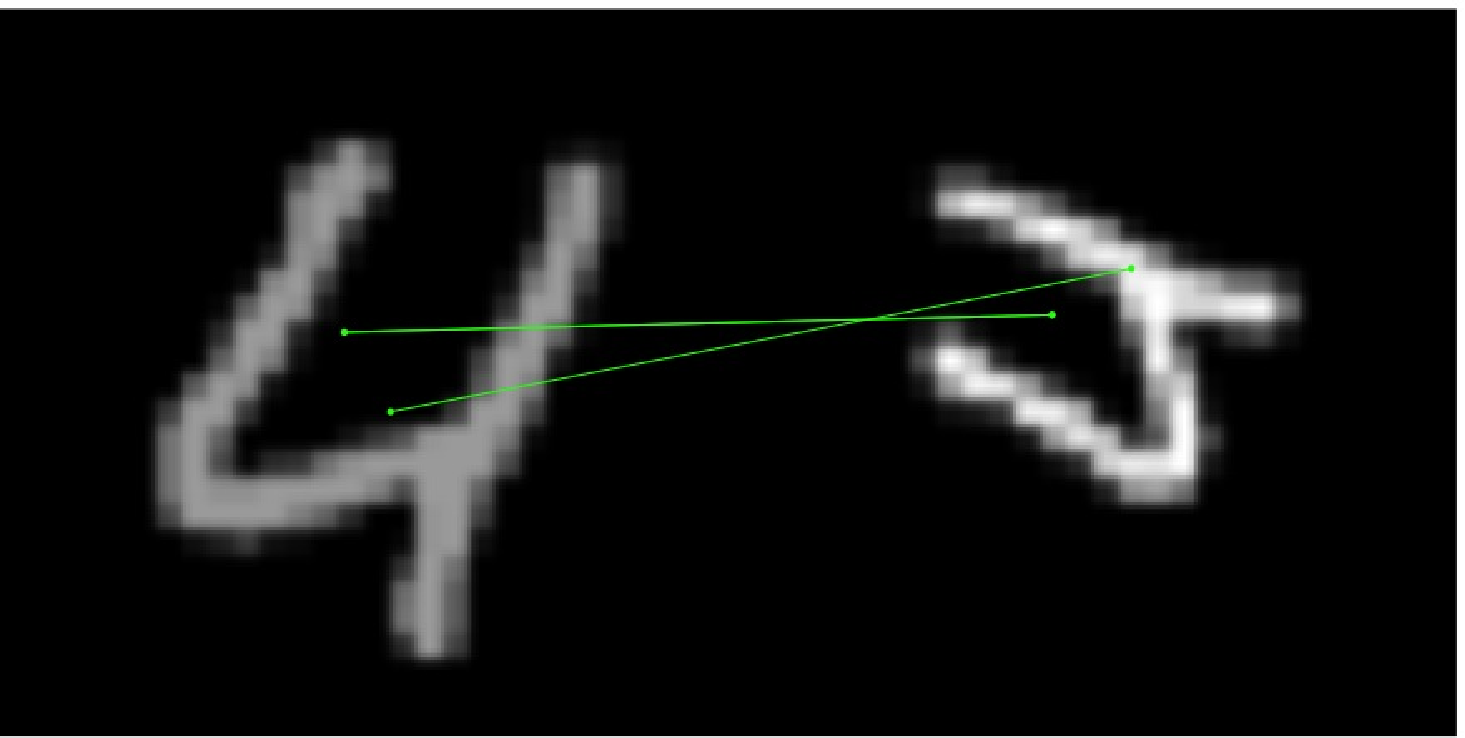}
}
\subfigure[0 matches]{
\includegraphics[width=0.3\textwidth]{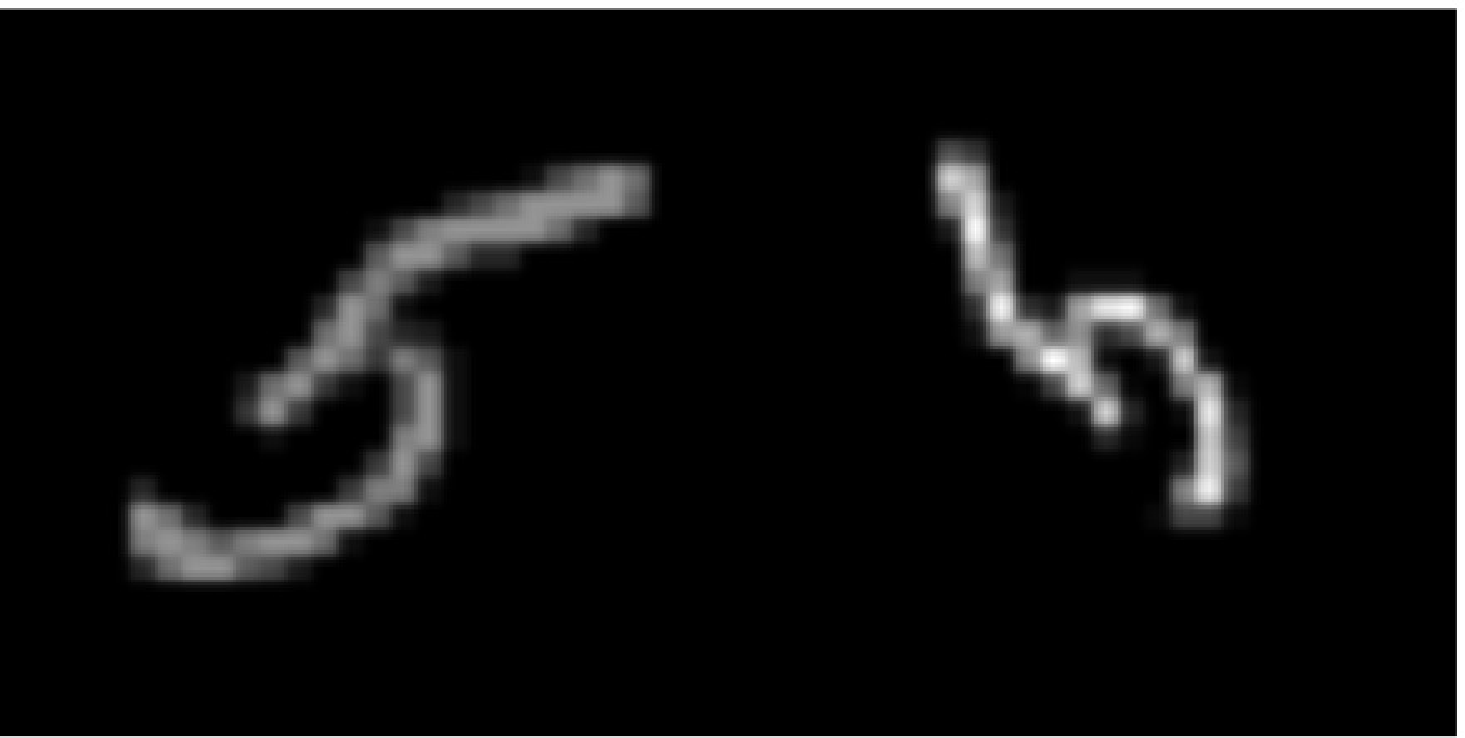}
}
\caption{\label{fig:digits_sift}Matched keypoints with SIFT features on handwritten digits.}
\end{figure}

The above examples show that, in the cases where images are sparse in geometric dictionaries of the form of Eq. (\ref{eq:dictionary}), the proposed registration approach might lead to better performance than baseline registration methods with standard visual features such as SIFT. 

\section{Conclusions}

We have proposed in this paper a simple registration algorithm based on the sparse representation of the input images in a parametric dictionary of geometric functions. Our method is general in the sense that we can achieve invariance to \textit{any} transformation group, provided that the geometric dictionary is properly constructed. We define novel properties of dictionaries, namely the robust linear independence (RLI) and transformation inconsistency in order to characterize the registration performance, which cannot be done with usual properties such as the coherence or the restricted isometry property. We show that our algorithm has low registration error when the RLI and the transformation inconsistency take small values. We also show that the proposed registration algorithm compares favorably with other baseline registration methods from the literature in illustrative alignment and classification experiments on simple visual objects and handwritten digits. To the best of our knowledge, this paper constitutes the first theoretically motivated work for image registration through sparse approximations in parametric dictionaries. We plan to extend our study to account also for the information conveyed by the coefficients of the sparse approximation, in order to further guide the registration process. Moreover, one future research direction consists in extending the algorithm to a scenario where different parts of the image can undergo different transformations. Finally, it is interesting to use the theoretical findings of this paper in order to study the design of proper dictionaries that behave well with respect to the newly introduced properties. 

\newpage
\appendix

\section{Proof of Theorem \ref{th:main_theorem}}

\label{app:proof_main_theorem}

We recall that $\eta_0$ denotes the optimal transformation between $p$ and $q$ and that $p$ and $q$ are given by:
\begin{align*}
p & = \sum_{i=1}^K c_i \phi_{\gamma_i} \\
q & = \sum_{i=1}^K d_i \phi_{\delta_i}.
\end{align*}
We can write:
\begin{align}
d_a(p,q) - d(p,q) & = \min_{\eta \in \mathcal{T}_a^{p,q}} \left\| U(\eta) p - q \right\|_2 - \left\| U(\eta_0) p - q \right\|_2 \label{eq:proof_mt_firstEq} \\
									& = \min_{\eta \in \mathcal{T}_a^{p,q}} \left\| U(\eta) p - U(\eta_0)p + U(\eta_0) p - q\right\|_2 - \left\| U(\eta_0) p - q \right\|_2 \\
									& \leq \min_{\eta \in \mathcal{T}_a^{p,q}} \left\| U(\eta) p - U(\eta_0)p \right\|_2 \\
									& = \min_{\eta \in \mathcal{T}_a^{p,q}} \left\| \sum_{i=1}^K c_i \phi_{\eta \circ \gamma_i} - \sum_{i=1}^K c_i \phi_{\eta_0 \circ \gamma_i} \right\|_2 \label{eq:proof_mt_1}
\end{align}
Let $(i^*, j^*)$ be the indices of the most correlated atoms when the two decompositions are optimally aligned:
\begin{align}
\label{eq:im_jm}
(i^*, j^*) = \argmin_{1 \leq i, j \leq K} \left\| \phi_{\eta_0 \circ \gamma_i} - \phi_{\delta_j} \right\|_2,
\end{align}
and let $\tilde{\eta}$ be the transformation between the corresponding features:
\begin{align*}
\tilde{\eta} = \delta_{j^*} \circ \gamma_{i^*}^{-1}
\end{align*}


By definition, $\tilde{\eta}$ belongs to the set of feature-to-feature transformations $\mathcal{T}_a^{p,q}$.

If $\| \phi_{\eta_0 \circ \gamma_{i^*}} - \phi_{\delta_{j^*}} \|_2 = 0$, then we have $\phi_{\eta_0 \circ \gamma_{i^*}} = \phi_{\delta_{j^*}}$. Since we suppose that $\gamma \mapsto U(\gamma) \phi$ is a bijective mapping, we have $\eta_0 \circ \gamma_{i^*} = \delta_{j^*}$ and we finally get $\eta_0 = \tilde{\eta}$. Hence, we have in this case a registration error $d_a(p,q) - d(p,q) = 0$.

We now focus on the case $\| \phi_{\eta_0 \circ \gamma_{i^*}} - \phi_{\delta_{j^*}} \|_2 > 0$. Thanks to Eq. (\ref{eq:proof_mt_1}), we have:
\begin{align}
d_a(p,q) - d(p,q) & \leq \left\| \sum_{i=1}^K c_i \phi_{\tilde{\eta} \circ \gamma_i} - \sum_{i=1}^K c_i \phi_{\eta_0 \circ \gamma_i} \label{eq:proof_mt_2}\right\|_2 \\ 
									& = \left\| \sum_{i=1}^K c_i \left( \phi_{\tilde{\eta} \circ \gamma_i} - \phi_{\eta_0 \circ \gamma_i} \right) \right\|_2 \\ 
									& \leq \sum_{i=1}^K | c_i | \| \phi_{\tilde{\eta} \circ \gamma_i} - \phi_{\eta_0 \circ \gamma_i} \|_2 \label{eq:proof_mt_lastEq}, 
\end{align}
by using the triangle inequality. Since $\| \phi_{\delta_{j^*}} - \phi_{\eta_0 \circ \gamma_{i^*}} \|_2 > 0$, we factorize the previous expression as follows:
\begin{align}
d_a(p,q) - d(p,q) & \leq \| \phi_{\delta_{j^*}} - \phi_{\eta_0 \circ \gamma_{i^*}} \|_2 \sum_{i=1}^K | c_i | \frac{\| \phi_{\tilde{\eta} \circ \gamma_i} - \phi_{\eta_0 \circ \gamma_i} \|_2}{\| \phi_{\tilde{\eta} \circ \gamma_{i^*}} - \phi_{\eta_0 \circ \gamma_{i^*}} \|_2} \\ 
									& \overset{(*)}{=} \| \phi_{\delta_{j^*}} - \phi_{\eta_0 \circ \gamma_{i^*}} \|_2 \sum_{i=1}^K | c_i | \frac{\left\| U(\eta_0^{-1} \circ \tilde{\eta}) \phi_{\gamma_i} - \phi_{\gamma_i} \right\|_2}{ \left\| U(\eta_0^{-1} \circ \tilde{\eta}) \phi_{\gamma_{i^*}} - \phi_{\gamma_{i^*}} \right\|_2} \\ 
									& \leq \| \phi_{\delta_{j^*}} - \phi_{\eta_0 \circ \gamma_{i^*}} \|_2 \sum_{i=1}^K \rho | c_i | \\ 
									& = \rho \| \phi_{\delta_{j^*}} - \phi_{\eta_0 \circ \gamma_{i^*}} \|_2  \| c \|_1 \label{eq:proof_mt_3},
\end{align}
where we have used in $(*)$ the fact that $U$ is unitary. $\rho$ is the transformation inconsistency parameter introduced in Definition \ref{def:inconsistency}.

We now focus on bounding $\left\| \phi_{\delta_{j^*}} - \phi_{\eta_0 \circ \gamma_{i^*}} \right\|_2$. In order to do so, we notice that $\phi_{\delta_{j^*}}$ and $\phi_{\eta_0 \circ \gamma_{i^*}}$ are respectively features in $q$ and $U(\eta_0) p$. Since we assume that $d(p,q) = \| U(\eta_0)p - q \| < \epsilon \sqrt{\| c \|_2^2 + \| d \|_2^2}$, by using appropriately the robust linear independence property (Definition \ref{def:rli}), we readily obtain an upper bound on $\| \phi_{\delta_{j^*}} - \phi_{\eta_0 \circ \gamma_{i^*}} \|_2$.
Formally, let $e$ be the vector of length $2K$ constructed from the concatenation of the coefficient vectors $c$ and $-d$ and define $\{\chi_j\}_{j=1}^{2K}$ as follows:
\begin{align*}
\chi_i = \eta_0 \circ \gamma_i \\
\chi_{K+i} = \delta_i.
\end{align*}
Using this definition, we have $U(\eta_0) p - q = \sum_{i=1}^K c_i \phi_{\eta_0 \circ \gamma_i} - \sum_{i=1}^K d_i \phi_{\delta_i} = \sum_{i=1}^{2K} e_i \phi_{\chi_i}$ and  $\| e \|_2 = \sqrt{\| c \|_2^2 + \| d \|_2^2}$. Since $d(p,q) < \epsilon \| e \|_2$ by hypothesis, and $\mathcal{D}$ is $\left( 2K, \epsilon, \alpha \right)$-RLI with $\alpha < \sqrt{2}$, there exist $i, j$ for which:
\begin{align}
\label{eq:proof_mt_4}
\left\| \frac{e_i \phi_{\chi_i}}{|e_i|} + \frac{e_j \phi_{\chi_j}}{|e_j|} \right\|_2 \leq \alpha,
\end{align}
as the atoms in the dictionary are normalized. If both $i$ and $j$ are not larger than $K$, the above inequality implies that:
\begin{align}
\left\| \frac{c_i}{|c_i|} \phi_{\eta_0 \circ \gamma_i} + \frac{c_j}{|c_j|} \phi_{\eta_0 \circ \gamma_j} \right\|_2 \overset{(a)}{=} \left\| \phi_{\eta_0 \circ \gamma_i} + \phi_{\eta_0 \circ \gamma_j} \right\|_2 \overset{(b)}{\geq} \sqrt{2},
\label{eq:proof_mt_5}
\end{align}
where (a) is obtained thanks to the positivity of $c$ and (b) is a consequence of the positivity of the atoms. Since we assume that $\alpha < \sqrt{2}$, Eq. (\ref{eq:proof_mt_4}) and Eq.  (\ref{eq:proof_mt_5}) cannot hold together. Hence, we exclude the case where $i \leq K$ and $j \leq K$. For the exact same reasons, it is easy to see that we cannot have $i \geq K+1$ and $j \geq K+1$. Therefore, the only possibility is $i \leq K$ and $j \geq K+1$ (or $j \leq K$ and $i \geq K+1$, which is identical, up to the relabeling of $i$ and $j$). Thus, by rewriting Eq. (\ref{eq:proof_mt_4}) we get:
\begin{align*}
\left\| \phi_{\eta_0 \circ \gamma_i} - \phi_{\delta_{j-K}} \right\|_2 \leq \alpha,
\end{align*}
thanks to the positivity of $c$ and $d$. Since $i^*$ and $j^*$ are by definition chosen to minimize the error between two features in $U(\eta_0) p$ and $q$ (Eq. (\ref{eq:im_jm})) we have: $\left\| \phi_{\eta_0 \circ \gamma_{i^*}} - \phi_{\delta_{j^*}} \right\|_2 \leq \left\| \phi_{\eta_0 \circ \gamma_i} - \phi_{\delta_{j-K}} \right\|_2 \leq \alpha$. Plugging this inequality into Eq.  (\ref{eq:proof_mt_3}), we get:
\begin{align}
d_a(p,q) - d(p,q) \leq \alpha \rho \| c \|_1.
\label{eq:proof_mt_6}
\end{align}
It is not hard to see that $d_a(p,q) = d_a(q,p)$ and $d(p,q) = d(q,p)$. Hence, we get:
\begin{align}
d_a(p,q) - d(p,q) \leq \alpha \rho \| d \|_1.
\label{eq:proof_mt_7}
\end{align}
By combining Eq. (\ref{eq:proof_mt_6}) and Eq. (\ref{eq:proof_mt_7}), we conclude that:
\[
d_a(p,q) - d(p,q) \leq \alpha \rho \min \left( \| c \|_1, \| d \|_1 \right).
\]


\section{Detailed study of the case where $\gamma \mapsto U(\gamma) \phi$ is not bijective}
\label{app:not_bijective}
We study in this appendix the case where $\gamma \mapsto U(\gamma) \phi$ is not a one-to-one mapping. In other words, we assume here that the generating function $\phi$ has symmetries in $\mathcal{T}$. More precisely, let $\mathcal{S}_{\phi}$ be defined by:
\begin{align*}
\mathcal{S}_{\phi} = \{ \gamma \in \mathcal{T} : U(\gamma) \phi = \phi\}.
\end{align*}
In group theory, $\mathcal{S}_{\phi}$ is known as the \emph{stabilizer} of $\phi$ in $\mathcal{T}$. Note that $\mathcal{S}_{\phi}$ is a subgroup of $\mathcal{T}$. Moreover, it is easy to see that the stabilizer of any atom $\phi_{\delta}$ can be obtained from $S_{\phi}$ with $\mathcal{S}_{\phi_{\delta}} = \delta \circ \mathcal{S}_{\phi} \circ \delta^{-1} = \{ \delta \circ \pi \circ \delta^{-1}: \pi \in \mathcal{S}_{\phi} \}$. Hence, given any $\delta \in \mathcal{T}$, the set of elements $\gamma$ in $\mathcal{T}$ that satisfy $\phi_{\delta} = \phi_{\gamma}$ is equal to $\delta \circ \mathcal{S}_\phi$.

When $\gamma \mapsto U(\gamma) \phi$ is a bijective mapping, $\mathcal{S}_{\phi}$ is equal to the trivial group. When $\mathcal{T} = SE(2)$ and $\phi$ is an ellipse-shaped generating function (Fig. \ref{fig:example1_rho}), the stabilizer contains two elements, namely the identity transformation and the rotation of angle $\pi$. Note that when $\phi$ is exactly circular, $\phi$ is symmetric with respect to all rotations; we get $\mathcal{S}_{\phi} = SO(2)$. 

In general, we avoid choosing a generating function whose stabilizer in $\mathcal{T}$ is an infinite subgroup, since the mother function should be discriminative enough for different transformations if we hope to recover the underlying transformation in $\mathcal{T}$. Our goal in this section is to show the modifications we need to perform in order to extend the assumption $|\mathcal{S}_{\phi}| = 1$ to $|\mathcal{S}_{\phi}| < \infty$, that is we need to assume that a limited number of symmetries exist in atom transformations.  

\subsection{Modified algorithm}

The main challenge of having $|\mathcal{S}_{\phi}| > 1$ is that several features can have the exact same appearance although they correspond to different transformations of the mother function. Clearly, arbitrarily choosing the transformation results generally in a wrong registration. The only way of solving this problem exactly is to examine all transformations that potentially generate a feature and test accordingly all feature-to-feature transformations. Formally, let $\phi_{\gamma}$ and $\phi_{\delta}$ be respectively arbitrary features in $p$ and $q$. As we mentioned earlier, the set of parameters that generate features having the same appearance as $\phi_{\gamma}$ is $\gamma \circ \mathcal{S}_{\phi}$. The same result holds for $\phi_{\delta}$. Hence, the set of transformations that map features of appearance $\phi_{\gamma}$ to features of appearance $\phi_{\delta}$ is given by:
\begin{align*}
\{ \delta \circ \pi \circ (\pi')^{-1} \circ \gamma^{-1}: \pi, \pi' \in \mathcal{S}_{\phi} \} = \{ \delta \circ \pi \circ \gamma^{-1}: \pi \in \mathcal{S}_{\phi} \}.
\end{align*}

We thus extend the set of feature-to-feature transformations $\mathcal{T}_a^{p,q}$ to:
\begin{align*}
\mathcal{T}_a^{p, q} = \{ \delta_i \circ \pi \circ \gamma_j^{-1} : 1 \leq i, j \leq K, \pi \in \mathcal{S}_{\phi} \}.
\end{align*}
Note that the only difference with respect to the set $\mathcal{T}_a^{p,q}$ defined in Section \ref{sec:registration} is that we compose in the middle of the expression with all transformations in the stabilizer group of $\phi$. Hence, the cardinality of $\mathcal{T}_a^{p,q}$ is equal to $|\mathcal{S}_{\phi}| K^2$. The rest of the algorithm (Algorithm \ref{alg:registration_algo}) remains unchanged. 

\subsection{Modified analysis}

We now turn to the analysis of the modified algorithm. First, it can be shown that in the case where images can be perfectly aligned, Proposition \ref{prop:prop1} holds for the modified algorithm when $|\mathcal{S}_{\phi}| < \infty$, when there is a finite number of symmetries. 

We then extend the analysis of the modified algorithm to the case where images cannot be perfectly aligned, but where the innovation is limited by $d(p,q) < \epsilon \sqrt{\| c \|_2^2 + \| d \|_2^2}$. The main difficulty of the analysis lies in the fact that we have the transformation inconsistency $\rho$ (as defined in Definition \ref{def:inconsistency}) equal to infinity when the mother function is symmetric (we can see this for example by considering the same setting as in Example \ref{ex:example_rho1} illustrated in Fig \ref{fig:example1_rho} with $\eta$ a rotation of $\pi$). We take into account the symmetries of the generating function in the following new definition of $\rho$: 
\begin{align}
\rho = \sup_{\eta \in \mathcal{T}} \sup_{\substack{\eta' \in \mathcal{T}_d \\ \eta' \notin \eta \circ \mathcal{S}_{\phi}}} \inf_{\pi \in \mathcal{S}_{\phi}} \sup_{\gamma \in \mathcal{T}_d} \frac{\left\| U(\eta \circ \pi \circ (\eta')^{-1}) \phi_{\gamma} - \phi_{\gamma} \right\|_2}{\left\| U(\eta) \phi - U(\eta') \phi \right\|_2},
\label{eq:inconsistency_symmetry}
\end{align}
where $\eta \circ \mathcal{S}_{\phi}$ denotes the set $\{ \eta \circ \gamma, \gamma \in \mathcal{S}_{\phi} \}$. Note that by constraining $\eta'$ to be outside the set $\eta \circ \mathcal{S}_\phi$, the denominator of the above equation is never equal to zero. Therefore, this new definition of the transformation inconsistency solves the problem that we have observed in Example \ref{ex:example_rho1} 
 for the particular case of generating functions having a symmetry of $\pi$ in $\mathcal{T} = SE(2)$. 

Note also that when $\mathcal{S}_\phi$ is the trivial group, the above definition of $\rho$ reduces to  Definition \ref{def:inconsistency}, since it is easy to check that 
\[
\frac{\left\| U(\eta \circ \pi \circ (\eta')^{-1}) \phi_{\gamma} - \phi_{\gamma} \right\|_2}{\left\| U(\eta) \phi - U(\eta') \phi \right\|_2} = 
\frac{\left\| U(\eta \circ (\eta')^{-1}) \phi_{\gamma} - \phi_{\gamma} \right\|_2}{\left\| U(\eta \circ (\eta')^{-1}) \phi_{\eta'} - \phi_{\eta'} \right\|_2} \leq \sup_{\gamma, \gamma' \in \mathcal{T}_d} \sup_{\eta \in \mathcal{T} \backslash \{\mathbb{I} \}} \left\{ \frac{\| U(\eta) \phi_{\gamma'} - \phi_{\gamma'} \|_2}{\| U(\eta) \phi_{\gamma} - \phi_{\gamma} \|_2} \right\},
\]
and the reverse inequality also holds.
Hence, this definition can be seen as an extension to the case where the generating function has intrinsic symmetries in $\mathcal{T}$. Intuitively, the transformation inconsistency $\rho$ is small whenever two transformations $\eta$ and $\eta'$ applied on the generating function that yield similar atoms in appearance will be such that $\eta \circ \pi \circ (\eta')^{-1}$ does not induce a large change in the appearance of \textit{any} atom in the dictionary $\mathcal{D}$, for some $\pi \in \mathcal{S}_\phi$.  

Using this new definition of $\rho$, we obtain the same bound of Theorem \ref{th:main_theorem} for the modified algorithm. In the following, we give the main differences in the proof of this statement with respect to the proof of Theorem \ref{th:main_theorem} given in Appendix \ref{app:proof_main_theorem}. 


\vspace{4mm}

\begin{proof}
Let $(i^*, j^*)$ be the indices defined in Eq. (\ref{eq:im_jm}), and let $\tilde{\eta} = \delta_{j^*} \circ \pi \circ \gamma_{i^*}^{-1}$, for any $\pi \in \mathcal{S}_{\phi}$. Clearly, we have $\tilde{\eta} \in \mathcal{T}_a^{p,q}$. 

In the case where $\| \phi_{\eta_0 \circ \gamma_{i^*}} - \phi_{\delta_{j^*}} \|_2 = 0$, there exists $\pi \in \mathcal{S}_{\phi}$ such that $\eta_0 \circ \gamma_{i^*} = \delta_{j^*} \circ \pi$, thus $\eta_0 = \delta_{j^*} \circ \pi \circ \gamma_{i^*}^{-1} \in \mathcal{T}_a^{p,q}$. Hence, in this case $d_a(p,q) = d(p,q)$.

We consider now the case where $\| \phi_{\eta_0 \circ \gamma_{i^*}} - \phi_{\delta_{j^*}} \|_2 > 0$. By using the same series of inequalities as in Eq. (\ref{eq:proof_mt_firstEq})- (\ref{eq:proof_mt_lastEq}), we know that:
\begin{align*}
d_a(p,q) - d(p,q) & \leq \| \phi_{\delta_{j^*}} - \phi_{\eta_0 \circ \gamma_{i^*}} \|_2 \sum_{i=1}^K | c_i | \frac{\| U(\eta_0^{-1} \circ \tilde{\eta}) \phi_{\gamma_i} - \phi_{\gamma_i} \|_2}{\| \phi_{\delta_{j^*}} - \phi_{\eta_0 \circ \gamma_{i^*}} \|_2} \\
									& \leq \| \phi_{\delta_{j^*}} - \phi_{\eta_0 \circ \gamma_{i^*}} \|_2 \sup_{\gamma \in \mathcal{T}_d} \left\{ \frac{\| U(\eta_0^{-1} \circ \tilde{\eta}) \phi_{\gamma} - \phi_{\gamma} \|_2}{\|\phi_{\delta_{j^*}} - \phi_{\eta_0 \circ \gamma_{i^*}} \|_2}  \right\}  \sum_{i=1}^K | c_i |
\end{align*}
Since this inequality is valid for any $\tilde{\eta}$ of the form $\delta_{j^*} \circ \pi \circ \gamma_{i^*}^{-1}$ where $\pi \in \mathcal{S}_\phi$, we deduce from the previous inequality that:
\begin{align*}
d_a(p,q) - d(p,q) & \leq \| \phi_{\delta_{j^*}} - \phi_{\eta_0 \circ \gamma_{i^*}} \|_2 \inf_{\pi \in \mathcal{S}_{\phi}} \sup_{\gamma \in \mathcal{T}_d} \left\{ \frac{\| U(\eta_0^{-1} \circ \delta_{j^*} \circ \pi \circ \gamma_{i^*}^{-1}) \phi_{\gamma} - \phi_{\gamma} \|_2}{\| \phi_{ \gamma_{i^*}} - \phi_{\eta_0^{-1} \circ \delta_{j^*}} \|_2} \right\} \| c \|_1 \\
									& \leq \| \phi_{\delta_{j^*}} - \phi_{\eta_0 \circ \gamma_{i^*}} \|_2 \sup_{\eta \in \mathcal{T}} \sup_{\substack{\eta' \in \mathcal{T}_d \\ \eta' \notin \eta \circ \mathcal{S}_\phi}} \inf_{\pi \in \mathcal{S}_{\phi}}  \sup_{\gamma \in \mathcal{T}_d} \left\{ \frac{\| U(\eta \circ \pi \circ (\eta')^{-1}) \phi_{\gamma} - \phi_{\gamma} \|_2}{\| U(\eta) \phi - U(\eta') \phi \|_2} \right\} \| c \|_1 \\
									& = \| \phi_{\delta_{j^*}} - \phi_{\eta_0 \circ \gamma_{i^*}} \|_2 \rho \| c \|_1
\end{align*}
By using the same upper bound on $\| \phi_{\delta_{j^*}} - \phi_{\eta_0 \circ \gamma_{i^*}} \|_2$ in the exact same way as in Appendix \ref{app:proof_main_theorem} (thanks to the RLI property), we obtain the desired result.
\end{proof}


\section{Gradient descent refinement}
\label{app:gradient_descent}

We describe in this appendix the local optimization technique that we use to refine the estimation of the transformation obtained with our registration algorithm. Specifically, we present here briefly our gradient descent approach that respects the intrinsic geometry of our registration problem. In order to do so, we first define an appropriate distance in $\mathcal{T}$. Then, we formulate the induction of the gradient descent on $\mathcal{T}$, where we follow an approach similar to the work by Jacques et. al. in \cite{jacques2008geometrical}

The most direct distance in $\mathcal{T}$ is the mere Euclidean distance: $\sqrt{\sum_{i=1}^P (\gamma_1^i - \gamma_2^i)^2}$ where $\gamma_1^i$ and $\gamma_2^i$ denote respectively the components of $\gamma_1 \in \mathcal{T}$ and $\gamma_2 \in \mathcal{T}$. However, this distance is artificial since it mixes several components that are different in nature (translation, rotation and scale components for example). Thus, we use instead a distance that is naturally introduced by the continuous dictionary $\mathcal{D}_c = \{ U(\gamma) \phi: \gamma \in \mathcal{T}\} \subset L^2$. That is, rather than considering the distance directly between the parameters, we consider the distance between \emph{the atoms generated by these parameters}. Hence, we first introduce a distance in the signal space, and translate naturally this distance to the parameter space. 

The space $\mathcal{D}_c$ is a continuous submanifold of $L^2$ \cite{donoho2005image}. We let $g(\gamma_1, \gamma_2)$ be the geodesic distance between $\phi_{\gamma_1}$ and $\phi_{\gamma_2}$ in $\mathcal{D}_c$. It corresponds to the shortest path in $\mathcal{D}_c$ between $\phi_{\gamma_1}$ and $\phi_{\gamma_2}$, where $\phi$ is the generating function of the dictionary. Formally, we have:
\begin{align*}
g(\gamma_1, \gamma_2) = \inf \left\{L(\phi_{z}): \text{ all curves $z: [0, 1] \rightarrow \mathcal{T}$ satisfying $z(0) = \gamma_1$ and $z(1) = \gamma_2$ } \right\},
\end{align*}
where $L$ is the length of the curve $\phi_z$:
\begin{align}
\label{eq:curve_length}
L(\phi_{z}) = \int_{0}^1 \left\| \frac{d \phi_{z(t)}}{dt} \right\|_{L^2} dt.
\end{align}

We use the chain rule to expand the previous expression:
\begin{align*}
\frac{d \phi_{z(t)}}{dt} = \sum_{i=1}^P \dot{z}^i(t) \partial_i \phi_{z(t)},
\end{align*}
where $\dot{z}^i (t)$ denotes the $i$-th component of $\frac{dz}{dt} (t)$ and $\partial_i \phi_{z(t)} = \frac{\partial \phi_{z(t)}}{\partial \gamma^i}$. By injecting in Eq. (\ref{eq:curve_length}), we get:
\begin{align*}
L(\phi_{z}) & = \int_{0}^1 \sqrt{\sum_{i=1}^P \sum_{j=1}^P \dot{z}^i(t) \dot{z}^j(t) \scalprod{\partial_i \phi_{z(t)}}{\partial_j \phi_{z(t)}}} dt						
\end{align*}

The previous equation introduces a natural notion of metric in the parameter space, that is, a way to calculate the scalar product between two elements in a tangent space of $\mathcal{T}$. In order to see this, let $G_{\gamma}$ be a matrix of dimension $P \times P$ defined as follows: $G_{\gamma} \triangleq \left( \scalprod{\partial_i \phi_\gamma}{\partial_j \phi_\gamma} \right)_{1 \leq i,j \leq P}$, for any $\gamma \in \mathcal{T}$. Given two elements $\xi$ and $\chi$ living in the tangent space of $\mathcal{T}$ at a point $\gamma$ , we define the metric as follows:
\begin{align}
\label{eq:scal_prod_manifold}
\scalprod{\xi}{\chi}_{\gamma} = \xi^T G_{\gamma} \chi.
\end{align}
This metric is chosen in such a way that the geodesic distance in $\mathcal{D}_c$ coincides with the geodesic distance in $\mathcal{T}$. The matrix $G_{\gamma}$ is refered to as the \emph{Riemannian metric} associated to the manifold $\mathcal{T}$. We assume that this matrix is positive definite in the rest of this section.

Endowed with the above metric, starting from a point $\tau_0 \in \mathcal{T}$, the gradient descent induction is given as follows: 
\begin{align*}
\tau_{i+1} = \tau_{i} - w \nabla J(\tau_{i}) \text{ for } i \geq 0,
\end{align*}
where
\begin{align}
\label{eq:gradient_to_prove}
\nabla J(\tau_{i}) = G_{\tau_i}^{-1} \begin{bmatrix} \partial_1 J(\tau_i) \\ \vdots \\ \partial_p J(\tau_i) \end{bmatrix}
\end{align} 
and $w$ defines the step size. On a practical level, the step size $w$ is chosen using a line search at each iteration. We limit the overall number of iterations in order to control the computational complexity of the algorithm.

One can check that the above definition of the gradient $\nabla J(\tau_{i})$ is natural, since the gradient is defined with the following equality.
\begin{align*}
\scalprod{\nabla J(\tau)}{\xi}_\tau = dJ_{\tau} (\xi),
\end{align*}
for any $\tau \in \mathcal{T}$ and $\xi$ belongs to the tangent space at $\tau$, and $dJ_{\tau} (\xi)$ gives the directional derivative of $J$ in the direction of $\xi$ evaluated at $\tau$. We can expand $dJ_{\tau} (\xi)$ as follows:
\begin{align}
\label{eq:directional_derivative}
dJ_{\tau} (\xi) = \sum_{i=1}^p \partial_i J(\tau) \xi^i = [\xi^1 \dots \xi^P] \begin{bmatrix} \partial_1 J(\tau) \\ \dots \\ \partial_P J(\tau) \end{bmatrix}.
\end{align}
Besides, by using the scalar production definition of Eq. (\ref{eq:scal_prod_manifold}), we have:
\begin{align}
\label{eq:lhs_gradient}
\scalprod{\nabla J(\tau)}{\xi}_\tau = \xi^T G_{\tau} \nabla J(\tau).
\end{align}
By combining Eq. (\ref{eq:directional_derivative}) and Eq. (\ref{eq:lhs_gradient}), we obtain the definition stated in Eq. (\ref{eq:gradient_to_prove}).

In order to illustrate the benefits of this local gradient-based optimization step, we conduct an experiment where we compare the accuracy of the estimated transformation using our approach with and without gradient descent. Specifically, we generate 100 random transformations of the Duck image in Fig. \ref{fig:test_images} and register the original image with the transformed images using both methods. The translation, rotation and scale errors are measured respectively with $\| b - b_0' \|_2$, $| a - a_0' |$ and $\min(|\hat{\theta} - \theta_0'|, 180 - |\hat{\theta} - \theta_0'|)$, where the optimal transformation is denoted by $\eta_0' = (b_0', a_0', \theta_0')$ and the estimated transformation is equal to $\hat{\eta} = (\hat{b}, \hat{a}, \hat{\theta} )$. Table \ref{tab:improvement_gradient_descent} gives the mean errors in translation, rotation and scale parameters.

\begin{table}[ht]
\footnotesize
\centering
\begin{tabular}{|c|c|c|}
  \hline
   & Without GD & With GD  \\
  \hline 
  Translation error & 2.67 & 0.72 \\
  \hline
  Scale error & 0.11 & 0.02 \\
  \hline
  Rotation errror & $7.7^{\circ}$  & $3.66^{\circ}$ \\
  \hline
\end{tabular}
\caption{Mean value of translation, scale and rotation error over 100 random trials. All the experiments are performed on the 'Duck' image (Fig \ref{fig:test_images}). The sparsity value $K$ is set to $15$.} 
\label{tab:improvement_gradient_descent}
\end{table}

We observe in practice that the overall performance of our algorithm increases substantially when gradient descent is used to refine the estimation of our registration algorithm.

\section{Proof of Example \ref{prop:rli_example}}

\label{app:proof_example_rli}

Let $a$ be an arbitrary real vector of $K$ elements, and let $\tau_1, \dots, \tau_K$ be any real numbers such that $\tau_1 < \dots < \tau_K$. Let $\epsilon$ be a sufficiently small real number that satisfies $0 < \epsilon < \sqrt{\frac{3}{4^K-1}}$. We suppose that $a$ satisfies $\left\| \sum_{i=1}^K a_i v_{\tau_i} \right\|_2 < \epsilon \|a\|_2$. Our aim is to prove that there exist two box functions $v_{\tau_{i}}$ and $v_{\tau_{j}}$ that satisfy:
\begin{align*}
\left\| \frac{a_{i} v_{\tau_i}}{\left\| a_{i} v_{\tau_i} \right\|_2} + \frac{a_{j} v_{\tau_j}}{\left\| a_{j} v_{j} \right\|_2} \right\|_2 \leq \alpha,
\end{align*}
with $\alpha = \epsilon \sqrt{\frac{2}{3} ( 4^K - 1 ) }$.

We assume without loss of generality that $\|a\|_2 = 1$. We first show the following result, that establishes a lower bound on one of the components of the coefficient vector $a$:
\begin{lemma}
\quad
There exists $i \in \{ 1, \dots, K \}$ such that $|a_{i}| \geq 2^{i-1} Y$, with $Y = \sqrt{\frac{3}{4^K-1}}$.
\end{lemma}
\begin{proof}
We prove this lemma by contradiction.
We have:
\begin{align*}
\| a \|_2^2 & = \sum_{i=1}^K | a_i |^2 \\
						& < \sum_{i=0}^{K-1} 2^{2i} \frac{3}{4^K-1} \\
						& = \frac{3}{4^K-1} \sum_{i=0}^{K-1} 4^i \\
						& = \frac{3}{4^K-1} \frac{4^K - 1}{3} \\
						& = 1,
\end{align*}
which contradicts the fact that $\|a\|_2 = 1$.
\end{proof}

We let $i^*$ be the smallest integer that satisfies $|a_i| \geq 2^{i-1} Y$. The following lemma shows that there exists necessarily an interval where the function $\left| \sum_{i=1}^K a_i v_{\tau_i} (t) \right|$ is larger than $Y$.

\begin{lemma}
\quad
\label{lemma:boxes}
\begin{enumerate}
\item There exists an index $j$ satisfying $\tau_{i^*} < \tau_j \leq \tau_{i^*} + 1$ such that $a_{j} a_{i^*} < 0$.
\item Let $j^*$ be the smallest integer larger than $i^*$ that verifies $a_{j^*} a_{i^*} < 0$. For all $t \in [\tau_{i^*}, \tau_{j^*})$, $\left| \sum_{i=1}^K a_i v_{\tau_i} (t) \right| \geq Y$.
\end{enumerate}
\end{lemma}

\begin{proof}
\quad
\begin{enumerate}
\item We prove the first statement by contradiction. Suppose that either all box functions between  $\tau_{i^*}$ and $\tau_{i^*} + 1$ are associated with coefficients that have the same sign as $a_{i^*}$, or no box functions exist between $\tau_{i^*}$ and $\tau_{i^*} + 1$. Let $j_0$ be the largest index such that $\tau_{i^*} \leq \tau_{j_0} < \tau_{i^*} + 1$. We have:
\begin{align*}
\left\| \sum_{i=1}^K a_i v_{\tau_i} \right\|_2^2 & = \int_{\tau_1}^{+\infty} \left| \sum_{i=1}^K a_i v_{\tau_i} (t) \right|^2 dt \\
																					& \geq \int_{\tau_{i^*}}^{\tau_{i^*}+1} \left| \sum_{i=1}^K a_i v_{\tau_i} (t) \right|^2 dt \\
																					& = \int_{\tau_{i^*}}^{\tau_{i^*}+1} \left| \sum_{i=1}^{j_0} a_i v_{\tau_i} (t) \right|^2 dt \\
																					& = \int_{\tau_{i^*}}^{\tau_{i^*}+1} \left| \sum_{i=i^*}^{j_0} a_i v_{\tau_i} (t) + \sum_{i=1}^{i^*-1} a_i v_{\tau_i} (t) \right|^2 dt
\end{align*}
By using the triangle inequality, we have for any $t \in [\tau_{i^*}, \tau_{i^*}+1]$:
\begin{align*}
\left| \sum_{i=i^*}^{j_0} a_i v_{\tau_i} (t) + \sum_{i=1}^{i^*-1} a_i v_{\tau_i} (t) \right| & \geq \left| \sum_{i=i^*}^{j_0} a_i v_{\tau_i} (t) \right| - \sum_{i=1}^{i^*-1} | a_i | \\
							   & \geq | a_{i^*} | - \sum_{i=1}^{i^*-1} | a_i |.
\end{align*}
The last inequality derives from the fact that the coefficients $a_i$ have all the same sign for $i \in \{i^*, \dots, j_0\}$. As $i^*$ is by definition the smallest integer which satisfies $|a_{i^*}| \geq 2^{i^*-1} Y$, we have $|a_i| < 2^{i-1} Y$ for all $i \in \{1, \dots, i^*-1\}$. Hence:
\begin{align*}
\sum_{i=1}^{i^*-1} | a_{i} | \leq Y \sum_{i=1}^{i^*-1} 2^{i-1} = Y (2^{i^* - 1} - 1).
\end{align*}
Thus, $\left| a_i^* \right| - \sum_{i=1}^{i^*-1} |a_i| \geq 2^{i^*-1} Y - Y (2^{i^* - 1} - 1) \geq Y$. Finally, we have:
\begin{align*}
\int_{\tau_{i^*}}^{\tau_{i^*}+1} \left| \sum_{i=i^*}^{j_0} a_i v_{\tau_i} (t) + \sum_{i=1}^{i^*-1} a_i v_{\tau_i} (t) \right|^2 dt \geq \int_{\tau_{i^*}}^{\tau_{i^*}+1} \left(  | a_{i^*} | - \sum_{i=1}^{i^*-1} | a_i | dt \right)^2 dt \geq Y^2,
\end{align*}
 which leads to a contradiction since $\epsilon < Y$.

\item Let $t \in [\tau_{i^*}, \tau_{j^*})$. Then, we have: 
\begin{align*}
\left| \sum_{i=1}^K a_i v_{\tau_i} (t) \right| & = \left| \sum_{i=1}^{j^*-1} a_i v_{\tau_i} (t) + \sum_{i=j^*}^{K} a_i v_{\tau_i} (t) \right|  \\
																				& = \left| \sum_{i=1}^{j^*-1} a_i v_{\tau_i} (t) \right| \\
																				& \geq \left| \sum_{i=i^*}^{j^*-1} a_i v_{\tau_i} (t) \right| - \sum_{i=1}^{i^*-1} |a_i| \\
																				& \geq \left| a_{i^*} \right| - \sum_{i=1}^{i^*-1} |a_i| .
\end{align*}
The last inequality is obtained due to the fact that $\tau_{j^*} \leq \tau_{i^*}+1$ (hence $v_{\tau_{i^*}} (t) = 1$) and that the coefficients $a_i$ have the same sign for all $i \in \{i^*, \dots, j^*-1\}$. As $i^*$ is by definition the smallest integer that satisfies $|a_{i^*}| \geq 2^{i^*-1} Y$, we have $|a_i| < 2^{i-1} Y$ for all $i \in \{1, \dots, i^*-1\}$. Hence:
\begin{align*}
\sum_{i=1}^{i^*-1} | a_{i} | \leq Y \sum_{i=1}^{i^*-1} 2^{i-1} = Y (2^{i^* - 1} - 1).
\end{align*}
Thus, $\left| a_i^* \right| - \sum_{i=1}^{i^*-1} |a_i| \geq  2^{i^*-1} Y - Y (2^{i^* - 1} - 1) \geq Y $, which concludes the proof of the lemma.
\end{enumerate}
\end{proof}

We now prove that two box functions have necessarily to be close to each other since the function $\left| \sum_{i=1}^K a_i v_{\tau_i} (t) \right|$ is large enough in the interval $[\tau_{i^*}, \tau_{j^*})$ (and at the same time $\left\| \sum_{i=1}^K a_i v_{\tau_i} \right\|_2 < \epsilon$). We have:
\begin{align}
\epsilon^2 \geq \left\| \sum_{i=1}^K a_i v_{\tau_i} \right\|_2^2 & = \int_{\tau_1}^{+\infty} \left| \sum_{i=1}^K a_i v_{\tau_i} (t) \right|^2 dt \nonumber \\
																					& = \int_{\tau_1}^{\tau_{j^*-1}} \left| \sum_{i=1}^K a_i v_{\tau_i} (t) \right|^2 dt + \int_{\tau_{j^*-1}}^{\tau_{j^*}} \left| \sum_{i=1}^K a_i v_{\tau_i} (t) \right|^2 dt + \int_{\tau_{j^*}}^{\infty} \left| \sum_{i=1}^K a_i v_{\tau_i} (t) \right|^2 dt \nonumber \\
																					& \geq \int_{\tau_{j^*-1}}^{\tau_{j^*}} \left| \sum_{i=1}^K a_i v_{\tau_i} (t) \right|^2 dt \nonumber \\
																					& \geq (\tau_{j^*} - \tau_{j^*-1}) Y^2 \label{eq:box_eq1},
\end{align}
thanks to Lemma \ref{lemma:boxes}. We thus get:
\begin{align*}
\epsilon^2 \geq (\tau_{j^*} - \tau_{j^*-1}) Y^2
\end{align*}

Moreover, the relation between $\tau_{j^*} - \tau_{j^*-1}$ and $\scalprod{v_{\tau_{j^*}}}{v_{\tau_{j^*-1}}}$ can be obtained easily:
\begin{align*}
\scalprod{v_{\tau_{j^*}}}{v_{\tau_{j^*-1}}} = \begin{cases}
  1 - \left| \tau_{j^*} - \tau_{j^*-1} \right| & \text{if $\left| \tau_{j^*} - \tau_{j^*-1} \right| \leq 1$} \\
  0 & \text{otherwise}
\end{cases}
\end{align*}
As $\epsilon < Y$, we have $\left| \tau_{j^*} - \tau_{j^*-1} \right| < 1$. Hence,
\begin{align*}
1 - \scalprod{v_{\tau_{j^*}}}{v_{\tau_{j^*-1}}} \leq \frac{\epsilon^2}{Y^2}.
\end{align*}
Moreover, as $a_{j^*-1} a_{j^*} < 0$ by construction, we have: 
\begin{align*}
\left\| \frac{a_{j^*-1} v_{\tau_{j^*-1}}}{|a_{j^*-1}|} + \frac{a_{j^*} v_{\tau_{j^*}}}{|a_{j^*}|} \right\|_2 = \left\| v_{\tau_{j^* - 1}} - v_{\tau_{j^*}} \right\|_2 = \sqrt{2\left(1-\scalprod{v_{\tau_{j^*}}}{v_{\tau_{j^*-1}}}\right)} \leq \sqrt{2} \frac{\epsilon}{Y},
\end{align*}
which concludes the proof.

\vspace{-2mm}
\bibliographystyle{siam}
\bibliography{draft_alignment}

\end{document}